\documentclass[lettersize,journal]{IEEEtran}
\usepackage[colorlinks,urlcolor=blue,linkcolor=blue,citecolor=blue]{hyperref}
\usepackage{amsmath,amsfonts}
\usepackage{algorithmicx}
\usepackage{algpseudocode}
\usepackage{titlesec} \titlespacing*{\paragraph}{0pt}{3pt}{0.5em}
\usepackage{array}
\usepackage{textcomp}
\usepackage{stfloats}
\usepackage{url}
\usepackage{amsmath}
\usepackage[ruled,vlined,linesnumbered]{algorithm2e}
\usepackage{stmaryrd}

\usepackage{amsthm}
\newtheorem{theorem}{Theorem}
\newtheorem{proposition}{Proposition}
\usepackage[table]{xcolor}
\usepackage[caption=false,font=normalsize,labelfont=sf,textfont=sf]{subfig}
\newtheorem{definition}{Definition}
\newtheorem{lemma}{Lemma}
\newtheorem{corollary}{Corollary}

\usepackage{graphicx}
\usepackage{verbatim}
\def\BibTeX{{\rm B\kern-.05em{\sc i\kern-.025em b}\kern-.08em
    T\kern-.1667em\lower.7ex\hbox{E}\kern-.125emX}}
\usepackage{balance}
\begin{document}
\title{Provenance Guided Incremental Learning Under Evolving Concept Definitions}
\author{
Ismail Lamaakal, \IEEEmembership{Student Member, IEEE}

\thanks{Lamaakal. I is with the Department of Computer Science, Faculty of Applied Sciences Nador, Mohammed Premier University, Oujda, Morocco. Corresponding author email: ismail.lamaakal@ieee.org}

}

\maketitle



\begin{abstract}
Learning systems deployed over long periods must adapt not only to
statistical changes in incoming data, but also to revisions of the
definitions that generate their prediction targets. Conventional
concept-drift methods typically infer such changes from observations or
prediction errors, even when the underlying policy, rule, or query has
been explicitly modified. This paper studies \emph{rule-induced concept
shift}, where the target-defining concept is revised directly, causing
previously stored instances to acquire different semantic labels
without requiring any change in their observed data. We introduce a
provenance-guided incremental learning framework that compiles
consecutive concept definitions into a structured rule delta, traces
the changed components through historical provenance, certifies records
whose previous labels remain valid, and restricts reevaluation to a
localized candidate region. Executable revisions are relabeled
automatically, ambiguous cases are handled through selective
supervision, and the resulting changes are used for incremental
predictor repair. A versioned concept memory further supports recurring
definitions. We also introduce RuleShift-Bench, spanning financial,
demographic, cybersecurity, and graph-structured data with threshold,
predicate, logical, relational, recurring, and mixed concept
revisions. Across the benchmark, provenance-guided repair attains
92.3\% accuracy and 90.2\% Macro-F1 while reprocessing 14.7\% of the
historical collection and retaining 94.6\% of affected records. Its
average update latency is 179\,s compared with 993\,s for complete
relabeling and retraining. The results demonstrate that an explicit
concept revision can be exploited as a data-maintenance signal,
allowing learning systems to update the supervision and predictive
state that depend on the change while preserving knowledge that remains
valid.
\end{abstract}




\section{Introduction}
\label{sec:introduction}

Machine-learning models are increasingly deployed in environments where
the meaning of the prediction target evolves over time. Fraud detection,
security monitoring, compliance checking, eligibility assessment,
content moderation, recommendation, and risk management are typical
examples in which a model trained under one operational definition may
later be used under a revised one. Recent studies in online fraud
detection and adaptive intrusion detection illustrate how rapidly
evolving behaviors can reduce the validity of models learned under
earlier operating conditions
\cite{intro_zhu2024,intro_seth2024,intro_shyaa2024}.
Maintaining predictive quality in such environments is therefore not
only a question of learning an accurate model initially, but also of
updating the learning system when the semantics of its target change.

A large body of work studies this problem through the lens of concept
drift and streaming learning. Recent surveys characterize concept drift
as a change in the statistical process governing an evolving data
stream and organize existing approaches according to how such changes
are detected, localized, and handled
\cite{intro_hinderA2024,intro_arora2024,intro_lukats2025}.
In the conventional setting, a change is not directly available to the
learner. Instead, it must be inferred from incoming observations,
prediction errors, changes in class frequencies, or shifts in the
underlying data distribution. Recent work accordingly studies drift
locality and distribution-based detection
\cite{intro_aguiar2024,intro_sun2024}, drift-type-aware adaptation
\cite{intro_li2024}, adaptive ensembles
\cite{intro_wei2024}, streaming neural adaptation
\cite{intro_chambers2025}, proactive adaptation
\cite{intro_guerrero2026}, and drift handling in distributed streaming
settings \cite{intro_chen2026}. Drift detectors, adaptive classifiers,
online learners, replay mechanisms, and recurring-drift methods
consequently focus on recognizing that the current predictor no longer
matches the evolving environment and on adapting the model after
sufficient evidence of change has appeared.

This view is appropriate when the mechanism responsible for the change
is hidden. However, many data-intensive applications evolve in a
different way. The target may be generated by an explicit policy,
business rule, compliance specification, relational condition, graph
criterion, or expert-defined decision procedure, and this definition
can itself be revised. In a transaction-monitoring system, for example,
an earlier policy may classify a transaction as suspicious when its
amount exceeds five thousand units or when it is foreign. A revised
policy may instead combine a lower transaction threshold with device
mismatch and additionally classify transactions involving blocked
merchants. Nothing about a previously stored transaction needs to
change for its correct label to become different; what changes is the
definition used to interpret that transaction. Existing work has
started to move beyond simply detecting drift toward locating,
characterizing, and explaining where statistical changes occur
\cite{intro_hinderA2024,intro_wang2024,intro_yang2025}, but these
approaches still principally reason from observed changes in data,
predictions, or learned relationships rather than from a directly
revised target-generating rule.

This situation exposes an important limitation of treating every target
change as ordinary statistical drift. If the revised definition is
already known, waiting for a sequence of prediction failures to reveal
the change discards information that is directly available to the
system. More importantly, detecting that a concept has changed does not
identify which historical examples have become semantically invalid.
Some records may depend on the modified part of the definition and
require reevaluation, whereas many others may remain correct under both
the previous and revised definitions. This distinction is related to
the broader observation that drift can be localized rather than global
and that different changes can affect different regions of the data
space \cite{intro_aguiar2024,intro_hinderA2024}. However, in our setting
the localization problem begins from an explicit modification of the
target-defining computation rather than from a statistically detected
distributional change.

A straightforward response is to apply the revised definition to the
complete historical database, regenerate every training label, and
retrain the predictor from scratch. Although correct, this strategy can
be unnecessarily expensive for large historical collections. The cost
becomes particularly significant when concept definitions contain
relational joins, graph dependencies, historical aggregates, external
services, or other expensive operations. Recent work on industrial
data streams similarly emphasizes both the computational difficulty of
repeated adaptation and the need to exploit reusable knowledge when
operating conditions evolve
\cite{intro_zhang2024,intro_zhang2025}. Complete recomputation also
ignores the fact that a local revision to a policy may affect only a
small fraction of previously observed instances. At the opposite
extreme, updating a model only from recent observations may be cheaper,
but it provides no guarantee that older examples whose semantic labels
have changed are properly reconsidered. The central problem is therefore not simply whether a concept has
changed, but how the known revision propagates through historical data
and predictive state. A useful maintenance mechanism should determine
which parts of the previous training collection can still be trusted,
which records need to be reevaluated, which revised labels can be
generated automatically, and which cases require additional
supervision. It should then update the predictor using this localized
change without unnecessarily rebuilding knowledge that remains valid.
This motivates a shift from conventional drift detection toward
explicit concept-definition maintenance: rather than rediscovering a
known semantic revision indirectly from subsequent observations, the
revision itself can be used to determine which historical supervision
and predictive knowledge actually require reconsideration.

This paper addresses this problem through provenance-guided incremental
learning under evolving concept definitions. The framework starts from
the previous and revised concept definitions, identifies the structural
difference between them, and combines that difference with provenance
information describing how historical outcomes depended on predicates,
attributes, relations, or computational paths. Historical records that
may be influenced by the revision are reconsidered, while records whose
outcomes can be certified as stable are preserved. Revised supervision
is then used to repair the predictor incrementally, and previously
encountered concept versions are retained so that recurring definitions
can be recovered efficiently. The resulting workflow connects four stages that are typically handled
independently. First, the system determines what changed in the concept
definition. Second, it uses provenance to determine which historical
records depend on that change. Third, it updates supervision only for
the relevant region, including selective human annotation when the
revision cannot be executed deterministically. Finally, it repairs the
predictor while preserving behavior associated with unaffected data.
This creates a direct path from concept-definition maintenance to
training-data maintenance and ultimately to model maintenance.

The main question addressed in this work is consequently the following:
\emph{Given an explicit change in a concept definition, can a learning
system determine which historical data and predictive state actually
need to be reconsidered, rather than rediscovering the change from
subsequent prediction errors or recomputing the entire historical
collection?}

The main contributions of this work are as follows:

\begin{enumerate}

    \item 
    We introduce rule-induced concept shift, a learning setting in which
    the definition generating the target changes explicitly rather than
    only through hidden statistical drift. The formulation separates
    the known semantic revision from its consequences for historical
    supervision and predictive state.

    \item 
    We develop a mechanism that analyzes the difference between
    consecutive concept definitions and uses historical provenance to
    identify records that may require reevaluation. The same mechanism
    certifies unaffected records when their previous outcomes remain
    invariant under the revised definition.

    \item 
    We combine automatic relabeling for executable changes, selective
    human supervision for ambiguous revisions, incremental predictor
    updating, and recovery of recurring concept versions within a
    unified maintenance framework.

    \item 
    We introduce RuleShift-Bench and evaluate the proposed framework
    across multiple data modalities and concept revisions, jointly
    considering predictive quality, historical reprocessing,
    annotation requirements, update latency, storage overhead, and
    recovery under recurring concepts.

\end{enumerate}

The remainder of the paper is organized as follows.
Section~\ref{sec:related_work} reviews the literature on concept drift,
incremental and continual learning, query evolution, and data provenance.
Section~\ref{sec:framework} formalizes learning under evolving concept
definitions and presents the proposed provenance-guided framework,
including rule-delta compilation, affected-data discovery, selective
supervision, incremental repair, and versioned concept memory.
Section~\ref{sec:theory} establishes the main correctness properties and
analyzes the computational conditions under which selective maintenance
is advantageous. Section~\ref{sec:experiments} introduces
RuleShift-Bench and evaluates predictive performance, computational
efficiency, annotation requirements, recurring concepts, component
ablations, and failure modes. Finally, Section~\ref{sec:conclusion}
summarizes the main findings, discusses the limitations of the current
framework, and outlines directions for future work.

\section{Related Work}
\label{sec:related_work}
This section reviews the main research areas related to evolving concept definitions, including concept drift and streaming learning, incremental and continual learning, incremental query processing, and data provenance.

\subsection{Concept Drift and Streaming Learning}

Concept drift and streaming learning study prediction under changing
data-generating environments.
Existing work includes statistical drift
detection, adaptive streaming classifiers, online learning, active
adaptation, dynamic windowing, and mechanisms for recurring drift
\cite{rw_wu2024,rw_din2024}.
These methods typically monitor prediction errors, feature
distributions, label distributions, or other stream statistics and
adapt the learner when sufficient evidence of change is observed
\cite{rw_su2024,rw_komorniczak2024}.
Recurring-drift methods additionally retain or recover previously useful
models when an earlier statistical regime reappears. The central
distinction in our setting is that the change in the target-defining
rule is explicitly available. Rather than statistically inferring that
drift has occurred from subsequent observations, we analyze the known
revision directly and determine how it changes historical supervision.

\subsection{Incremental and Continual Learning}

Incremental and continual learning address how predictive models can
incorporate new information without repeatedly training from scratch.
Representative approaches include replay, regularization against
forgetting, online parameter updates, incremental classifiers, and
dynamic or adaptive ensembles
\cite{rw_adaptive2025,rw_textstream2025}.
These techniques provide mechanisms for
efficiently updating a predictor and preserving previously learned
knowledge, and they are therefore complementary to the model-repair
stage considered in this work
\cite{rw_cgofed2025,rw_onepass2025}.
However, they generally assume that the
examples used for adaptation have already been identified. They do not
normally use an explicit change in the target definition to determine
which historical examples have become semantically invalid and which
remain correct. Our focus therefore precedes conventional incremental
updating: we first determine which historical supervision actually
requires reconsideration and then repair the predictor using that
localized change.

\subsection{Incremental Query Processing and Query Evolution}

Incremental query processing provides an important data-management
foundation for our setting
\cite{rw_insertdelete2024,rw_updatedependent2025}.
Incremental view maintenance and
incremental computation avoid recomputing complete query results when
only part of the underlying computation has changed
\cite{rw_cqupdates2025,rw_aqpupdates2025},
while query
differencing and work on evolving queries analyze how modifications to
predicates, operators, joins, and other query components affect
previously computed outputs. We transfer this principle from database
maintenance to learning-system maintenance. A concept definition can be
viewed as the computation that generates supervision; when this
definition changes, its structural difference determines which
historical labels may become invalid. The resulting affected training
data then determines what must be changed in the predictor. This creates
a direct bridge from \emph{query maintenance} to
\emph{training-data maintenance} and finally to
\emph{model maintenance}.

\subsection{Data Provenance and Lineage}

Data provenance and lineage describe how a result depends on underlying
tuples, attributes, predicates, relations, transformations, or
computational paths
\cite{rw_gregori2025,rw_sambara2024}.
Provenance has been widely used for explanation,
auditing, debugging, reproducibility, impact analysis, and tracing the
origin of query results
\cite{rw_godwin2024,rw_schlegel2026}.
In this work, provenance is used as an active
maintenance mechanism rather than only as an explanatory artifact.
Given a revision from \(Q_t\) to \(Q_{t+1}\), provenance identifies
which historical evaluations depend on the changed components and which
records remain insulated from those changes. We therefore exploit
provenance not only to explain previous query results, but to determine
whether a revised concept definition can change the supervision
associated with a historical instance. This connection enables
selective reevaluation and provides the basis for certifying stable
records before incremental model repair.


\section{Problem Formulation and Proposed Framework}
\label{sec:framework}

We consider long-lived predictive systems in which the definition of
the target concept can be explicitly revised after deployment. Unlike
conventional concept-drift settings, where the existence and location
of a change must typically be inferred from incoming observations,
prediction errors, or distributional statistics, the setting considered
here provides direct information about how the target-generating rule
has evolved. The central problem is therefore not merely to determine
whether a concept has changed, but to identify which historical records
can actually be affected by the revision and how the deployed predictor
should be updated without unnecessarily reconsidering the complete data
collection.

The proposed framework connects concept-rule evolution with
provenance-guided data maintenance and incremental model repair. Given
an old concept definition \(Q_t\) and its revised version
\(Q_{t+1}\), we first represent both definitions as structured predicate
graphs and compile their difference into a typed rule delta describing
which predicates, thresholds, logical operators, or relational
dependencies have changed. The resulting rule delta is then combined
with record-level provenance to determine which historical examples
could have their target assignments modified by the revision. Records
whose outputs can be certified as unchanged are excluded from
unnecessary reprocessing, while potentially affected records are
selectively re-evaluated. Automatically resolvable changes are directly
relabelled, ambiguous cases are routed to selective supervision, and
the resulting data are used to incrementally repair the deployed
predictor. A versioned concept memory additionally maintains previous
concept definitions and associated maintenance information so that
recurring definitions can be handled efficiently. Figure~\ref{fig:framework} summarizes the complete framework. The old
and revised concept definitions are first compared by the rule-delta
compiler, which produces the changed predicates and operators. The
provenance analyzer then partitions the historical data into certified
stable, exactly affected, and ambiguous regions. The exactly affected
subset is relabelled automatically, the ambiguous subset is passed to
selective annotation, and the resulting supervision is combined with a
stability buffer drawn from certified stable records to perform
incremental predictor repair. Beneath this pipeline, the versioned
concept memory stores concept graphs, provenance summaries, affected
regions, and repaired model states so that future revisions, including
recurring ones, can be handled more efficiently.
\begin{figure*}[htbp]
    \centering
    \includegraphics[width=\textwidth]
    {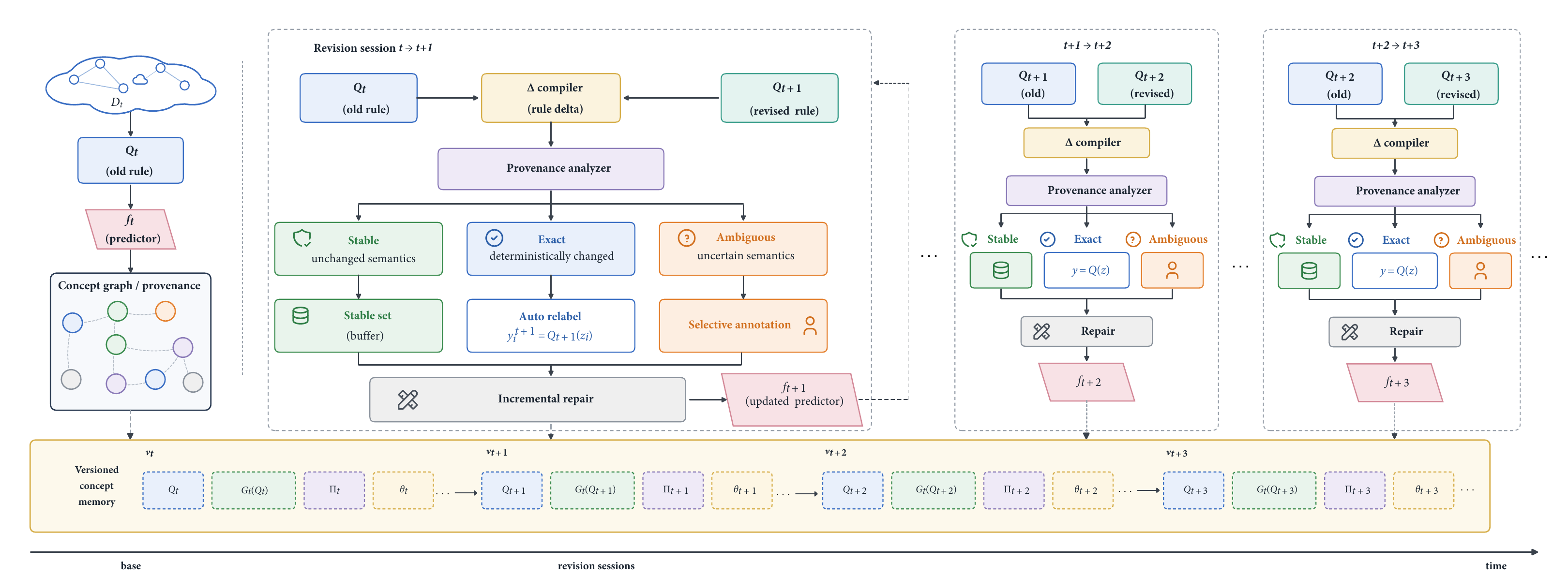}
    \caption{
    \textbf{Overview of the proposed provenance-guided incremental learning
    framework under evolving concept definitions.}
    At revision session \(t\!\rightarrow\!t+1\), the previous concept
    definition \(Q_t\) and the revised definition \(Q_{t+1}\) are first
    compared by the rule-delta compiler to identify predicates, parameters,
    logical operators, or relational dependencies that have changed.
    The provenance analyzer then traces these changes through the historical
    data and separates records into three groups: certified stable records,
    whose previous targets remain valid; exactly affected records, whose
    revised targets can be computed deterministically; and ambiguous records,
    for which additional supervision is required. Certified stable examples
    provide a stability buffer, exactly affected examples are relabeled
    automatically using \(Q_{t+1}\), and ambiguous examples are handled
    through selective annotation. The resulting revision-specific data are
    used to incrementally repair the deployed predictor from
    \(f_{\theta_t}\) to \(f_{\theta_{t+1}}\) without recomputing the complete
    historical dataset. Successive revisions are processed in the same
    manner, while a versioned concept memory stores each concept definition,
    predicate graph, provenance state, affected-data information, and model
    state to support efficient future updates and recurring concept
    definitions.}
    \label{fig:framework}
\end{figure*}
Taken together, the four components introduced in
Sections~\ref{subsec:provenance_discovery}--
\ref{subsec:versioned_memory} complete the main operational logic of
the method. The rule-delta compiler determines what changed, the
provenance analyzer determines where that change can matter, the
exact--ambiguous decomposition determines which updated labels can be
obtained automatically and which require supervision, the repair module
updates the predictor accordingly, and the versioned concept memory
maintains the system state across successive concept revisions.

\subsection{Learning Under Evolving Concept Definitions}
\label{subsec:problem_formulation}

Let \(\mathcal{Z}\) denote the complete information space maintained by
the underlying data system, and let \(z_i\in\mathcal{Z}\) represent the
information associated with entity \(i\). This representation may
include attributes distributed across relational tables, historical
aggregates, temporal information, graph relations, external knowledge,
or expert-provided information. At time \(t\), the target concept is
specified by a concept program
\(Q_t:\mathcal{Z}\rightarrow\mathcal{Y}_t\), which assigns
\(y_i^t=Q_t(z_i)\). Depending on the application, \(Q_t\) may contain
threshold predicates, Boolean conditions, relational operations,
graph-path constraints, or combinations thereof.

Although \(Q_t\) defines the target, directly executing the complete
concept program for every prediction request may be impractical.
Production systems often operate under tighter latency and information
constraints than offline data processing. We therefore distinguish
\(z_i\) from the information available to the deployed predictor. Let
\(\pi:\mathcal{Z}\rightarrow\mathcal{X}\) produce the prediction-time
representation \(x_i=\pi(z_i)\). The deployed model
\(f_{\theta_t}\) then provides an efficient approximation of the
current concept, \(f_{\theta_t}(x_i)\approx Q_t(z_i)\). This distinction
is important because evaluating \(Q_t\) may require expensive joins,
long historical windows, external services, privileged attributes,
graph traversal, or manual verification that are unavailable or too
costly at prediction time.

\paragraph{Rule-induced concept shift.}
We consider settings in which the concept program is explicitly revised
from \(Q_t\) to \(Q_{t+1}\), which we refer to as
\emph{rule-induced concept shift}. Such a shift does not require the
underlying record, prediction-time features, or their distributions to
change. The same historical instance may instead acquire a different
correct target solely because the definition used to interpret it has
changed.

Let the historical collection at the time of revision be
\(\mathcal{D}_t=\{z_i\}_{i=1}^{N}\). The records whose targets change
form the affected set
\[
\mathcal{D}_t^{\Delta}
=
\{z_i\in\mathcal{D}_t:
Q_t(z_i)\neq Q_{t+1}(z_i)\},
\]
while its complement
\(\mathcal{D}_t^{\mathrm{st}}
=\mathcal{D}_t\setminus\mathcal{D}_t^{\Delta}\)
contains records whose previous targets remain valid. Importantly,
\(Q_t\neq Q_{t+1}\) does not imply that every historical record is
affected. For example, revising a threshold from
\(\texttt{amount}>5000\) to \(\texttt{amount}>3000\) can change the
decision only for records lying between the two thresholds; records
outside this disagreement region retain the same outcome.

This distinction shows that the structural size of a rule revision and
the number of affected records are different quantities. A small rule
change may influence a large portion of the data, whereas a more
substantial structural revision may affect only a localized subset.
Identifying which historical examples require reconsideration therefore
requires reasoning jointly about what changed in the concept definition
and how individual records depend on the modified components.

A straightforward adaptation strategy evaluates \(Q_{t+1}\) over the
entire historical collection, reconstructs all labels, and retrains the
predictor from scratch. This becomes unnecessarily expensive when
\(|\mathcal{D}_t^{\Delta}|\ll|\mathcal{D}_t|\), particularly when
concept evaluation involves costly database, relational, graph, or
human operations. We therefore formulate adaptation as a joint
data-maintenance and predictive-update problem: given the previous
predictor, the historical collection, and the explicit transition from
\(Q_t\) to \(Q_{t+1}\), the objective is to obtain a predictor aligned
with the revised concept while minimizing unnecessary reevaluation,
human supervision, model-update cost, and cross-version storage.

The explicit concept transition thus provides structured information
about the origin of the change rather than treating adaptation solely
as the response to an unknown statistical event. The next step is to
convert this revision into a computational representation that exposes
which components of the concept definition have changed and can
therefore influence historical supervision.

\subsection{Concept Rule Representation and Delta Compilation}
\label{subsec:rule_delta}

The first operational step determines \emph{what changed} between two
consecutive concept definitions. Treating \(Q_t\) and \(Q_{t+1}\) as
opaque functions would hide the internal structure of the revision and
provide little guidance about which historical records may be affected.
We therefore represent each concept program as a canonical predicate
directed acyclic graph,
\(G_t=G(Q_t)=(\mathcal{V}_t,\mathcal{E}_t,\phi_t)\), where
\(\mathcal{V}_t\) contains computational nodes,
\(\mathcal{E}_t\) represents their dependencies, and \(\phi_t\)
associates each node with its operator type and parameters. The graph
is directed from lower-level predicates toward the final concept output,
making the decision dependencies explicit.

The representation supports the principal components required by the
concept definitions considered in this work. Atomic and threshold nodes
encode conditions on attributes, logical nodes represent conjunction,
disjunction, and negation, relational nodes capture dependencies across
entities or tables, and graph-path nodes describe one-hop or multi-hop
relations. More complex concepts are obtained by composing these
elements. Unlike the final target label alone, this representation
retains the computational structure needed to reason about how a local
rule revision can propagate to historical data.

Given \(G_t\) and \(G_{t+1}\), the rule-delta compiler aligns persistent
components using predicate identity, referenced attributes or relations,
operator type, parameters, and canonical graph position. Differences
that remain after alignment are summarized by the typed rule delta
\[
\mathcal{C}_t=\operatorname{Diff}(G_t,G_{t+1}),
\]
where each element records the modified component, its edit type, and,
when relevant, its previous and revised parameters.

We consider five principal edit families: threshold revisions,
predicate insertions, predicate deletions, logical rewrites, and
relational or graph-path rewrites. For example, replacing
\(Q_t=P_1\land P_2\) with \(Q_{t+1}=P_1\land P_3\) preserves
\(P_1\), removes \(P_2\), and inserts \(P_3\), yielding
\(\mathcal{C}_t=\{-P_2,+P_3\}\). A threshold revision instead preserves
the predicate identity while recording the parameter change, whereas a
logical rewrite preserves the predicates but modifies the operator that
connects them. A typed delta is useful because different edit classes induce different
patterns of affected records. Threshold revisions are localized to the
region where the old and new thresholds disagree; predicate insertions
or deletions affect records whose decisions depend on the modified
condition; logical rewrites depend on combinations of predicate values;
and relational changes may propagate through connected tuples or graph
neighborhoods. The compiler can therefore narrow the relevant
dependencies before executing the complete revised concept program.

The structural delta alone, however, does not identify the true
affected set \(\mathcal{D}_t^{\Delta}\). Two records evaluated by the
same concept may depend on different branches, tuples, or relations, so
a changed component can be decisive for one record and irrelevant to
another. Selecting every record associated with an edited rule component
would therefore remain overly conservative. The rule-delta compiler consequently provides the changed dependencies
that must be traced through the historical collection. The next stage
combines \(\mathcal{C}_t\) with record-level provenance to separate
records that can be certified stable from those requiring further
evaluation, thereby connecting structural concept change to
provenance-guided affected-data discovery.

\subsection{Provenance Guided Affected Data Discovery}
\label{subsec:provenance_discovery}

The structural delta \(\mathcal{C}_t\) identifies which components of
the concept definition have changed, but not which historical records
are affected. Records evaluated under the same concept may depend on
different predicates, tuples, relations, or graph paths; consequently,
a modified component can be decisive for one record and irrelevant to
another. We therefore reconsider a historical record only when its
previous concept evaluation may depend on components altered by the
revision. For each \(z_i\in\mathcal{D}_t\), we maintain a provenance
representation \(\operatorname{Prov}(z_i,Q_t)\) describing the
predicates, source tuples, relations, graph paths, or other
computational dependencies that contributed to \(Q_t(z_i)\).
Depending on the rule language, this information can be represented by
activated predicate nodes, tuple lineage, path signatures, or compressed
dependency indices. Unlike the final label alone, provenance records
\emph{how} the previous concept decision was obtained.

A direct dependency test provides useful intuition: if the provenance
of a record does not involve any changed component, the record may be
unaffected. However, a simple intersection test is not sufficient in
general because revisions to logical operators or ancestor nodes can
alter the evaluation path even when the same leaf predicates remain
present. The provenance analyzer therefore considers the changed
components together with their relevant dependency closure in the
concept graph.

This analysis partitions the historical collection into a
\emph{certified stable set}
\(\mathcal{D}_t^{\mathrm{safe}}\), containing records whose targets are
guaranteed to remain unchanged, and a \emph{candidate set}
\[
\mathcal{D}_t^{\mathrm{cand}}
=
\mathcal{D}_t\setminus\mathcal{D}_t^{\mathrm{safe}},
\]
containing records for which invariance cannot be certified. Only the
candidate set is reevaluated under the revised concept, and the actual
changed-label set is recovered as
\[
\mathcal{D}_t^{\Delta}
=
\{z_i\in\mathcal{D}_t^{\mathrm{cand}}:
Q_t(z_i)\neq Q_{t+1}(z_i)\}.
\]
Thus, provenance does not predict the revised label directly; it reduces
the region over which the revised concept must be executed.

The partition is deliberately conservative. A record is excluded from
reevaluation only when its stability can be certified; otherwise it
remains in \(\mathcal{D}_t^{\mathrm{cand}}\), even if subsequent
evaluation shows that its label is unchanged. This avoids retaining
obsolete supervision while allowing substantial reductions in data
access, query execution, relabeling, and update latency when the
revision is localized and provenance is informative. The rule-delta
compiler therefore determines \emph{what changed}, while provenance
determines \emph{where that change can matter}. The formal conditions
under which stable records can be certified are established in
Section~\ref{subsec:stability_certificate}.

\subsection{Exact and Ambiguous Concept Changes}
\label{subsec:exact_ambiguous}

After identifying the candidate set
\(\mathcal{D}_t^{\mathrm{cand}}\), the next step is to determine whether
the revised concept can be resolved automatically for each candidate.
If \(Q_{t+1}\) is fully executable, the updated target can be obtained
directly. Otherwise, the revision may depend on missing, delayed,
uncertain, external, or expert-interpreted information and therefore
require additional supervision.

We partition the affected region as
\[
\mathcal{D}_t^{\Delta}
=
\mathcal{D}_t^{\mathrm{exact}}
\cup
\mathcal{D}_t^{\mathrm{amb}},
\]
where \(\mathcal{D}_t^{\mathrm{exact}}\) contains records whose revised
targets can be computed deterministically and relabeled automatically,
while \(\mathcal{D}_t^{\mathrm{amb}}\) contains records for which direct
execution is unavailable or insufficiently reliable. Ambiguity may
arise from expert judgment, incomplete or delayed attributes, uncertain
learned predicates, noisy relational or graph information, or external
knowledge not fully represented in the data system.

This distinction avoids assuming that every concept revision is fully
machine-executable. In practical settings, a revised policy or concept
definition may contain both deterministic and ambiguous components, so
the framework exploits automatic relabeling wherever possible and
reserves human intervention only for unresolved cases.

Rather than annotating the entire ambiguous region, we select a smaller
subset
\[
\mathcal{D}_t^{\mathrm{label}}
\subseteq
\mathcal{D}_t^{\mathrm{amb}}
\]
using uncertainty and representativeness. Uncertainty prioritizes
records whose revised targets are difficult to infer, while
representativeness reduces redundant annotation. Human supervision is
therefore concentrated on the portion of the concept revision where it
provides the greatest information.

This decomposition reduces both computational and supervision cost:
exact revisions are handled automatically, while annotation is limited
to the ambiguous concept-delta region rather than the full historical
collection. The experiments in
Section~\ref{subsec:annotation_recurrence} evaluate this selective
annotation strategy under limited labeling budgets.

\subsection{Incremental Predictor Repair}
\label{subsec:incremental_repair}

After identifying the exact and labeled ambiguous examples, the next
step is to update the deployed predictor. The framework constructs a
repair dataset
\(\mathcal{D}_t^{\mathrm{repair}}=
\mathcal{D}_t^{\mathrm{exact}}\cup
\mathcal{D}_t^{\mathrm{label}}\), which contains records whose revised
targets are either obtained deterministically or provided through
selective annotation. This repair set represents the portion of the
historical data that carries direct evidence about how the predictor
must change under the new concept definition.

At the same time, the framework retains a smaller subset
\(\mathcal{D}_t^{\mathrm{stable}}\subseteq
\mathcal{D}_t^{\mathrm{safe}}\), sampled from the certified stable
region. The purpose of
\(\mathcal{D}_t^{\mathrm{stable}}\) is not to teach new semantics, but
to preserve previously correct predictive behavior on regions that are
known to remain unchanged. This is an important distinction. Under
rule-induced concept shift, the goal is not indiscriminate preservation
of all prior knowledge, because part of that knowledge may have become
obsolete. Instead, preservation should be targeted specifically at
those records and decision regions whose validity has been certified by
the provenance analysis.

To keep the framework general, we deliberately avoid tying the repair
mechanism to a specific predictor family. Let \(f_{\theta_t}\) denote
the deployed model before the concept revision and
\(f_{\theta_{t+1}}\) the repaired model after adaptation. The update is
driven by a simple objective
\(\mathcal{L}=
\mathcal{L}_{\mathrm{repair}}+
\lambda \mathcal{L}_{\mathrm{stable}}\), where
\(\mathcal{L}_{\mathrm{repair}}\) encourages the model to fit the
revised targets in
\(\mathcal{D}_t^{\mathrm{repair}}\), while
\(\mathcal{L}_{\mathrm{stable}}\) constrains the model to preserve
appropriate behavior on
\(\mathcal{D}_t^{\mathrm{stable}}\). The scalar \(\lambda\geq0\)
balances the strength of preservation relative to revision.

The role of \(\mathcal{L}_{\mathrm{repair}}\) is conceptually simple:
it teaches the model which decisions should change after the concept
update. Depending on the predictor family, this term may correspond to
a cross-entropy loss for classification, a logistic objective for
binary risk prediction, or an online update objective for stream-based
models. The role of \(\mathcal{L}_{\mathrm{stable}}\) is equally
important: it prevents unnecessary drift in regions of the input space
whose semantic interpretation remains valid. Without this term, even a
small repair dataset may induce broader parameter movement that harms
performance on stable records.

The resulting learning procedure should be understood as an
\emph{incremental repair} rather than full retraining. The model is not
discarded and relearned from scratch after every concept revision.
Instead, it is selectively adjusted using the smallest supervision set
that reflects the revised concept while preserving stable knowledge
identified through provenance-based certification. This design directly
matches the objective of the paper: to update predictive behavior only
where the validity of previous knowledge has actually changed.

An important advantage of this formulation is that it is predictor
agnostic. The same maintenance pipeline can be instantiated with
gradient boosting models, neural tabular predictors, or online
decision-tree methods. This is strategically important because the core
contribution of the framework lies in how concept evolution is compiled
into data maintenance and targeted model updating, not in the design of
a new task-specific backbone. Demonstrating the same framework across
several predictive families will later support the claim that the
method is broadly applicable rather than architecture dependent.

\subsection{Versioned Concept Memory}
\label{subsec:versioned_memory}

Many deployed systems do not experience a single isolated concept
revision. Instead, concept definitions evolve repeatedly over time and
may even return to previously used states. A robust maintenance
framework should therefore preserve not only the current predictor, but
also a structured record of past concept versions and the information
needed to update between them efficiently.

To support this setting, we maintain a concept history
\(\mathcal{H}_t=\{\mathcal{M}_1,\ldots,\mathcal{M}_t\}\), where each
memory item is
\(\mathcal{M}_t=
\{Q_t,G(Q_t),\Pi_t,\mathcal{D}_t^{\Delta},\theta_t\}\). Here,
\(Q_t\) is the concept definition valid at time \(t\), \(G(Q_t)\) is
its predicate-graph representation, \(\Pi_t\) denotes the associated
provenance or index state, \(\mathcal{D}_t^{\Delta}\) stores the
observed affected-data information generated by the transition into
this concept version, and \(\theta_t\) denotes the model state after
adaptation to that version.

This memory serves several purposes. First, it maintains a compact
representation of how the concept has evolved over time, which is
useful for auditing, reproducibility, and rollback. Second, it allows
the framework to reuse structural knowledge from previous revisions.
For example, previously constructed predicate graphs and provenance
indices can reduce the cost of analyzing a new update. Third, it
supports recurring concept definitions, which are common in practice
when policies are revised temporarily, seasonal definitions recur, or a
later update reverts to a prior rule.

Consider a sequence
\(Q_1\rightarrow Q_2\rightarrow Q_3\rightarrow Q_1\). When the concept
returns to \(Q_1\), the system does not need to treat the new state as
completely unseen. Instead, it can retrieve the corresponding memory
item \(\mathcal{M}_1\) and reuse several types of previously computed
information: the query structure \(G(Q_1)\), the relevant provenance
summaries \(\Pi_1\), earlier affected-data patterns associated with the
same definition, and a compact model state \(\theta_1\) that already
approximates the target induced by \(Q_1\). This ability to reuse prior
knowledge distinguishes versioned concept maintenance from repeatedly
solving each revision as an independent adaptation problem.

The versioned concept memory also provides a coherent place to attach
metadata about each revision, such as timestamps, rule-delta types,
annotation budgets, or storage statistics. Although these auxiliary
elements are not all required in the basic formulation, they are useful
for later experimental analysis, particularly when comparing update
latency, provenance overhead, and recurring-definition recovery across
multiple revision sequences.

\section{Theoretical Properties}
\label{sec:theory}

We next establish several properties of the proposed maintenance
procedure. The analysis addresses three questions that are central to
the framework: whether the changed-label set can be recovered exactly
when the concept definitions are executable, when provenance analysis
can safely certify that a historical record does not require
reevaluation, and under what conditions the resulting incremental
procedure is computationally preferable to full recomputation. To keep
the main manuscript compact, we provide proof sketches here and defer
the complete proofs and additional cases to the supplementary
material. For the formal analysis, we consider deterministic concept programs
represented by finite directed acyclic graphs. Each leaf evaluates an
atomic, threshold, relational, or graph predicate, while internal nodes
compose their outputs through deterministic operators. The analysis
below is stated for the Boolean rule language used by the principal
benchmark protocols; multiclass definitions can be represented through
multiple decision outputs or an equivalent deterministic decision
graph.

\subsection{Exactness of Concept Delta}
\label{subsec:delta_exactness}

The first property concerns the relation between the explicit concept
revision and the historical records whose targets actually change. For
a deterministic pair of executable definitions \(Q_t\) and
\(Q_{t+1}\), recall that the affected historical set is
\(\mathcal{D}_t^{\Delta}=
\{z_i\in\mathcal{D}_t:
Q_t(z_i)\neq Q_{t+1}(z_i)\}\). Thus, when both concept versions can be
executed, membership in \(\mathcal{D}_t^{\Delta}\) is determined
directly by disagreement between their outputs rather than by a
statistical estimate of concept drift.

\begin{proposition}[Exact Concept-Delta Recovery]
\label{prop:exact_delta_recovery}
Let \(Q_t\) and \(Q_{t+1}\) be deterministic executable concept
programs over the historical collection \(\mathcal{D}_t\). Suppose that
\(\mathcal{D}_t^{\mathrm{safe}}\) is a sound stable set, i.e., every
\(z_i\in\mathcal{D}_t^{\mathrm{safe}}\) satisfies
\[
Q_t(z_i)=Q_{t+1}(z_i).
\]
Define
\[
\mathcal{D}_t^{\mathrm{cand}}
=
\mathcal{D}_t\setminus\mathcal{D}_t^{\mathrm{safe}}.
\]
Then evaluating \(Q_t\) and \(Q_{t+1}\) only over
\(\mathcal{D}_t^{\mathrm{cand}}\) recovers exactly the same changed-label
set as evaluating them over the entire historical collection.
\end{proposition}

To see this, soundness of
\(\mathcal{D}_t^{\mathrm{safe}}\) guarantees that no changed-label
record can occur inside the certified stable set. Hence
\(\mathcal{D}_t^{\Delta}\subseteq
\mathcal{D}_t^{\mathrm{cand}}\). Comparing \(Q_t(z_i)\) and
\(Q_{t+1}(z_i)\) for every candidate therefore returns
\(\{z_i\in\mathcal{D}_t^{\mathrm{cand}}:
Q_t(z_i)\neq Q_{t+1}(z_i)\}=
\mathcal{D}_t^{\Delta}\). The candidate restriction can consequently
reduce computation without altering the resulting set of changed
targets, provided that stable-record certification is sound.

This result also clarifies the role of provenance analysis. The
proposed framework does not approximate
\(\mathcal{D}_t^{\Delta}\) by assigning heuristic drift scores to
historical examples. Instead, provenance is used to eliminate records
whose outputs can be proven invariant to the rule revision. Exact
execution of the changed portions of the concept program is then
reserved for the remaining candidates.

\subsection{Provenance Stability Certificate}
\label{subsec:stability_certificate}

The main theoretical question is therefore when a record can be placed
in \(\mathcal{D}_t^{\mathrm{safe}}\) without executing the complete
revised concept program. The answer depends jointly on the structural
rule delta and the evaluation provenance of the record.

For a record \(z_i\), let
\(\operatorname{Prov}(z_i,Q_t)\) denote its evaluation provenance under
the old concept program. In the Boolean rule DAG, each internal node
takes the outputs of its children and produces a deterministic result.
A change originating from a modified predicate can influence the final
concept output only if that change can propagate through a directed
path from the modified component to the root.

We call a persistent logical node \emph{blocking} for a particular
record when one of its unchanged inputs already fixes its output
independently of the changed branch. Specifically, an AND node is
blocking when it has an unchanged child evaluating to false, because
its output remains false regardless of the values of its other
children. Similarly, an OR node is blocking when it has an unchanged
child evaluating to true. These conditions correspond directly to
record-specific provenance: the unchanged child provides a sufficient
witness for the node output.

Let \(\mathcal{C}_t^{\mathrm{eff}}(z_i)\) denote the set of components
identified by the rule delta whose local values or local operators can
differ for \(z_i\) between \(Q_t\) and \(Q_{t+1}\). A threshold change,
for example, belongs to
\(\mathcal{C}_t^{\mathrm{eff}}(z_i)\) only when the record lies in the
interval where the old and revised threshold evaluations can disagree.
A changed component that produces the same local result for \(z_i\)
cannot by itself alter the final concept output.

\begin{theorem}[Provenance Stability Certificate]
\label{thm:provenance_stability}
Consider a deterministic Boolean concept DAG constructed from atomic
predicates and the logical operators AND, OR, and NOT. For a record
\(z_i\), suppose that every directed path from each effective changed
component in \(\mathcal{C}_t^{\mathrm{eff}}(z_i)\) to the concept-output
node contains a persistent blocking node whose decisive input remains
unchanged under the revision. Then the concept assignment is invariant,
\[
Q_t(z_i)=Q_{t+1}(z_i),
\]
and consequently \(z_i\) can be safely included in
\(\mathcal{D}_t^{\mathrm{safe}}\).
\end{theorem}

The intuition follows from how changes propagate through a Boolean
DAG. If \(Q_t(z_i)\neq Q_{t+1}(z_i)\), at least one changed local
evaluation must influence the output node. Such influence requires an
uninterrupted path from an effective changed component to the root. At
an AND node with an unchanged false input, however, the node output is
fixed to false irrespective of the changed branch. Likewise, an OR
node with an unchanged true input remains true independently of the
changed branch. A blocking node therefore stops the propagation of a
local change. If every possible influence path contains such a node,
no effective rule change can reach the output, contradicting the
assumption that the final concept assignment changes. As a simple example, consider
\(Q_t=P_1\land P_2\) and a revision that changes only \(P_2\). For any
record satisfying \(P_1(z_i)=0\), the conjunction remains false
irrespective of how \(P_2\) is revised. The unchanged false evaluation
of \(P_1\) therefore serves as a stability certificate, and such a
record does not require execution of the revised \(P_2\). Conversely,
when \(P_1(z_i)=1\), the changed predicate can determine the final
output and the record must remain in
\(\mathcal{D}_t^{\mathrm{cand}}\).

A similar effect occurs for disjunction. If
\(Q_t=P_1\lor P_2\) and only \(P_2\) changes, then a record satisfying
the persistent condition \(P_1(z_i)=1\) remains positive regardless of
the revised value of \(P_2\). In contrast, when \(P_1(z_i)=0\), the
output depends on the changed branch and cannot be certified without
further evaluation. The theorem provides a stronger guarantee than conventional
uncertainty- or similarity-based data selection. Such methods estimate
whether a record is likely to be affected, whereas the proposed
certificate identifies sufficient conditions under which a record is
provably unaffected. The certificate is deliberately conservative:
failure to certify a record does not imply that its target changes; it
only means that invariance cannot be established from the available
provenance information. Such records remain in
\(\mathcal{D}_t^{\mathrm{cand}}\) and are resolved by selective
reevaluation. The same principle extends beyond simple Boolean leaves. Threshold,
relational, and graph predicates can first be tested for local
invariance under the rule delta. If their local output is unchanged for
a record, they do not enter
\(\mathcal{C}_t^{\mathrm{eff}}(z_i)\); if their local output may change,
their influence is traced upward through the predicate graph using the
same propagation criterion. 

An immediate consequence of Proposition~\ref{prop:exact_delta_recovery} and Theorem~\ref{thm:provenance_stability} is that
provenance-based pruning preserves exact affected-set recovery. Because
the certificate places only truly invariant records in
\(\mathcal{D}_t^{\mathrm{safe}}\), all records whose labels actually
change remain in \(\mathcal{D}_t^{\mathrm{cand}}\), where explicit
evaluation of the revised concept can recover
\(\mathcal{D}_t^{\Delta}\) exactly.

\subsection{Incremental Processing Complexity}
\label{subsec:complexity}

We finally analyze the computational benefit of restricting concept
maintenance to the candidate region. Let \(N=|\mathcal{D}_t|\) denote
the number of historical records and let \(C_Q\) denote the average
cost of executing the complete revised concept program on one record.
A straightforward full-recomputation strategy therefore requires
\(O(NC_Q)\) concept-evaluation work before model retraining. This cost
can be particularly large when \(Q_{t+1}\) includes relational joins,
historical aggregation, external lookups, or multi-hop graph
operations.

The proposed procedure introduces several smaller costs. Let
\(C_{\Delta Q}\) denote the one-time cost of comparing
\(G(Q_t)\) and \(G(Q_{t+1})\) and constructing the typed rule delta,
\(C_{\mathrm{prov}}\) the cost of retrieving and processing provenance
information required for candidate discovery, and \(C_{\Delta}\) the
average cost of evaluating only the changed or unresolved portions of
the revised concept for one candidate record. Ignoring the subsequent
predictor-specific optimization for the moment, the concept-maintenance
cost is therefore
\(O(C_{\Delta Q}+C_{\mathrm{prov}}+
|\mathcal{D}_t^{\mathrm{cand}}|C_{\Delta})\).

The provenance term depends on the indexing strategy. A naive
implementation may inspect provenance metadata for all \(N\) records,
which reduces the achievable systems-level speedup even though complete
query execution is avoided. In the implementation considered here, an
inverted provenance index associates each predicate or rule component
with the records whose evaluation depends on it. Candidate discovery
can then be driven directly by the changed component set
\(\mathcal{C}_t\). If \(\mathcal{I}_t(c)\) denotes the provenance
posting list associated with component \(c\), the lookup work is
proportional to the changed-component postings rather than necessarily
to the complete database, approximately
\(O(|\mathcal{C}_t|+
\sum_{c\in\mathcal{C}_t}|\mathcal{I}_t(c)|)\), excluding duplicate
elimination and index-maintenance overhead.

The total end-to-end update additionally includes predictor repair.
Let \(C_{\mathrm{repair}}\) denote the cost of updating the deployed
predictor from the repair and stable subsets, and
\(C_{\mathrm{index}}\) the cost of maintaining provenance structures
after the revision. The resulting update cost can be summarized as
\(O(C_{\Delta Q}+C_{\mathrm{prov}}+
|\mathcal{D}_t^{\mathrm{cand}}|C_{\Delta}
+C_{\mathrm{repair}}+C_{\mathrm{index}})\). By comparison, a
full-relabel-and-retrain pipeline incurs approximately
\(O(NC_Q+C_{\mathrm{train}}^{\mathrm{full}})\), where
\(C_{\mathrm{train}}^{\mathrm{full}}\) denotes the cost of rebuilding
the predictive model from the complete revised training collection.

\begin{proposition}[Incremental Advantage]
\label{prop:incremental_advantage}
Assume that \(C_{\Delta Q}\), \(C_{\mathrm{prov}}\), and
\(C_{\mathrm{index}}\) are asymptotically smaller than \(NC_Q\), that
\(C_{\Delta}\leq C_Q\), and that
\[
\frac{|\mathcal{D}_t^{\mathrm{cand}}|}{N}\rightarrow 0
\]
as the historical collection grows. Then the concept-recomputation cost
of the proposed procedure is asymptotically smaller than the cost of
full concept recomputation.
\end{proposition}
The result follows directly because the dominant record-level term
changes from \(NC_Q\) to
\(|\mathcal{D}_t^{\mathrm{cand}}|C_{\Delta}\). Defining the candidate
ratio as
\(\rho_t=|\mathcal{D}_t^{\mathrm{cand}}|/N\), the record-level
recomputation ratio is approximately
\(\rho_t C_{\Delta}/C_Q\). Consequently, the largest computational
advantage is expected when the concept revision is localized, the
provenance index is selective, and evaluating the changed rule fragment
is substantially cheaper than executing the complete concept program.

This analysis also exposes an important operating limit of the method.
When the revision is global and
\(|\mathcal{D}_t^{\mathrm{cand}}|\approx N\), or when nearly every
record depends on the changed predicates, the advantage of selective
maintenance naturally decreases. In the limiting case where every
record must be reevaluated and \(C_{\Delta}\approx C_Q\), the
record-processing cost approaches that of full recomputation. This
behavior is expected rather than pathological: selective maintenance
is most beneficial precisely when the semantic change is localized
relative to the historical data.

\section{Experiments}
\label{sec:experiments}

We evaluate the proposed framework from three complementary
perspectives: predictive performance under evolving concept
definitions, efficiency in identifying and reprocessing affected
records, and robustness across different predictive model families.
The experiments use the proposed RuleShift-Bench, which combines
financial, demographic, cybersecurity, and graph-structured data with
controlled concept-rule revisions. Unless otherwise stated, each
experiment is repeated with three random seeds and stochastic results
are reported as mean$\pm$standard deviation.

\subsection{RuleShift-Bench}
\label{subsec:ruleshift_bench}

\textbf{PaySim.}
For financial transactions, we use PaySim, which contains
6,362,620 synthetic mobile-money transactions spanning 744 hourly
steps
\cite{bench_khan2024,bench_vijayanand2025}.
The released data contain 11 fields; we use \texttt{isFraud} as
the target and the remaining 10 fields as the raw predictor variables.
Concept versions represent evolving
fraud-risk policies through amount-threshold changes, insertion or
removal of transaction conditions, logical recombination of risk
predicates, recurring rules, and mixed rule--temporal distribution
shift.

\textbf{Census-Income.}
For policy-oriented concept evolution, we use the Census-Income (KDD)
dataset containing 299,285 records from the 1994--1995 U.S. Current
Population Surveys and 40 demographic and employment attributes
\cite{bench_yin2024,bench_chen2024}.
We construct changing eligibility definitions
from income, age, employment, education, household, and work-related
predicates. This dataset supports threshold revisions, predicate
addition and removal, logical restructuring, recurrence, and mixed
concept--population shift.

\textbf{UNSW-NB15.}
For cybersecurity, we use the complete UNSW-NB15 collection containing
2,540,044 network-flow records
\cite{bench_zoghi2024}.
The released representation contains
49 fields, including the attack category and binary attack label,
leaving 47 predictive fields when these two supervision variables are
excluded.
Revised security concepts are
constructed from flow, protocol, service, byte-volume, state, and
relational endpoint conditions, allowing us to study changing
risk thresholds, predicate compositions, and network-relation rules.

\textbf{ogbn-arxiv.}
For graph-structured concepts, we use \texttt{ogbn-arxiv}, containing
169,343 paper nodes and 1,166,243 directed citation edges
\cite{bench_su2024}.
Each node is
represented by a 128-dimensional feature vector and is associated with
publication time and one of 40 subject areas.
Concept definitions combine node attributes,
publication time, citation relations, neighborhood properties, and
multi-hop graph conditions, making this dataset particularly suitable
for evaluating relational and graph-path revisions.

We instantiate eight concept versions
\(Q_0,\ldots,Q_7\) for each benchmark family: one initial definition
followed by one controlled instance of each shift type R1--R7.
Threshold values are determined from the training partition only; the
default threshold transition changes the selected predicate from its
75th to its 60th percentile. R2 inserts one previously unused
predicate, R3 removes one active predicate, R4 replaces conjunction by
disjunction or vice versa, R5 modifies a relation or graph-path
condition, R6 evaluates \(Q_1\rightarrow Q_2\rightarrow Q_1\), and R7
combines the corresponding rule update with a temporal or
distributional change in the data (see Table \ref{tab:benchmark_stats}).

\begin{table}[htbp]
\centering
\caption{Principal datasets used in RuleShift-Bench. Rule versions
include the initial definition and seven subsequent rule-shift
protocols.}
\label{tab:benchmark_stats}
\setlength{\tabcolsep}{2.7pt}
\renewcommand{\arraystretch}{0.95}
\footnotesize
\resizebox{\columnwidth}{!}{%
\begin{tabular}{|l|r|c|c|c|c|c|c|}
\hline
\textbf{Dataset} &
\textbf{Records} &
\textbf{Features} &
\textbf{Versions} &
\textbf{Thresh.} &
\textbf{Pred.} &
\textbf{Logic} &
\textbf{Relation} \\
\hline
PaySim &
6,362,620 &
10 &
8 &
\checkmark &
\checkmark &
\checkmark &
-- \\
\hline
Census-Income &
299,285 &
40 &
8 &
\checkmark &
\checkmark &
\checkmark &
-- \\
\hline
UNSW-NB15 &
2,540,044 &
47 &
8 &
\checkmark &
\checkmark &
\checkmark &
\checkmark \\
\hline
ogbn-arxiv &
169,343 &
128 &
8 &
\checkmark &
\checkmark &
\checkmark &
\checkmark \\
\hline
\end{tabular}}
\end{table}

The seven benchmark transformations are summarized in
Table~\ref{tab:rule_shifts}. R1--R4 modify attribute-level rule
structure, R5 specifically evaluates relational dependence, R6 tests
whether a previous definition can be recovered after an intermediate
revision, and R7 evaluates the more difficult setting in which concept
evolution and data-distribution change occur simultaneously.

\begin{table}[htbp]
\centering
\caption{Rule-induced concept-shift protocols.}
\label{tab:rule_shifts}
\setlength{\tabcolsep}{3.0pt}
\renewcommand{\arraystretch}{0.95}
\footnotesize
\begin{tabular}{|c|l|l|}
\hline
\textbf{ID} & \textbf{Shift} & \textbf{Example} \\
\hline
R1 & Threshold revision &
\(x>q_{0.75}\rightarrow x>q_{0.60}\) \\
\hline
R2 & Predicate insertion &
\(A\rightarrow A\land B\) \\
\hline
R3 & Predicate deletion &
\(A\land B\rightarrow A\) \\
\hline
R4 & Logical rewrite &
\(A\lor B\rightarrow A\land B\) \\
\hline
R5 & Relational rewrite &
\(R_1\rightarrow R_2\) / path change \\
\hline
R6 & Recurring definition &
\(Q_1\rightarrow Q_2\rightarrow Q_1\) \\
\hline
R7 & Mixed shift &
rule change + data shift \\
\hline
\end{tabular}
\end{table}

\subsection{Experimental Setup}
\label{subsec:experimental_setup}

We evaluate three predictor families to verify that the proposed
maintenance mechanism is not tied to a particular learning
architecture. The first is XGBoost with 300 trees, maximum depth 6,
and learning rate 0.05. The second is a neural tabular predictor with
three hidden layers of dimensions 256, 128, and 64, dropout 0.2, and
AdamW optimization with learning rate \(10^{-3}\). The third is a
Hoeffding Adaptive Tree for streaming evaluation. The same predictor
configuration is retained across concept versions unless explicitly
stated otherwise.

PaySim and the complete UNSW-NB15 collection are ordered using their
temporal information and divided into 70\% training, 10\% validation,
and 20\% test partitions. Census-Income uses its published
199,523-record training partition and 99,762-record test partition,
with 10\% of the original training set reserved for validation.
For \texttt{ogbn-arxiv}, we retain the official temporal split: papers
published through 2017 are used for training, papers from 2018 for
validation, and papers from 2019 onward for testing. All rule
thresholds, predicate statistics, and provenance structures are
constructed using training data only.

\subsection{Comparison Methods}
\label{subsec:comparison_methods}

We compare against several complementary adaptation strategies.
\emph{Old Predictor} retains \(f_{\theta_t}\) without adaptation.
\emph{Full Relabel and Retrain} executes \(Q_{t+1}\) over all available
historical records and retrains the predictor from the complete revised
dataset, providing the principal computational upper reference.
\emph{Sliding Window} updates the model using the most recent 20\% of
training records, while \emph{Online Update} processes revised examples
sequentially
\cite{cmp_hu2024,cmp_stevanoski2024}
and \emph{Replay} combines the incoming revision data with
a memory containing 5\% of the original training records
\cite{cmp_lin2025,cmp_qiu2025}.
We further
include an ADWIN-triggered adaptation pipeline as a representative
drift-detection baseline.

For data-selection comparisons, \emph{Random Reevaluation} processes
the same number of historical records as the proposed candidate set but
selects them uniformly at random, while \emph{Uncertainty Reevaluation}
selects records with the highest predictive uncertainty
\cite{cmp_fajri2024,cmp_hoarau2024}.
\emph{Provenance Selection + Full Retraining} uses the candidate set
identified by provenance but subsequently performs conventional global
model retraining, thereby isolating the contribution of incremental
predictor repair. Finally, \emph{Oracle Affected Set} receives the true
set \(\mathcal{D}_t^{\Delta}\) and therefore represents the upper
reference for affected-data identification. In particular, comparison
with \emph{Full Relabel and Retrain} measures the computational benefit
relative to complete maintenance, whereas comparison with
\emph{Oracle Affected Set} quantifies the remaining gap caused by
candidate-set discovery.

\subsection{Main Results}
\label{subsec:main_results}

We first evaluate whether provenance-guided maintenance can preserve the
predictive quality of complete recomputation while substantially
reducing the amount of historical data that must be reconsidered.
Table~\ref{tab:main_results} reports accuracy, Macro-F1, affected
recall, the percentage of historical records reprocessed, annotation
demand, and end-to-end update latency across PaySim, Census-Income,
UNSW-NB15, and \texttt{ogbn-arxiv}, averaged over rule revisions
R1--R5.

\begin{table*}[htbp]
\centering
\caption{Overall comparison across the four RuleShift-Bench data
families under rule revisions R1--R5.}
\label{tab:main_results}
\setlength{\tabcolsep}{5.0pt}
\renewcommand{\arraystretch}{1.05}
\footnotesize
\begin{tabular}{|l|c|c|c|c|c|c|}
\hline
\rowcolor{gray!8}
\textbf{Method}
&
\textbf{Acc. (\%)}
&
\textbf{Macro-F1 (\%)}
&
\textbf{Affected Recall (\%)}
&
\textbf{Reprocessed (\%)}
&
\textbf{Labels}
&
\textbf{Time (s)}
\\
\hline

Old Predictor
& 82.7 & 78.9 & 0.0 & 0.0 & 0 & 0.0 \\
\hline

Full Relabel + Retrain
& \textbf{92.8} & \textbf{90.7} & 100.0 & 100.0 & 0 & 993 \\
\hline

Sliding Window
& 88.5 & 85.8 & 63.2 & 20.0 & 0 & 389 \\
\hline

Online Update
& 87.9 & 85.1 & 68.5 & 12.5 & 0 & 241 \\
\hline

Replay
& 89.2 & 86.7 & 74.6 & 17.5 & 0 & 325 \\
\hline

ADWIN + Retrain
& 87.6 & 84.9 & 66.0 & 28.4 & 0 & 441 \\
\hline

Random Reevaluation
& 88.8 & 85.9 & 71.1 & 14.7 & 0 & 210 \\
\hline

Uncertainty Reevaluation
& 90.1 & 87.6 & 79.0 & 14.7 & 0 & 226 \\
\hline

Provenance Selection + Full Retraining
& 92.1 & 89.9 & 94.6 & 14.7 & 0 & 515 \\
\hline

\rowcolor{green!6}
\textbf{Provenance-Guided Repair}
& \textbf{92.3}
& \textbf{90.2}
& \textbf{94.6}
& \textbf{14.7}
& 0
& \textbf{179}
\\
\hline

\rowcolor{blue!5}
Oracle Affected Set
& 92.6 & 90.5 & 100.0 & 10.8 & 0 & 154 \\
\hline

\end{tabular}
\end{table*}

Provenance-Guided Repair achieves 92.3\% accuracy and 90.2\%
Macro-F1, remaining within 0.5 percentage points of Full Relabel +
Retrain while processing only 14.7\% of the historical collection.
The update time decreases from 993\,s to 179\,s, corresponding to an
approximately \(5.5\times\) reduction in latency. At the same
reprocessing budget, the proposed method also retains 94.6\% of the
affected records, compared with 79.0\% for Uncertainty Reevaluation
and 71.1\% for Random Reevaluation. The comparison with Provenance Selection + Full Retraining isolates
the effect of incremental predictor repair. Both methods identify the
same candidate region and obtain 94.6\% affected recall, whereas the
incremental update reduces processing time from 515\,s to 179\,s.
The Oracle Affected Set reaches 90.5\% Macro-F1 while processing
10.8\% of the historical data, leaving only a small margin between
provenance-guided candidate discovery and exact knowledge of the
affected region.

Table~\ref{tab:dataset_efficiency} further separates the comparison by
benchmark. The predictive difference between complete recomputation
and incremental repair remains small across all four data families,
whereas the reduction in update cost is substantially larger.

\begin{table}[htbp]
\centering
\caption{Dataset-wise comparison between complete recomputation and
provenance-guided incremental repair.}
\label{tab:dataset_efficiency}
\setlength{\tabcolsep}{5.3pt}
\renewcommand{\arraystretch}{1.05}
\footnotesize
\resizebox{\columnwidth}{!}{%
\begin{tabular}{|l|r|c|c|c|c|c|}
\hline
\rowcolor{gray!8}
\textbf{Dataset}
&
\textbf{Records}
&
\textbf{Full F1}
&
\textbf{Repair F1}
&
\textbf{Reprocessed}
&
\textbf{Full Time}
&
\textbf{Repair Time}
\\
\hline

PaySim
& 6,362,620
& 91.4
& 90.9
& 12.1\%
& 1775 s
& 181 s
\\
\hline

Census-Income
& 299,285
& 89.2
& 88.7
& 16.3\%
& 154 s
& 28 s
\\
\hline

UNSW-NB15
& 2,540,044
& 92.6
& 92.2
& 13.8\%
& 903 s
& 148 s
\\
\hline

\texttt{ogbn-arxiv}
& 169,343
& 89.5
& 89.0
& 16.6\%
& 742 s
& 121 s
\\
\hline

\end{tabular}}
\end{table}

On PaySim, update latency decreases from 1775\,s to 181\,s while
Macro-F1 changes from 91.4\% to 90.9\%. Census-Income reduces the
update from 154\,s to 28\,s, and UNSW-NB15 from 903\,s to 148\,s.
For \texttt{ogbn-arxiv}, the update decreases from 742\,s to 121\,s
while maintaining a 0.5-point Macro-F1 difference. These results show
that selective concept maintenance remains effective across
transactional, demographic, network, and graph-structured data.

\subsection{Scalability and Computational Efficiency}
\label{subsec:scalability}

We next examine the computational and predictive behavior of the
different maintenance strategies. Figure~\ref{fig:scalability}
summarizes dataset-level latency, predictive performance, and
affected-data coverage using complementary bar-chart views.

Figure~\ref{fig:scalability}(a) compares the update latency of complete
recomputation and provenance-guided repair across the four datasets.
The largest reduction occurs on PaySim, where the update decreases by
approximately \(9.8\times\). Census-Income obtains a \(5.5\times\)
reduction, while UNSW-NB15 and \texttt{ogbn-arxiv} both obtain
approximately \(6.1\times\). The graph benchmark remains relatively
costly despite its smaller number of records because relational and
path-based concept evaluation requires additional structural
processing.

Figure~\ref{fig:scalability}(b) compares Accuracy and Macro-F1 across
all adaptation strategies. Full Relabel + Retrain provides the
strongest overall predictive reference, while Provenance-Guided Repair
remains close to it and also approaches the Oracle Affected Set.
Conventional online, window-based, replay, and data-selection methods
show larger predictive degradation because they do not explicitly use
the dependencies induced by the revised concept definition. Figure~\ref{fig:scalability}(c) directly compares affected recall and
the fraction of historical data reprocessed. This comparison
highlights the main efficiency property of the framework. Full
recomputation obtains complete affected coverage by processing the
entire database, whereas Provenance-Guided Repair reaches 94.6\%
affected recall while processing only 14.7\%. Random and uncertainty
reevaluation operate at the same 14.7\% data budget but recover only
71.1\% and 79.0\% of affected records, respectively. The difference
shows that provenance-based selection concentrates computation more
effectively on records whose semantic validity depends on the revised
rule.

\begin{figure*}[htbp]
\centering

\subfloat[Dataset latency.]{
    \includegraphics[width=0.315\textwidth]
    {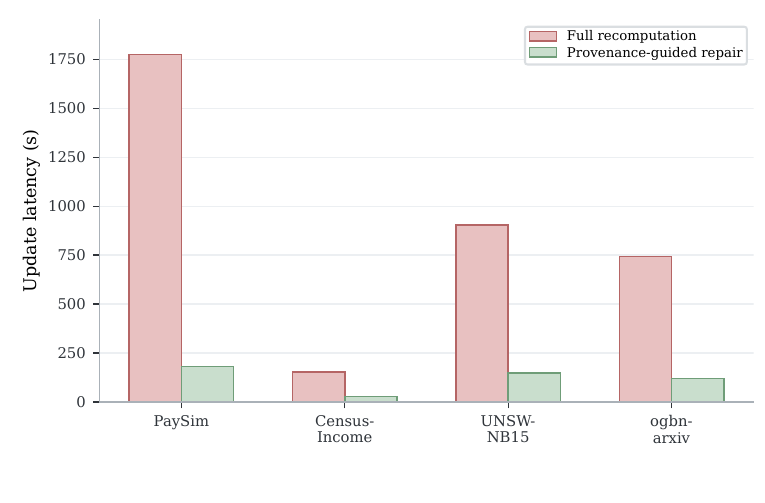}
    \label{fig:dataset_latency}
}
\hfill
\subfloat[Predictive performance.]{
    \includegraphics[width=0.315\textwidth]
    {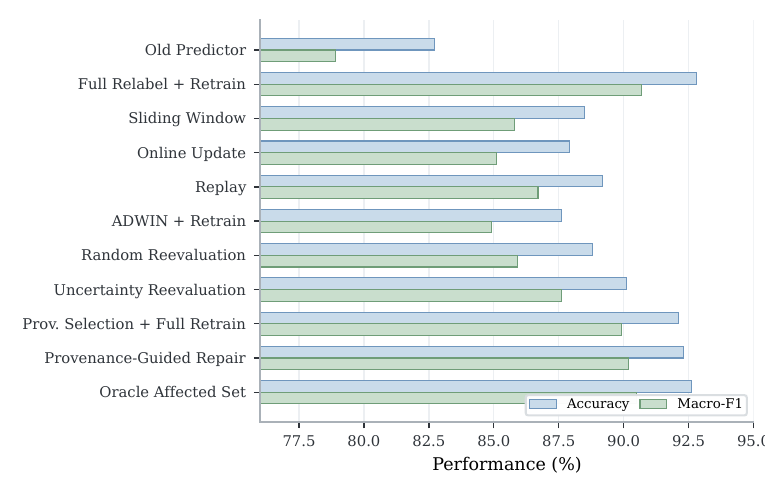}
    \label{fig:predictive_performance}
}
\hfill
\subfloat[Affected coverage.]{
    \includegraphics[width=0.315\textwidth]
    {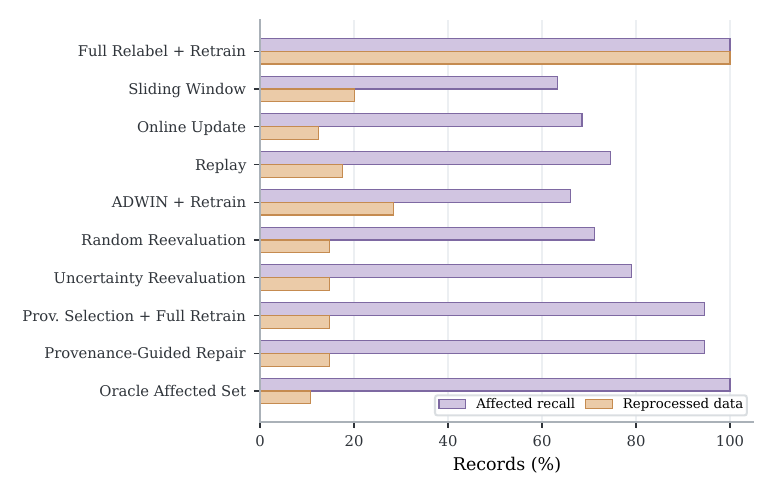}
    \label{fig:affected_coverage}
}

\caption{Predictive and computational behavior of the evaluated
maintenance strategies across the RuleShift-Bench datasets.}
\label{fig:scalability}
\end{figure*}

\subsection{Annotation Efficiency and Recurring Concepts}
\label{subsec:annotation_recurrence}

We next evaluate ambiguous concept revisions for which some updated
targets cannot be obtained directly through deterministic rule
execution. The annotation budget is varied as
\(B\in\{1,2,4,8,16,32\}\), and we compare random sampling, global
predictive uncertainty, and uncertainty restricted to the ambiguous
concept-delta region.

As shown in Table~\ref{tab:annotation_budget} and
Fig.~\ref{fig:annotation_results}(a), delta-constrained uncertainty
provides the strongest performance throughout the annotation range.
With only one annotation, Macro-F1 reaches 78.6\%, compared with
75.2\% for global uncertainty and 72.4\% for random sampling. At
\(B=32\), the corresponding values are 89.2\%, 85.8\%, and 83.5\%.
Restricting annotation to the ambiguous region therefore uses the
available supervision more effectively than selecting uncertain
examples globally.
\begin{table}[htbp]
\centering
\caption{Macro-F1 across different annotation budgets for ambiguous
concept revisions.}
\label{tab:annotation_budget}
\setlength{\tabcolsep}{3.6pt}
\renewcommand{\arraystretch}{1.02}
\footnotesize
\begin{tabular}{|l|c|c|c|c|c|c|}
\hline
\rowcolor{gray!8}
\textbf{Selection}
& \textbf{1}
& \textbf{2}
& \textbf{4}
& \textbf{8}
& \textbf{16}
& \textbf{32}
\\
\hline

Random
& 72.4 & 74.6 & 77.1 & 79.8 & 82.1 & 83.5 \\
\hline

Global uncertainty
& 75.2 & 77.9 & 80.6 & 83.1 & 84.9 & 85.8 \\
\hline

\rowcolor{green!6}
\textbf{Delta-constrained uncertainty}
& \textbf{78.6}
& \textbf{81.7}
& \textbf{84.4}
& \textbf{86.9}
& \textbf{88.4}
& \textbf{89.2}
\\
\hline
\end{tabular}
\end{table}
The final experiment considers recurring concept definitions through
the sequence \(Q_1\rightarrow Q_2\rightarrow Q_3\rightarrow Q_1\).
When \(Q_1\) returns, the versioned concept memory reuses the
corresponding rule representation, provenance information,
affected-data state, and compact predictive state (see Table \ref{tab:recurrence}).

\begin{table}[htbp]
\centering
\caption{Recovery under the recurring concept sequence
\(Q_1\!\rightarrow\!Q_2\!\rightarrow\!Q_3\!\rightarrow\!Q_1\).}
\label{tab:recurrence}
\setlength{\tabcolsep}{3.4pt}
\renewcommand{\arraystretch}{1.02}
\footnotesize
\resizebox{\columnwidth}{!}{%
\begin{tabular}{|l|c|c|c|c|c|}
\hline
\rowcolor{gray!8}
\textbf{Method}
&
\textbf{\(Q_1\) First}
&
\textbf{\(Q_1\) Return}
&
\textbf{\(RG\)}
&
\textbf{Reprocessed}
&
\textbf{Latency}
\\
\hline

Full Retrain
& 91.4 & 84.8 & 6.6 & 100.0\% & 998 s \\
\hline

Replay
& 91.2 & 88.7 & 2.5 & 18.0\% & 315 s \\
\hline

\rowcolor{green!6}
\textbf{Versioned Repair}
& \textbf{91.5}
& \textbf{90.6}
& \textbf{0.9}
& \textbf{12.6\%}
& \textbf{146 s}
\\
\hline

\rowcolor{blue!5}
Oracle Version Reuse
& 91.5 & 91.1 & 0.4 & 8.7\% & 119 s \\
\hline
\end{tabular}}
\end{table}

Versioned Repair recovers 90.6\% accuracy when \(Q_1\) returns,
compared with 91.5\% during its first occurrence, producing a recovery
gap of 0.9 percentage points. Replay reaches 88.7\% with a larger
2.5-point gap. Versioned Repair also reduces recovery latency from
315\,s to 146\,s relative to Replay while decreasing the reprocessed
fraction from 18.0\% to 12.6\%. Figure~\ref{fig:annotation_results}(b) separately compares the first
and returning \(Q_1\) accuracy, making the effect of concept recurrence
visible without combining heterogeneous quantities on the same axis.
Figure~\ref{fig:annotation_results}(c) reports the corresponding
maintenance cost. Full Retrain defines the complete-processing
reference, whereas Versioned Repair substantially reduces both
reprocessing and recovery latency and remains close to Oracle Version
Reuse.

\begin{figure*}[htbp]
\centering

\subfloat[Annotation budget.]{
    \includegraphics[width=0.315\textwidth]
    {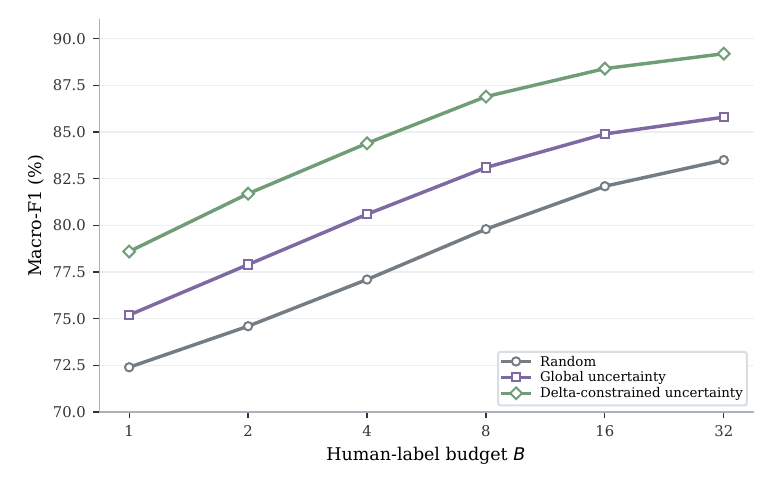}
    \label{fig:annotation_efficiency}
}
\hfill
\subfloat[Recurrence accuracy.]{
    \includegraphics[width=0.315\textwidth]
    {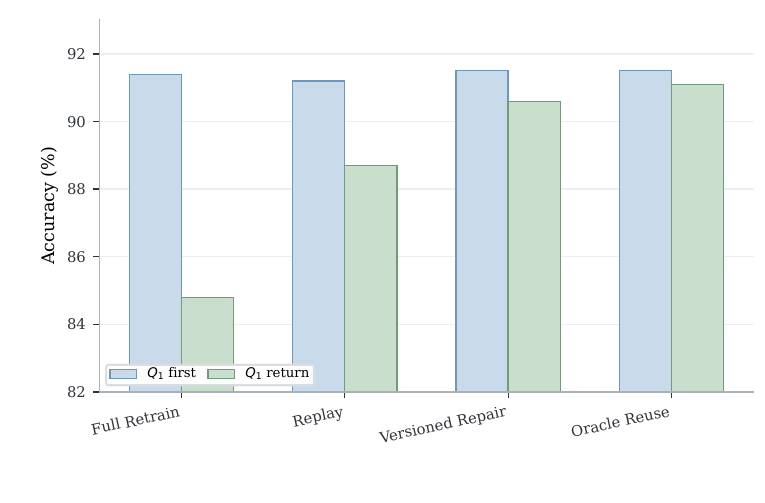}
    \label{fig:recurrence_accuracy}
}
\hfill
\subfloat[Recovery cost.]{
    \includegraphics[width=0.315\textwidth]
    {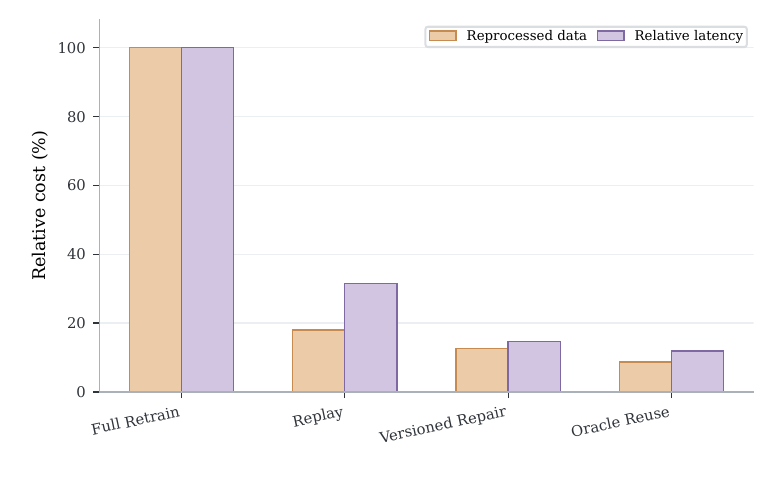}
    \label{fig:recurrence_cost}
}

\caption{Annotation efficiency and concept recurrence under limited
supervision and repeated concept definitions.}
\label{fig:annotation_results}
\end{figure*}

\subsection{Ablations and Failure Analysis}
\label{subsec:ablation_failure}

We next study the contribution of the main components of the proposed
framework through a compact ablation analysis. Starting from the full
system, we remove one component at a time and keep the remaining
pipeline unchanged. The evaluated components are the rule-delta
compiler, provenance analysis, stability certification,
ambiguous-region selection, versioned memory, and incremental repair. Figure~\ref{fig:ablation} provides a visual view of the results. Figure~\ref{fig:ablation_heatmap} summarizes the impact
of each removed component across all four metrics at once, making the
different roles of the modules immediately visible. The predictive
effect is especially clear in
Fig.~\ref{fig:ablation_retention}, where removing the rule-delta
compiler or ambiguous-region selection leads to the largest reduction
in retained predictive quality. In contrast,
Fig.~\ref{fig:ablation_processing} and
Fig.~\ref{fig:ablation_latency} show that provenance, stability
certification, and incremental repair are the most important
components for limiting historical processing and update time.

The ablation results show that the different components affect the
framework in different ways. Removing the rule-delta compiler reduces
Macro-F1 from 90.2\% to 89.3\% and decreases affected recall from
94.6\% to 88.1\%, while more than doubling the amount of historical
data that must be reprocessed. This indicates that explicitly
extracting the structural difference between \(Q_t\) and \(Q_{t+1}\)
is important for narrowing the update region before downstream
maintenance is performed. The provenance component has the strongest effect on efficiency. When
provenance is removed, affected recall reaches 100.0\%, but this is
obtained by effectively revisiting the entire historical collection,
which raises reprocessed data to 100.0\% and update latency to
862\,s. This behavior confirms that provenance is the main mechanism
that allows the framework to identify which historical records are
actually connected to the revised rule components.

Stability certification mainly reduces unnecessary reevaluation.
Without it, affected recall remains high at 98.4\%, but the system
must reprocess 38.9\% of the historical data instead of 14.7\%. This
shows that the certificate is valuable not because it changes the
updated concept itself, but because it safely excludes records whose
previous assignments remain valid. Ambiguous-region selection has a stronger influence on predictive
quality. Removing it reduces Macro-F1 to 88.7\%, the lowest among all
ablations, while also decreasing affected recall to 92.0\%. This
result highlights the importance of directing supervision and
reevaluation toward the unresolved region of the concept revision
rather than treating all candidate records uniformly.

Versioned memory mainly affects long-lived maintenance efficiency.
When it is removed, Macro-F1 decreases to 89.2\%, while reprocessed
data rises to 24.6\% and update latency rises to 268\,s. This
suggests that storing previous concept states, affected-data
information, and compact model states improves the reuse of previously
encountered semantic configurations. Finally, removing incremental repair leaves the selected candidate
region unchanged, so affected recall remains at 94.6\% and the
reprocessed fraction stays at 14.7\%. However, update latency rises
from 179\,s to 515\,s because the predictor must be rebuilt more
heavily after the affected data are identified. This confirms that
candidate discovery and selective model repair provide complementary
benefits.

\begin{figure}[htbp]
\centering
\resizebox{1\columnwidth}{!}{%
\begin{minipage}{\textwidth}
\centering

\subfloat[Sensitivity map.]{
    \includegraphics[width=0.47\textwidth]
    {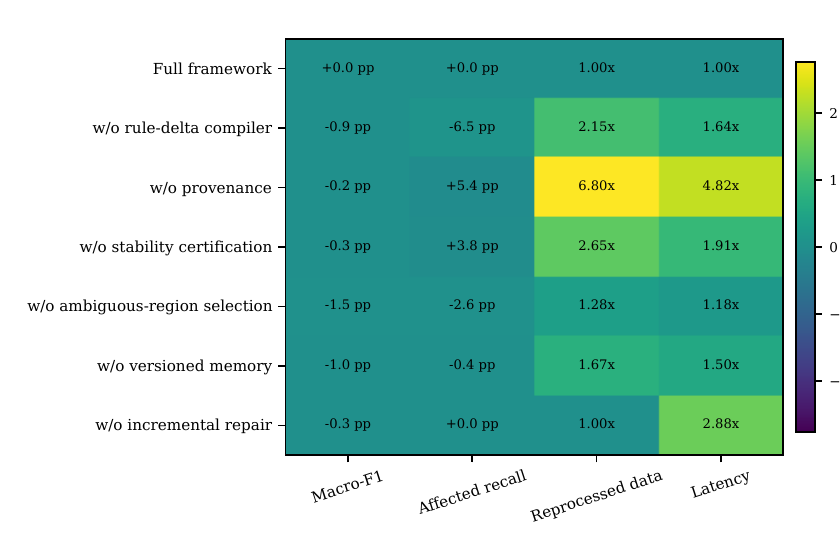}
    \label{fig:ablation_heatmap}
}
\hfill
\subfloat[Predictive retention.]{
    \includegraphics[width=0.47\textwidth]
    {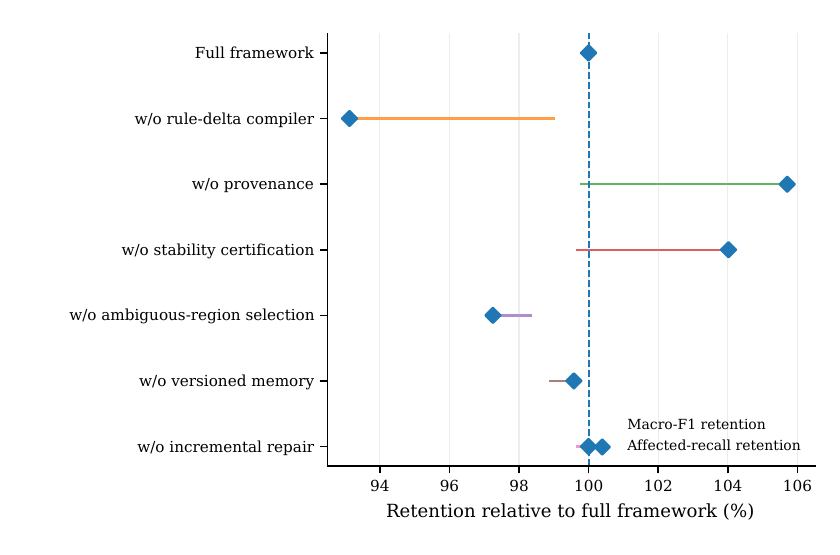}
    \label{fig:ablation_retention}
}

\vspace{1mm}

\subfloat[Processing amplification.]{
    \includegraphics[width=0.47\textwidth]
    {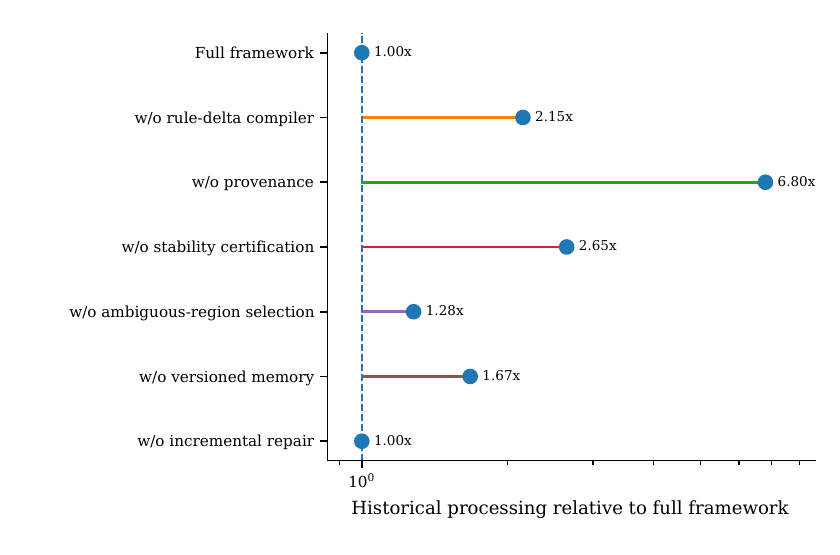}
    \label{fig:ablation_processing}
}
\hfill
\subfloat[Latency range.]{
    \includegraphics[width=0.47\textwidth]
    {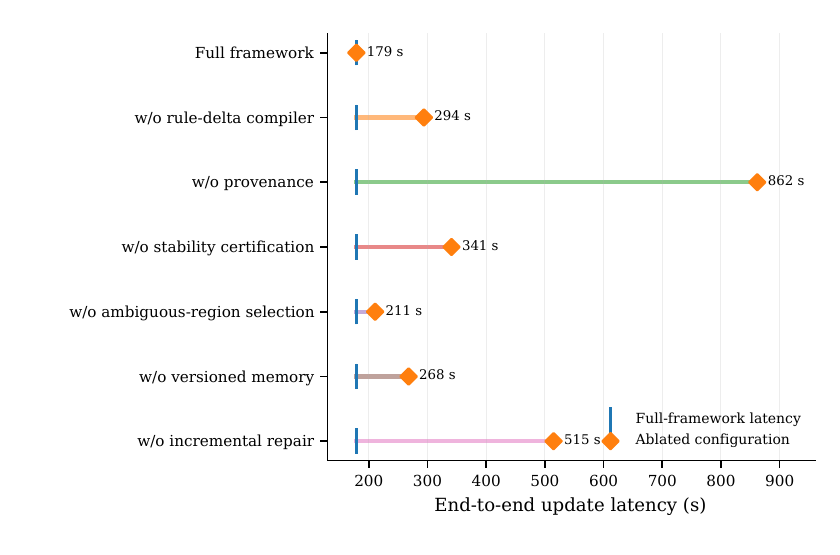}
    \label{fig:ablation_latency}
}

\end{minipage}%
}

\caption{Ablation analysis of the proposed framework from predictive
and system perspectives.}
\label{fig:ablation}
\end{figure}

Despite its advantages, the framework has several identifiable failure modes that delimit the conditions under which selective maintenance remains most effective.

\textbf{Global concept revision.}
The main computational advantage of the framework comes from the fact
that only part of the historical collection is usually affected by a
concept revision. When a revision changes the label of most records,
the candidate set becomes close to the full dataset and selective
maintenance naturally approaches complete recomputation. In this
regime, the framework still remains correct, but its computational
advantage becomes smaller.

\textbf{Missing provenance.}
In some practical systems, historical lineage information may be
incomplete or unavailable. Without reliable provenance, the system
must rely on a more conservative candidate set or reconstruct
dependencies approximately from the stored query structure and raw
data. As shown by the ablation study, this mainly increases the amount
of historical processing required for the update.

\textbf{Non-executable concepts.}
Some concept definitions cannot be expressed entirely as deterministic
predicates. They may depend on expert interpretation, incomplete
knowledge, or learned semantic conditions. In this situation, a larger
fraction of the candidate region is transferred to the ambiguous set
and requires selective human supervision. The annotation study in the
previous subsection shows that directing this supervision to the
ambiguous concept-delta region remains effective.

\textbf{Highly coupled queries.}
When concept definitions involve broad joins, dense relational
dependencies, or long graph paths, a small predicate revision can
influence a large portion of the historical database. In such cases,
fewer records can be certified stable and the candidate region
becomes larger. The framework still applies, but the amount of
reprocessed data and the update cost increase as the dependency
structure becomes more global.
\vspace{-0.3cm}
\section{Conclusion}
\label{sec:conclusion}

This paper studied learning under evolving concept definitions, where the rule, policy, or query generating the target is explicitly revised after deployment. We introduced a provenance-guided incremental learning framework that analyzes the structural difference between consecutive concept definitions, traces this change through historical provenance, certifies stable records, selectively reevaluates potentially affected instances, and incrementally repairs the deployed predictor while preserving valid knowledge. RuleShift-Bench evaluates this setting across financial, demographic, cybersecurity, and graph-structured data, and the results show that Provenance-Guided Repair reaches 92.3\% accuracy and 90.2\% Macro-F1 while reprocessing only 14.7\% of the historical collection, retaining 94.6\% of affected records, and reducing average update latency from 993\,s for complete relabeling and retraining to 179\,s. The framework also improves annotation efficiency for ambiguous revisions and supports recurring concepts through versioned memory. Its main limitations arise when concept revisions are global, historical provenance is incomplete, concept definitions are not fully executable, or highly coupled relational and graph dependencies enlarge the candidate region, in which cases the benefit of selective maintenance decreases. Future work will therefore investigate provenance reconstruction for legacy systems, richer probabilistic and temporal concept languages, more expressive relational and graph constraints, adaptive optimization of provenance granularity and annotation cost, scalable compression and retrieval of long concept histories, and unified treatment of explicit rule revisions together with latent statistical drift. Overall, the results support a data-maintenance perspective on adaptive learning in which an explicitly known concept revision is used to determine which historical supervision and predictive state actually require reconsideration rather than repeatedly recomputing the complete learning system.

\section*{Acknowledgment}
I would like to thank Prof. Yassine Maleh for his valuable discussions and insightful comments during the preparation of this work.

\begin{IEEEbiography}
[{\includegraphics[width=1in,height=1.5in,clip,keepaspectratio]{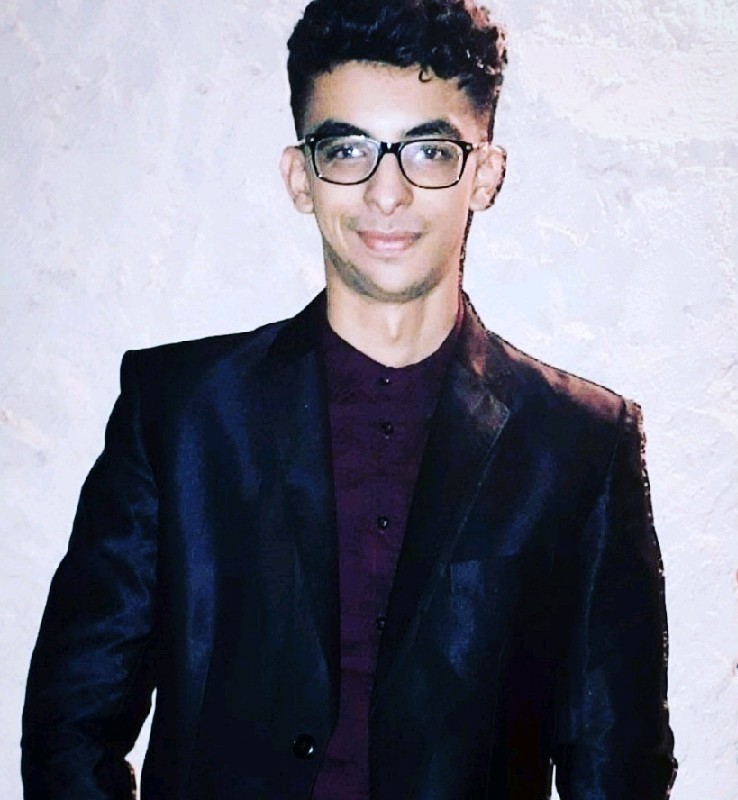}}]
{Ismail Lamaakal} (Student Member, IEEE) received the M.S. degree in computer science from the Multidisciplinary Faculty of Nador, Mohammed First University, Oujda, Morocco. He is currently pursuing the Ph.D. degree in computer science at the same university. His research interests include artificial intelligence, Tiny Machine Learning (TinyML), the Internet of Things (IoT), embedded systems, and edge intelligence. His current research focuses on the design and implementation of resource-efficient machine learning solutions for resource-constrained embedded and IoT devices, with particular emphasis on intelligent, low-power, and real-time applications.
\end{IEEEbiography}

\appendices
\section{Theoretical Analysis and Algorithmic Details}
\label{supp:sec:theoretical_analysis}

The main paper presents the central theoretical properties of the
proposed framework in compact form. This supplementary section develops
those results in greater detail and makes explicit the assumptions,
intermediate definitions, and algorithmic steps required by the
correctness arguments. The analysis is organized around the complete
maintenance path followed by the framework. We begin by defining the
rule language used to express executable concept definitions and the
canonical directed acyclic graph used to represent their computational
structure. We then formalize how two consecutive concept definitions are
aligned and converted into a typed rule delta. This is followed by a
formal treatment of record-level provenance, effective rule changes,
and the construction of the certified stable and candidate regions.
These definitions make it possible to state precisely when a historical
record can be excluded from reevaluation without risking an obsolete
target.

The second part of the section develops the correctness arguments. We
first establish that any sound stable-set construction preserves the
true changed-label set inside the candidate region. We then prove the
provenance stability certificate used to construct such a stable set
for deterministic Boolean concept DAGs. The proof is based on a
difference-propagation argument: if the final concept assignment were
to change, at least one effective rule change would have to propagate
through an uninterrupted dependency path to the concept output.
Persistent AND and OR nodes with unchanged controlling inputs interrupt
such propagation and therefore provide record-specific certificates of
invariance. Finally, we derive the computational cost of each stage of
the maintenance procedure and provide a complete proof of the
incremental-advantage result stated in the main paper.

The exact guarantees in this section apply to deterministic and
executable portions of a concept definition evaluated over a fixed
historical data snapshot. Concept components requiring unavailable
information, uncertain relational evidence, external knowledge, or
expert interpretation are not automatically treated as exact. Such
cases remain unresolved and are handled through the ambiguous-region
and selective-supervision mechanism described in the main paper.


\subsection{Formal Setting and Notation}
\label{supp:subsec:formal_setting}

Let \(\mathcal{Z}\) denote the complete information space maintained by
the underlying data system. A historical entity \(i\) is represented by
\(z_i\in\mathcal{Z}\), which can include ordinary attributes,
relational information, graph neighborhoods, historical aggregates, or
other information available to the concept-definition system. At
revision time \(t\), the historical collection is

\begin{equation}
\mathcal{D}_t
=
\{z_i\}_{i=1}^{N},
\end{equation}

where \(N=|\mathcal{D}_t|\). The previous concept definition is denoted
by \(Q_t\), and the revised concept definition by \(Q_{t+1}\). Both are
interpreted over the same historical snapshot when determining the
effect of the rule revision. Thus, unless explicitly stated otherwise,
the information \(z_i\) itself is assumed unchanged while comparing
\(Q_t(z_i)\) and \(Q_{t+1}(z_i)\).

The true set of historical records whose semantic target changes is

\begin{equation}
\mathcal{D}_t^{\Delta}
=
\left\{
z_i\in\mathcal{D}_t
:
Q_t(z_i)\neq Q_{t+1}(z_i)
\right\}.
\label{supp:eq:true_affected_set}
\end{equation}

Its complement contains the records whose target remains valid under
the revised definition. The purpose of provenance-guided maintenance is
not to approximate the definition of
\(\mathcal{D}_t^{\Delta}\), which is already exact, but to avoid
evaluating the complete revised concept program on records for which
target invariance can already be established from the structural rule
change and the historical provenance.

The framework therefore constructs two operational regions. The
certified stable set is denoted by
\(\mathcal{D}_t^{\mathrm{safe}}\), and the remaining records form the
candidate set

\begin{equation}
\mathcal{D}_t^{\mathrm{cand}}
=
\mathcal{D}_t
\setminus
\mathcal{D}_t^{\mathrm{safe}}.
\label{supp:eq:candidate_definition}
\end{equation}

A central requirement of the theory is that stability certification be
conservative. A record can be placed in
\(\mathcal{D}_t^{\mathrm{safe}}\) only when the framework has a
sufficient certificate that its target cannot change. Failure to obtain
such a certificate does not imply that the target has changed. It means
only that the record must remain in
\(\mathcal{D}_t^{\mathrm{cand}}\) until the revised concept is evaluated
or additional information becomes available.


\subsection{Executable Concept Rule Language}
\label{supp:subsec:rule_language}

The theoretical results require the computational structure of a
concept definition to be explicit. We therefore model an executable
concept as a deterministic Boolean program constructed from primitive
predicates and logical operators. The primitive predicates may
themselves depend on attributes, relational data, or graph structure,
but once their required information is available they return a
deterministic Boolean value.

\begin{definition}[Executable Concept Program]
A deterministic Boolean concept program \(Q\) is generated by the
grammar

\begin{equation}
\begin{aligned}
Q ::= {}&
P
\mid
\neg Q
\mid
\operatorname{AND}(Q_1,\ldots,Q_k) \\
&
\mid
\operatorname{OR}(Q_1,\ldots,Q_k),
\qquad k\geq2,
\end{aligned}
\label{supp:eq:concept_grammar}
\end{equation}

where \(P\) is an executable primitive predicate.
\end{definition}

The primitive-predicate family used in the framework is

\begin{equation}
P
::=
P_{\mathrm{atom}}
\mid
P_{\mathrm{thr}}
\mid
P_{\mathrm{rel}}
\mid
P_{\mathrm{path}}.
\label{supp:eq:primitive_grammar}
\end{equation}

An atomic predicate evaluates an ordinary deterministic property of the
record. A typical form is

\begin{equation}
P_{\mathrm{atom}}(z)
=
\mathbb{I}
\left[
a(z)\bowtie c
\right],
\label{supp:eq:atomic_predicate}
\end{equation}

where \(a(z)\) is an attribute or deterministic derived quantity,
\(c\) is a constant, and
\(\bowtie\in\{=,\neq,<,\leq,>,\geq\}\).

Threshold predicates are represented separately because changes to
their threshold parameters constitute one of the principal rule-shift
types studied in the paper. A threshold predicate has the form

\begin{equation}
P_{\mathrm{thr}}(z;\gamma)
=
\mathbb{I}
\left[
a(z)\bowtie\gamma
\right],
\label{supp:eq:threshold_predicate}
\end{equation}

where \(\gamma\) is an explicit parameter of the rule. Representing
predicate identity separately from \(\gamma\) allows a transition from
\(\gamma_t\) to \(\gamma_{t+1}\) to be recognized as a parameter
revision rather than incorrectly represented as deletion of one
predicate and insertion of another.

A relational predicate is written in the general form

\begin{equation}
P_{\mathrm{rel}}(z;R,\psi),
\label{supp:eq:relational_predicate}
\end{equation}

where \(R\) identifies the required relation or collection of relations
and \(\psi\) specifies the deterministic relational condition. The
internal computation may contain joins, existence tests, selections, or
aggregations. For the proof below, the internal relational query does
not need to be reduced to Boolean logic; it is sufficient that the
result presented to the concept DAG is deterministic and that the
relevant source dependencies can be represented by provenance.

Similarly, a graph-path predicate is represented by

\begin{equation}
P_{\mathrm{path}}(z;\rho,\psi),
\label{supp:eq:path_predicate}
\end{equation}

where \(\rho\) specifies a graph relation or path pattern and \(\psi\)
describes the condition evaluated over the corresponding nodes, edges,
or paths.

For a fixed record \(z\), the semantics of a concept program are
defined recursively. Negation satisfies

\begin{equation}
\llbracket \neg Q\rrbracket_z
=
1-\llbracket Q\rrbracket_z.
\label{supp:eq:not_semantics}
\end{equation}

For a conjunction,

\begin{equation}
\llbracket
\operatorname{AND}
(Q_1,\ldots,Q_k)
\rrbracket_z
=
\bigwedge_{j=1}^{k}
\llbracket Q_j\rrbracket_z,
\label{supp:eq:and_semantics}
\end{equation}

while a disjunction satisfies

\begin{equation}
\llbracket
\operatorname{OR}
(Q_1,\ldots,Q_k)
\rrbracket_z
=
\bigvee_{j=1}^{k}
\llbracket Q_j\rrbracket_z.
\label{supp:eq:or_semantics}
\end{equation}

For the binary formulation used in the theoretical analysis, the
concept assignment is therefore

\begin{equation}
Q_t(z_i)
=
\llbracket Q_t\rrbracket_{z_i}.
\label{supp:eq:concept_semantics}
\end{equation}

The same reasoning can be applied to multiclass definitions by
representing the class decision through multiple Boolean outputs or an
equivalent deterministic decision graph.


\subsection{Canonical Predicate-DAG Representation}
\label{supp:subsec:canonical_graph}

The syntactic form in which two logically equivalent rules are written
should not determine whether the framework interprets them as different
concepts. For example, \(P_1\land P_2\) and
\(P_2\land P_1\) represent the same Boolean expression even though the
textual order of the predicates differs. We therefore convert every
concept definition into a canonical predicate DAG before computing the
rule delta.

The canonical representation of \(Q_t\) is

\begin{equation}
G_t
=
G(Q_t)
=
(\mathcal{V}_t,\mathcal{E}_t,\phi_t,r_t),
\label{supp:eq:canonical_graph}
\end{equation}

where \(\mathcal{V}_t\) is the set of computational nodes,
\(\mathcal{E}_t\) contains directed dependencies from lower-level
computations toward their consumers, \(\phi_t(v)\) stores the operator
type and parameters associated with node \(v\), and \(r_t\) is the
concept-output node.

For record \(z_i\), the value produced by node \(v\) is denoted by

\begin{equation}
\nu_t(v,z_i).
\label{supp:eq:node_value}
\end{equation}

In particular, the final concept result satisfies

\begin{equation}
Q_t(z_i)
=
\nu_t(r_t,z_i).
\label{supp:eq:root_value}
\end{equation}

Canonicalization performs only semantics-preserving operations. Nested
instances of associative AND and OR operators are flattened,
commutative children are placed in a deterministic order, referenced
attributes and relations are normalized, and identical subexpressions
may be represented by shared DAG nodes. A bottom-up structural signature
is then assigned to every node. One convenient definition is

\begin{equation}
\sigma(v)
=
H
\left(
\operatorname{type}(v),
\operatorname{param}(v),
\Sigma(v)
\right),
\label{supp:eq:node_signature}
\end{equation}

where \(H\) is a structural hash and \(\Sigma(v)\) contains the
canonical sequence of child signatures.

These signatures allow components that remain identical across
consecutive definitions to be aligned efficiently. Parameter revisions,
such as a threshold change, are handled using a secondary structural key
that preserves the identity of the referenced attribute and operator
while allowing the parameter itself to differ.


\subsection{Rule-Delta Compilation and Change Completeness}
\label{supp:subsec:delta_compilation}

Given the canonical graphs \(G_t\) and \(G_{t+1}\), the rule-delta
compiler determines which components persist and which components have
been modified, inserted, or removed. The output is a typed rule delta

\begin{equation}
\mathcal{C}_t
=
\operatorname{Diff}
(G_t,G_{t+1}).
\label{supp:eq:typed_delta}
\end{equation}

The compiler distinguishes threshold revisions, predicate insertion,
predicate deletion, logical rewrites, and relational or graph-path
rewrites. Importantly, the delta does not need to be a minimum graph-edit
sequence. Correctness requires a different property: every semantic
revision capable of modifying the output of at least one record must be
represented by the delta. Conservative additional edits are harmless
for correctness because they only enlarge the subsequent region of
records that may need to be examined.

\begin{definition}[Change-Complete Delta]
A rule delta \(\mathcal{C}_t\) is change-complete if every difference
between \(Q_t\) and \(Q_{t+1}\) that can alter the concept assignment of
some record is represented directly by an element of
\(\mathcal{C}_t\) or conservatively by an enclosing modified component.
\end{definition}

When an alignment is uncertain, the compiler therefore chooses the
conservative interpretation. It records the old node as deleted and the
new node as inserted rather than assuming that they represent the same
persistent component. This choice may produce additional candidates,
but it avoids incorrectly hiding a real semantic revision.

The complete compilation procedure is shown in
Algorithm~\ref{supp:alg:delta_compilation}.

\begin{algorithm}[htbp]
\caption{Typed Concept Rule-Delta Compilation}
\label{supp:alg:delta_compilation}
\KwIn{Previous rule \(Q_t\), revised rule \(Q_{t+1}\)}
\KwOut{Canonical graphs \(G_t,G_{t+1}\), alignment \(\alpha_t\),
and typed rule delta \(\mathcal{C}_t\)}

Canonicalize \(Q_t\) and obtain \(G_t\)\;

Canonicalize \(Q_{t+1}\) and obtain \(G_{t+1}\)\;

Compute bottom-up canonical signatures for all nodes\;

Align components having identical canonical signatures\;

Attempt secondary alignment for components having identical structural
roles but revised parameters\;

Reject any secondary alignment that cannot be established
unambiguously\;

Initialize \(\mathcal{C}_t\leftarrow\emptyset\)\;

\ForEach{aligned persistent component}{
    record parameter revisions, logical rewrites, or
    relational/path revisions when present\;
}

\ForEach{unmatched component of \(G_t\)}{
    add a deletion edit to \(\mathcal{C}_t\)\;
}

\ForEach{unmatched component of \(G_{t+1}\)}{
    add an insertion edit to \(\mathcal{C}_t\)\;
}

\Return \(G_t,G_{t+1},\alpha_t,\mathcal{C}_t\)\;
\end{algorithm}


\subsection{Record-Level Provenance and Candidate Retrieval}
\label{supp:subsec:formal_provenance}

The structural delta determines what changed in the concept definition,
but a rule component can have different relevance for different
historical records. The role of provenance is therefore to connect a
structural edit to the records whose previous evaluation depends on the
modified computation. For every historical record, we represent the
relevant provenance as

\begin{equation}
\operatorname{Prov}(z_i,Q_t)
=
\left(
\operatorname{Dep}_t(i),
\operatorname{Val}_t(i),
\operatorname{Src}_t(i)
\right).
\label{supp:eq:provenance_representation}
\end{equation}

The dependency component \(\operatorname{Dep}_t(i)\) records the
predicate and operator dependencies required to reason about the
previous output. The value component
\(\operatorname{Val}_t(i)\) stores node values that may later serve as
stability witnesses. Finally,
\(\operatorname{Src}_t(i)\) identifies source-level dependencies such as
relational tuples, graph nodes, edges, or path signatures. The
representation may contain conservative extra dependencies; correctness
does not require a minimal lineage expression.

To avoid scanning all provenance objects after every revision, the
framework maintains an inverted provenance index. For component \(c\),

\begin{equation}
\mathcal{I}_t(c)
=
\left\{
i:
c\in\operatorname{Dep}_t(i)
\right\}.
\label{supp:eq:inverted_index}
\end{equation}

A complication arises when the revision inserts a new component.
Because the new component was absent from \(Q_t\), no historical posting
list can exist for it. Candidate retrieval therefore uses an affected
frontier consisting of persistent components whose evaluation context
is changed by the edit. For an insertion or deletion, this can be the
nearest persistent ancestor whose child structure changes. For a
threshold revision it is normally the persistent threshold predicate
itself, and for relational or graph edits it can additionally include
the relevant relation or path identifiers.

We denote this frontier by

\begin{equation}
\mathcal{F}_t
=
\operatorname{Frontier}
(\mathcal{C}_t,G_t,G_{t+1}).
\label{supp:eq:affected_frontier}
\end{equation}

The initial record pool retrieved from the provenance index is then

\begin{equation}
\mathcal{R}_t
=
\bigcup_{c\in\mathcal{F}_t}
\mathcal{I}_t(c).
\label{supp:eq:retrieved_pool}
\end{equation}

This pruning step is valid only when the frontier and provenance index
are dependency-complete. If the implementation cannot guarantee that a
potentially affected record will appear in the retrieved postings, the
safe fallback is not to exclude that record. In the extreme case the
framework uses

\begin{equation}
\mathcal{R}_t
=
\mathcal{D}_t.
\label{supp:eq:full_fallback}
\end{equation}

This conservative fallback is important because incomplete provenance
should reduce computational efficiency rather than invalidate the
semantic correctness of the maintenance procedure.


\subsection{Record-Specific Effective Rule Changes}
\label{supp:subsec:effective_changes}

Even among records retrieved through the affected frontier, not every
syntactic edit is necessarily capable of changing the local rule
evaluation. The framework therefore computes a record-specific
effective change set

\begin{equation}
\mathcal{C}_t^{\mathrm{eff}}(z_i)
\subseteq
\mathcal{C}_t.
\label{supp:eq:effective_set}
\end{equation}

The set is intentionally conservative. It must contain every changed
component whose revised local behavior can actually alter the
computation for \(z_i\), but it may contain additional components when
local invariance cannot be determined cheaply. Formally, the property
required by the correctness analysis is

\begin{equation}
Q_t(z_i)\neq Q_{t+1}(z_i)
\quad\Longrightarrow\quad
\mathcal{C}_t^{\mathrm{eff}}(z_i)\neq\emptyset.
\label{supp:eq:effective_soundness}
\end{equation}

Consider, for example, a threshold revision

\begin{equation}
a(z)>5000
\quad\longrightarrow\quad
a(z)>3000.
\label{supp:eq:threshold_example}
\end{equation}

The local threshold output differs only when

\begin{equation}
3000<a(z_i)\leq5000.
\label{supp:eq:threshold_disagreement}
\end{equation}

Records outside this interval therefore do not contain the threshold
revision in their effective change set, provided that the attribute
value is known exactly.

A logical rewrite behaves differently. Consider

\begin{equation}
P_1\lor P_2
\quad\longrightarrow\quad
P_1\land P_2.
\label{supp:eq:logic_example}
\end{equation}

If the two persistent predicates retain their values, the two logical
operators differ only when the predicate truth values disagree. Hence
the logical rewrite is locally effective only when

\begin{equation}
P_1(z_i)\neq P_2(z_i).
\label{supp:eq:logic_disagreement}
\end{equation}

Relational and graph changes follow the same principle. If local
provenance is sufficient to establish that the relational or graph
predicate result is unchanged, the edit need not be propagated further
for that record. Otherwise it remains an effective component and is
handled conservatively.


\subsection{Persistent Blocking Nodes}
\label{supp:subsec:blocking_nodes}

The central observation behind stability certification is that a local
change does not necessarily propagate to the final output. Boolean
operators possess controlling input values that can make their result
independent of one or more changed branches. Provenance supplies the
record-specific values required to recognize this situation.

\begin{definition}[Persistent Blocking Node]
For a fixed record \(z_i\), a persistent logical node is a blocking node
with respect to a changed incoming branch when another persistent input
has an unchanged value that fixes the operator output independently of
the changed branch.
\end{definition}

For an AND node, the controlling value is false. Thus an AND node
\(v\) is blocking when it contains a persistent child \(u\) satisfying

\begin{equation}
\nu_t(u,z_i)
=
\nu_{t+1}(u,z_i)
=
0.
\label{supp:eq:and_blocking}
\end{equation}

For an OR node, the controlling value is true. It is therefore blocking
when a persistent child satisfies

\begin{equation}
\nu_t(u,z_i)
=
\nu_{t+1}(u,z_i)
=
1.
\label{supp:eq:or_blocking}
\end{equation}

The persistence requirement is essential. An old false input to an AND
node cannot be used as a blocking witness if that input's own rule
definition has changed and its revised value is unknown. The framework
must first establish that the witness itself is invariant.

A NOT node does not possess an analogous controlling value. If its
child changes, the NOT output also changes. NOT can therefore transmit a
difference but cannot block one.


\subsection{Local Invariance at a Blocking Node}
\label{supp:subsec:blocking_lemma}

Before proving the global stability theorem, we establish the local
property on which it depends.

\begin{lemma}[Blocking-Node Invariance]
\label{supp:lemma:blocking_invariance}
Let \(v\) be a persistent AND or OR node. If \(v\) contains an
unchanged controlling child for record \(z_i\), then the output of
\(v\) is invariant under changes to all other incoming branches.
\end{lemma}

\begin{IEEEproof}
We consider the AND and OR cases separately.

First suppose that \(v\) is an AND node. By the definition of a
persistent blocking node, there exists a child \(u\) whose value is
false under both concept versions. Therefore,

\begin{equation}
\nu_t(u,z_i)
=
\nu_{t+1}(u,z_i)
=
0.
\label{supp:eq:lemma_and_child}
\end{equation}

The semantics of conjunction imply that the output of an AND node is
false whenever at least one of its children is false. The values of the
remaining children are therefore irrelevant once the controlling false
input is present. Under the previous concept definition,

\begin{equation}
\nu_t(v,z_i)
=
0.
\label{supp:eq:lemma_and_old}
\end{equation}

The same controlling child remains false after the rule revision, so
the revised output is also

\begin{equation}
\nu_{t+1}(v,z_i)
=
0.
\label{supp:eq:lemma_and_new}
\end{equation}

Consequently,

\begin{equation}
\nu_t(v,z_i)
=
\nu_{t+1}(v,z_i).
\label{supp:eq:lemma_and_equal}
\end{equation}

This conclusion does not depend on whether the remaining children
change. The unchanged false input alone determines the AND result.

Now consider the OR case. By assumption, there exists a persistent
child \(u\) whose value is true under both definitions:

\begin{equation}
\nu_t(u,z_i)
=
\nu_{t+1}(u,z_i)
=
1.
\label{supp:eq:lemma_or_child}
\end{equation}

The output of an OR node is true whenever at least one child is true.
Hence the old output is

\begin{equation}
\nu_t(v,z_i)
=
1,
\label{supp:eq:lemma_or_old}
\end{equation}

and the revised output remains

\begin{equation}
\nu_{t+1}(v,z_i)
=
1.
\label{supp:eq:lemma_or_new}
\end{equation}

Therefore,

\begin{equation}
\nu_t(v,z_i)
=
\nu_{t+1}(v,z_i).
\label{supp:eq:lemma_or_equal}
\end{equation}

In both cases, the unchanged controlling child fixes the logical-node
output independently of every changed branch. The node therefore
prevents a difference arriving through those branches from propagating
further toward the concept output.
\end{IEEEproof}


\subsection{Difference Propagation Through the Concept DAG}
\label{supp:subsec:difference_propagation}

The next lemma formalizes the converse intuition. If the concept output
actually changes, then some effective local change must be connected to
the output by a dependency path along which the difference survives.

\begin{lemma}[Unblocked Difference Propagation]
\label{supp:lemma:difference_propagation}
Assume that the concept delta is change-complete and that the
record-specific effective set is conservative. If

\begin{equation}
Q_t(z_i)\neq Q_{t+1}(z_i),
\label{supp:eq:propagation_root_diff}
\end{equation}

then there exists at least one directed dependency path from a component
in \(\mathcal{C}_t^{\mathrm{eff}}(z_i)\) to the concept-output node
along which no persistent blocking node eliminates the difference.
\end{lemma}

\begin{IEEEproof}
The proof proceeds by tracing the observed output difference backward
through the deterministic concept computation.

Let \(r\) denote the concept-output node. The hypothesis gives

\begin{equation}
\nu_t(r,z_i)
\neq
\nu_{t+1}(r,z_i).
\label{supp:eq:root_values_differ}
\end{equation}

There are two possible reasons for this difference. The output node
itself may have been changed by the rule revision, or the output node
may persist while one of the values supplied to it differs.

If the local operator or input structure of \(r\) is itself modified
and this modification is capable of producing the observed difference,
then \(r\) is an effective changed component. In that case the required
path begins at the changed component itself and terminates at the output.

Otherwise, \(r\) persists with the same deterministic operator. A
deterministic operator supplied with exactly the same child values must
produce exactly the same result. Because
\eqref{supp:eq:root_values_differ}
states that its output differs, at least one of its child computations
must therefore differ between the two concept versions. Let \(v_1\)
denote such a child. Then

\begin{equation}
\nu_t(v_1,z_i)
\neq
\nu_{t+1}(v_1,z_i).
\label{supp:eq:first_child_diff}
\end{equation}

If \(v_1\) is itself a changed component, the backward trace terminates.
If it is not, its deterministic output difference must again be caused
by a differing child. We can therefore continue the same argument
recursively.

Because the concept representation is a finite DAG, this backward
process cannot continue indefinitely and cannot enter a cycle.
Eventually it reaches a primitive predicate or an internal rule site
whose local semantics differ between the old and revised definitions.
By change completeness, this site is represented in
\(\mathcal{C}_t\). Because it participates in the output difference for
\(z_i\), conservative effective-change analysis ensures that it is also
contained in

\begin{equation}
\mathcal{C}_t^{\mathrm{eff}}(z_i).
\label{supp:eq:effective_contains_origin}
\end{equation}

Reversing the backward trace yields a directed path from this effective
changed component to the output node.

It remains to show that this path cannot contain a blocking node. Every
node selected during the backward trace has a different value under the
two concept versions. However,
Lemma~\ref{supp:lemma:blocking_invariance} establishes that a persistent
blocking node necessarily has the same value under both versions.
Therefore none of the difference-carrying nodes on the traced path can
be a blocking node.

Hence a changed concept output implies the existence of at least one
unblocked path from an effective changed component to the root.
\end{IEEEproof}


\subsection{Complete Proof of the Provenance Stability Certificate}
\label{supp:subsec:stability_proof}

We can now give the complete proof of
Theorem~\ref{thm:provenance_stability} from the main paper.

\begin{IEEEproof}
Fix an arbitrary historical record \(z_i\). By hypothesis, every
directed path from every component in

\begin{equation}
\mathcal{C}_t^{\mathrm{eff}}(z_i)
\label{supp:eq:theorem_effective_set}
\end{equation}

to the concept-output node contains at least one persistent blocking
node whose controlling input remains unchanged after the rule revision.

We prove that the final concept assignment cannot change.

Assume, for contradiction, that the concept assignment does change.
Then

\begin{equation}
Q_t(z_i)
\neq
Q_{t+1}(z_i).
\label{supp:eq:theorem_contradiction}
\end{equation}

Under the change-completeness and conservative effective-change
conditions,
Lemma~\ref{supp:lemma:difference_propagation} applies. It follows from
\eqref{supp:eq:theorem_contradiction} that there must exist an
effective changed component

\begin{equation}
c
\in
\mathcal{C}_t^{\mathrm{eff}}(z_i)
\label{supp:eq:theorem_origin}
\end{equation}

and a directed path from \(c\) to the concept-output node along which
the semantic difference is not eliminated.

Let this path be

\begin{equation}
c=v_0
\rightarrow
v_1
\rightarrow
\cdots
\rightarrow
v_m=r,
\label{supp:eq:theorem_path}
\end{equation}

where \(r\) is the concept-output node.

The theorem hypothesis states that every such path contains a
persistent blocking node. Therefore there exists at least one node
\(v_j\) on
\eqref{supp:eq:theorem_path}
that has an unchanged controlling input.

By Lemma~\ref{supp:lemma:blocking_invariance}, the output of that node
must be invariant. Hence,

\begin{equation}
\nu_t(v_j,z_i)
=
\nu_{t+1}(v_j,z_i).
\label{supp:eq:block_on_path}
\end{equation}

Equation~\eqref{supp:eq:block_on_path} means that any difference
arriving at \(v_j\) through the changed branch is destroyed at
\(v_j\). The portion of the concept computation above \(v_j\) therefore
receives exactly the same value from that branch under both concept
versions.

Consequently, the path in
\eqref{supp:eq:theorem_path}
cannot be an uninterrupted path carrying a difference from the changed
component \(c\) to the root. This contradicts
Lemma~\ref{supp:lemma:difference_propagation}, which requires such an
unblocked path whenever the final concept assignment differs.

The assumption
\eqref{supp:eq:theorem_contradiction}
must therefore be false. Hence,

\begin{equation}
Q_t(z_i)
=
Q_{t+1}(z_i).
\label{supp:eq:theorem_conclusion}
\end{equation}

Thus the historical target of \(z_i\) is invariant under the concept
revision, and the record may be safely included in
\(\mathcal{D}_t^{\mathrm{safe}}\).
\end{IEEEproof}


\subsection{Graph-Cut Interpretation of the Stability Certificate}
\label{supp:subsec:graph_cut}

The previous theorem can also be interpreted as a graph-separation
result. This interpretation is useful because it clarifies why the
certificate remains valid for nested Boolean expressions rather than
only for simple two-predicate examples.

For record \(z_i\), define the effective source set as

\begin{equation}
S_i
=
\mathcal{C}_t^{\mathrm{eff}}(z_i).
\label{supp:eq:source_set}
\end{equation}

Let \(B_i\) denote the set of all persistent blocking nodes established
for that record. If every path from \(S_i\) to the concept output
intersects \(B_i\), then \(B_i\) is a vertex cut separating every
potential source of semantic difference from the output.

Once the nodes in \(B_i\) are replaced conceptually by their invariant
Boolean values, no effective changed component can influence the root.
The stability certificate can therefore be expressed as the following
implication:

\begin{equation}
B_i
\text{ separates }
S_i
\text{ from the output}
\quad\Longrightarrow\quad
Q_t(z_i)=Q_{t+1}(z_i).
\label{supp:eq:cut_interpretation}
\end{equation}

This interpretation also provides a direct implementation strategy:
candidate certification can be performed using an upward traversal from
the effective changed components, terminating a branch whenever a
blocking node is encountered.


\subsection{Record-Level Stability Certification Algorithm}
\label{supp:subsec:certificate_algorithm}

Algorithm~\ref{supp:alg:certificate} implements the preceding theorem.
The algorithm deliberately returns only two logical outcomes:
\emph{certified stable} and \emph{not certified}. The second outcome
must not be interpreted as evidence that the target actually changes.
It means only that the sufficient conditions for exclusion have not
been established.

\begin{algorithm}[htbp]
\caption{Record-Level Provenance Stability Certification}
\label{supp:alg:certificate}
\KwIn{
Record \(z_i\), changed components \(\mathcal{C}_t\),
canonical concept graphs, and provenance
\(\operatorname{Prov}(z_i,Q_t)\)
}
\KwOut{\texttt{stable} or \texttt{not-certified}}

Determine
\(\mathcal{C}_t^{\mathrm{eff}}(z_i)\)\;

\If{
\(\mathcal{C}_t^{\mathrm{eff}}(z_i)=\emptyset\)
}{
    \Return \texttt{stable}\;
}

Initialize a queue with all effective changed components\;

Mark all components as unvisited\;

\While{the queue is not empty}{

    Remove one current component \(u\)\;

    \If{\(u\) is the concept-output node}{
        \Return \texttt{not-certified}\;
    }

    Inspect each parent of \(u\) toward the output\;

    \ForEach{such parent \(v\)}{

        \eIf{\(v\) has a persistent invariant blocking witness}{
            stop propagation through \(v\)\;
        }{
            add \(v\) to the queue if not previously visited\;
        }
    }

    \If{the structural correspondence required to continue the
    propagation analysis is unavailable or ambiguous}{
        \Return \texttt{not-certified}\;
    }
}

\Return \texttt{stable}\;
\end{algorithm}

The conservative behavior in the final condition is essential. Missing
provenance, ambiguous rule alignment, or unavailable witness values
cannot justify exclusion. Such uncertainty causes the record to remain
in the candidate region.


\subsection{Complete Provenance-Guided Candidate Discovery}
\label{supp:subsec:candidate_discovery}

The complete candidate-discovery procedure combines the inverted
provenance index with the stability certificate. Records that lie
outside a dependency-complete affected frontier can be excluded
directly. Records retrieved by the frontier are examined individually
because Boolean context may provide additional certificates that are
not visible from posting-list membership alone.

\begin{algorithm}[htbp]
\caption{Provenance-Guided Candidate Discovery}
\label{supp:alg:candidate_discovery}
\KwIn{
Historical collection \(\mathcal{D}_t\),
rule delta \(\mathcal{C}_t\),
concept graphs, and provenance index \(\Pi_t\)
}
\KwOut{
\(\mathcal{D}_t^{\mathrm{safe}}\) and
\(\mathcal{D}_t^{\mathrm{cand}}\)
}

Construct the affected frontier \(\mathcal{F}_t\)\;

\eIf{dependency completeness of the frontier/index is guaranteed}{
    retrieve the union of postings associated with
    \(\mathcal{F}_t\) and call it \(\mathcal{R}_t\)\;
}{
    set \(\mathcal{R}_t\leftarrow\mathcal{D}_t\)\;
}

Initialize
\(\mathcal{D}_t^{\mathrm{safe}}\)
with records excluded by dependency-complete retrieval\;

Initialize
\(\mathcal{D}_t^{\mathrm{cand}}\leftarrow\emptyset\)\;

\ForEach{\(z_i\in\mathcal{R}_t\)}{

    run Algorithm~\ref{supp:alg:certificate}\;

    \eIf{the result is \texttt{stable}}{
        add \(z_i\) to
        \(\mathcal{D}_t^{\mathrm{safe}}\)\;
    }{
        add \(z_i\) to
        \(\mathcal{D}_t^{\mathrm{cand}}\)\;
    }
}

\Return
\(\mathcal{D}_t^{\mathrm{safe}},
\mathcal{D}_t^{\mathrm{cand}}\)\;
\end{algorithm}


\subsection{Conditions Required for Exact Correctness}
\label{supp:subsec:correctness_conditions}

The correctness results are conditional guarantees rather than claims
that every compressed or approximate provenance implementation is
automatically exact. The following conditions make the boundary of the
guarantee explicit.

\textbf{C1: Fixed historical snapshot.}
The two concept definitions are compared on the same \(z_i\). If the
underlying tuples, relations, graph edges, or external state have also
changed, these changes must themselves be represented as input
dependencies or deltas.

\textbf{C2: Deterministic executability.}
Every record handled by the exact part of the procedure has
deterministically evaluable old and revised concept outputs.

\textbf{C3: Change-complete structural delta.}
No semantic rule revision capable of changing a historical target is
omitted from \(\mathcal{C}_t\).

\textbf{C4: Dependency-complete provenance retrieval.}
If the rule revision can affect \(z_i\), the provenance index or its
conservative fallback must retain \(z_i\) for further examination.

\textbf{C5: Sound provenance witnesses.}
Every node value used to certify a blocking condition must equal the
true old evaluation, and the invariance of that witness under the new
definition must itself be established.

\textbf{C6: Conservative effective-change construction.}
A rule component that may generate a local difference for \(z_i\)
cannot be removed incorrectly from
\(\mathcal{C}_t^{\mathrm{eff}}(z_i)\).

\textbf{C7: Conservative treatment of uncertainty.}
Missing provenance, unresolved relations, uncertain concept semantics,
or ambiguous component alignment cannot be interpreted as stability.
The corresponding record remains in the candidate or ambiguous region.

The first six conditions establish the deterministic correctness result,
while the seventh ensures that uncertainty cannot create an erroneous
stable-set exclusion.


\subsection{Complete Proof of Exact Concept-Delta Recovery}
\label{supp:subsec:exact_delta_proof}

We next prove Proposition~\ref{prop:exact_delta_recovery} from the main
paper in full detail. The key observation is that the proposition does
not require the stable set to contain every unchanged record. It
requires only that every record placed in that set truly be unchanged.
The candidate set may therefore contain false positives, i.e., records
that are eventually found to retain the same label. Such false
positives increase computation but cannot cause an affected record to
be lost.

\begin{IEEEproof}[Proof of Proposition~\ref{prop:exact_delta_recovery}]

Let the true changed-label set be

\begin{equation}
\mathcal{D}_t^{\Delta}
=
\left\{
z_i\in\mathcal{D}_t:
Q_t(z_i)\neq Q_{t+1}(z_i)
\right\}.
\label{supp:eq:prop1_true_set}
\end{equation}

The candidate-only procedure returns

\begin{equation}
\widehat{\mathcal{D}}_t^{\Delta}
=
\left\{
z_i\in\mathcal{D}_t^{\mathrm{cand}}:
Q_t(z_i)\neq Q_{t+1}(z_i)
\right\}.
\label{supp:eq:prop1_candidate_set}
\end{equation}

To prove equality between these sets, we establish both set inclusions.

\textbf{First inclusion: every truly affected record is a candidate.}

Assume that

\begin{equation}
z_i
\in
\mathcal{D}_t^{\Delta}.
\label{supp:eq:prop1_take_affected}
\end{equation}

By definition of the affected set,

\begin{equation}
Q_t(z_i)
\neq
Q_{t+1}(z_i).
\label{supp:eq:prop1_label_diff}
\end{equation}

Suppose, for contradiction, that \(z_i\) were not contained in the
candidate set. Since

\begin{equation}
\mathcal{D}_t^{\mathrm{cand}}
=
\mathcal{D}_t
\setminus
\mathcal{D}_t^{\mathrm{safe}},
\label{supp:eq:prop1_candidate_complement}
\end{equation}

the assumption that \(z_i\notin\mathcal{D}_t^{\mathrm{cand}}\) implies

\begin{equation}
z_i
\in
\mathcal{D}_t^{\mathrm{safe}}.
\label{supp:eq:prop1_in_safe}
\end{equation}

The proposition assumes that the stable set is sound. Therefore every
record in this set satisfies

\begin{equation}
Q_t(z_i)
=
Q_{t+1}(z_i).
\label{supp:eq:prop1_safe_equality}
\end{equation}

Equations~\eqref{supp:eq:prop1_label_diff} and
\eqref{supp:eq:prop1_safe_equality} contradict one another. Hence an
affected record cannot belong to the certified stable set. We have
therefore established

\begin{equation}
\mathcal{D}_t^{\Delta}
\subseteq
\mathcal{D}_t^{\mathrm{cand}}.
\label{supp:eq:prop1_first_inclusion}
\end{equation}

\textbf{Second inclusion: candidate reevaluation cannot create a false
changed-label record.}

Take any record

\begin{equation}
z_i
\in
\widehat{\mathcal{D}}_t^{\Delta}.
\label{supp:eq:prop1_take_candidate_delta}
\end{equation}

From the definition in
\eqref{supp:eq:prop1_candidate_set}, this implies two facts. First,

\begin{equation}
z_i
\in
\mathcal{D}_t^{\mathrm{cand}},
\label{supp:eq:prop1_candidate_membership}
\end{equation}

and second,

\begin{equation}
Q_t(z_i)
\neq
Q_{t+1}(z_i).
\label{supp:eq:prop1_candidate_disagreement}
\end{equation}

Because the candidate set is a subset of the historical collection,

\begin{equation}
\mathcal{D}_t^{\mathrm{cand}}
\subseteq
\mathcal{D}_t.
\label{supp:eq:prop1_candidate_subset_data}
\end{equation}

Therefore \(z_i\) is a historical record satisfying exactly the
disagreement condition in the definition of
\(\mathcal{D}_t^{\Delta}\). It follows that

\begin{equation}
z_i
\in
\mathcal{D}_t^{\Delta}.
\label{supp:eq:prop1_candidate_is_true}
\end{equation}

Hence,

\begin{equation}
\widehat{\mathcal{D}}_t^{\Delta}
\subseteq
\mathcal{D}_t^{\Delta}.
\label{supp:eq:prop1_second_inclusion}
\end{equation}

It remains to show the reverse inclusion. From
\eqref{supp:eq:prop1_first_inclusion}, every truly affected record is
already known to lie in the candidate set. Such a record also satisfies
the old--new disagreement condition by definition. Therefore every
element of \(\mathcal{D}_t^{\Delta}\) satisfies the membership
criterion defining
\(\widehat{\mathcal{D}}_t^{\Delta}\), and consequently

\begin{equation}
\mathcal{D}_t^{\Delta}
\subseteq
\widehat{\mathcal{D}}_t^{\Delta}.
\label{supp:eq:prop1_reverse_inclusion}
\end{equation}

Combining
\eqref{supp:eq:prop1_second_inclusion}
and
\eqref{supp:eq:prop1_reverse_inclusion}
gives

\begin{equation}
\widehat{\mathcal{D}}_t^{\Delta}
=
\mathcal{D}_t^{\Delta}.
\label{supp:eq:prop1_final}
\end{equation}

Thus evaluating the concept transition only over the candidate set
recovers exactly the same changed-label set that would be obtained by
evaluating the complete historical collection. The proof also shows why
candidate-set over-approximation is acceptable: additional unchanged
records can increase computation, but they cannot change the recovered
affected set as long as no truly affected record is placed in the
stable set.
\end{IEEEproof}


\subsection{Correctness of Provenance-Guided Pruning}
\label{supp:subsec:pruning_correctness}

The stability theorem and exact-recovery proposition combine directly
into the following consequence.

\begin{corollary}[Exact Recovery Under Sound Provenance Certification]
\label{supp:cor:exact_recovery}
Suppose Conditions C1--C7 hold and
\(\mathcal{D}_t^{\mathrm{safe}}\) contains only records satisfying the
provenance stability certificate. Then every truly affected historical
record remains in the candidate region, and exact reevaluation of the
candidate region recovers the complete changed-label set.
\end{corollary}

\begin{IEEEproof}
Theorem~\ref{thm:provenance_stability} establishes that every record
accepted by the certificate satisfies

\begin{equation}
Q_t(z_i)
=
Q_{t+1}(z_i).
\label{supp:eq:cor_safe}
\end{equation}

Consequently, the resulting
\(\mathcal{D}_t^{\mathrm{safe}}\)
is a sound stable set. Proposition~\ref{prop:exact_delta_recovery}
then applies immediately and guarantees that the complement of this
stable set contains every record whose target actually changes.
Candidate-only reevaluation therefore recovers
\(\mathcal{D}_t^{\Delta}\) exactly.
\end{IEEEproof}


\subsection{Exact and Ambiguous Candidate Resolution}
\label{supp:subsec:exact_ambiguous}

After candidate discovery, the framework distinguishes deterministic
reevaluation from unresolved semantic cases. This distinction is
important for the theoretical interpretation because a candidate record
need not already be known to be affected. Candidate membership means
only that invariance has not been certified.

Let the executable candidate subset be

\begin{equation}
\mathcal{E}_t
=
\left\{
z_i\in\mathcal{D}_t^{\mathrm{cand}}
:
Q_{t+1}(z_i)
\text{ is deterministically executable}
\right\}.
\label{supp:eq:executable_candidates}
\end{equation}

The unresolved candidate subset is

\begin{equation}
\mathcal{U}_t^{\mathrm{amb}}
=
\mathcal{D}_t^{\mathrm{cand}}
\setminus
\mathcal{E}_t.
\label{supp:eq:unresolved_candidates}
\end{equation}

For executable candidates, the actual changed-label records are
identified exactly:

\begin{equation}
\mathcal{D}_t^{\mathrm{exact}}
=
\left\{
z_i\in\mathcal{E}_t:
Q_t(z_i)\neq Q_{t+1}(z_i)
\right\}.
\label{supp:eq:exact_affected}
\end{equation}

Conceptually, the genuinely affected ambiguous region is

\begin{equation}
\mathcal{D}_t^{\mathrm{amb}}
=
\mathcal{D}_t^{\Delta}
\cap
\mathcal{U}_t^{\mathrm{amb}}.
\label{supp:eq:true_ambiguous}
\end{equation}

However, membership in this set may not be known before expert
supervision or missing information becomes available. The implementation
therefore performs annotation selection over
\(\mathcal{U}_t^{\mathrm{amb}}\), rather than assuming that every
unresolved candidate is truly affected. This avoids circular reasoning
in which the algorithm would need to know the revised label before
deciding whether the revised label must be requested.


\subsection{Complete Incremental Maintenance Algorithm}
\label{supp:subsec:complete_algorithm}

The complete procedure is shown in
Algorithm~\ref{supp:alg:complete_update}. The algorithm makes explicit
the separation between structural concept analysis, data maintenance,
semantic resolution, predictor repair, and version-state maintenance.

\begin{algorithm}[htbp]
\caption{Provenance-Guided Incremental Concept Maintenance}
\label{supp:alg:complete_update}
\KwIn{
\(Q_t,Q_{t+1}\),
historical collection \(\mathcal{D}_t\),
predictor \(f_{\theta_t}\),
provenance/index state \(\Pi_t\),
annotation budget \(B\)
}
\KwOut{
Updated predictor \(f_{\theta_{t+1}}\) and updated maintenance state
}

Compile \(Q_t\) and \(Q_{t+1}\) into canonical concept graphs\;

Compute typed rule delta \(\mathcal{C}_t\)\;

Construct the affected provenance frontier\;

Run Algorithm~\ref{supp:alg:candidate_discovery} to obtain
\(\mathcal{D}_t^{\mathrm{safe}}\) and
\(\mathcal{D}_t^{\mathrm{cand}}\)\;

Initialize the exact repair set and unresolved candidate set\;

\ForEach{\(z_i\in\mathcal{D}_t^{\mathrm{cand}}\)}{

    \eIf{the required revised rule components are executable}{

        obtain the old target from stored supervision or \(Q_t\)\;

        evaluate the revised concept \(Q_{t+1}(z_i)\)\;

        \If{the old and revised targets differ}{
            add the revised example to
            \(\mathcal{D}_t^{\mathrm{exact}}\)\;
        }

    }{
        add \(z_i\) to
        \(\mathcal{U}_t^{\mathrm{amb}}\)\;
    }
}

Select at most \(B\) informative examples from
\(\mathcal{U}_t^{\mathrm{amb}}\) and obtain revised supervision\;

Construct
\(\mathcal{D}_t^{\mathrm{label}}\)\;

Construct the repair set
\(\mathcal{D}_t^{\mathrm{repair}}\)\;

Sample a stability buffer
\(\mathcal{D}_t^{\mathrm{stable}}
\subseteq
\mathcal{D}_t^{\mathrm{safe}}\)\;

Incrementally repair \(f_{\theta_t}\) using the repair set and
stability buffer\;

Update the provenance and index structures required by \(Q_{t+1}\)\;

Store the revised rule, graph, provenance state, affected-data state,
and predictor state in the versioned concept memory\;

\Return updated predictor and maintenance state\;
\end{algorithm}


\subsection{Detailed Computational Complexity}
\label{supp:subsec:detailed_complexity}

We next derive the cost of the complete maintenance path. The analysis
separates structural rule processing, provenance retrieval, stability
certification, candidate reevaluation, predictor repair, and index
maintenance because these terms scale differently. This decomposition
is important: a rule revision can be structurally small while still
having a large data impact, and conversely a relatively complex
structural revision can remain inexpensive when only a small historical
region depends on the changed components.

Let the combined sizes of the two consecutive concept graphs be

\begin{equation}
V
=
|\mathcal{V}_t|
+
|\mathcal{V}_{t+1}|,
\label{supp:eq:combined_vertices}
\end{equation}

and

\begin{equation}
E
=
|\mathcal{E}_t|
+
|\mathcal{E}_{t+1}|.
\label{supp:eq:combined_edges}
\end{equation}

Canonicalization and bottom-up structural hashing visit every node and
dependency edge. Ignoring ordering of commutative children, this
requires

\begin{equation}
O(V+E)
\label{supp:eq:canonical_cost}
\end{equation}

time. When AND and OR children are explicitly sorted, the additional
cost is

\begin{equation}
O
\left(
\sum_{v}
\deg(v)\log\deg(v)
\right).
\label{supp:eq:sorting_cost}
\end{equation}

For rule languages with bounded operator arity, this additional term
remains linear up to a constant factor.

Exact component alignment can be performed using hash tables indexed by
canonical signatures. Its expected cost is therefore linear in the
number of nodes. Secondary parameter-aware alignment adds another
linear pass when stable structural keys are indexed. The complete
one-time rule-delta cost is summarized by

\begin{equation}
C_{\Delta Q}
=
O
\left(
V+E+
\sum_{v}
\deg(v)\log\deg(v)
\right).
\label{supp:eq:delta_compilation_cost}
\end{equation}

The provenance cost depends on the postings associated with the
affected frontier. Let

\begin{equation}
F_t
=
|\mathcal{F}_t|
\label{supp:eq:frontier_size}
\end{equation}

be the number of frontier components, and let

\begin{equation}
L_t
=
\sum_{c\in\mathcal{F}_t}
|\mathcal{I}_t(c)|
\label{supp:eq:posting_volume}
\end{equation}

be the total number of posting entries retrieved before duplicate
elimination. With hashed sets or bitmap unions, the corresponding
frontier-retrieval cost is approximately

\begin{equation}
O(F_t+L_t).
\label{supp:eq:posting_cost}
\end{equation}

Let

\begin{equation}
R_t
=
|\mathcal{R}_t|
\label{supp:eq:retrieved_count}
\end{equation}

denote the number of unique retrieved records, and let
\(\bar h_t\) denote the average number of concept-graph edges examined
while performing stability certification for one such record.
Record-specific certification therefore requires

\begin{equation}
O(R_t\bar h_t).
\label{supp:eq:certificate_cost}
\end{equation}

The worst case occurs when the entire rule graph must be traversed for
each retrieved record, giving

\begin{equation}
O(R_tE).
\label{supp:eq:certificate_worst_case}
\end{equation}

However, blocking nodes can terminate individual paths early, so the
actual traversal can be substantially smaller than this bound.

Let

\begin{equation}
K_t
=
|\mathcal{D}_t^{\mathrm{cand}}|
\label{supp:eq:candidate_count}
\end{equation}

and let \(C_{\Delta}\) denote the average cost of evaluating the changed
or unresolved concept fragment for one candidate. Candidate
reevaluation requires

\begin{equation}
O(K_tC_{\Delta}).
\label{supp:eq:candidate_eval_cost}
\end{equation}

If changed-fragment execution is unavailable, the framework can always
evaluate the full revised concept on a candidate, in which case

\begin{equation}
C_{\Delta}
=
C_Q,
\label{supp:eq:full_candidate_cost}
\end{equation}

where \(C_Q\) is the average cost of full concept evaluation for one
record.

Combining posting retrieval and stability certification into the
provenance-processing term \(C_{\mathrm{prov}}\), the concept-maintenance
cost becomes

\begin{equation}
T_{\mathrm{inc}}^{Q}
=
O
\left(
C_{\Delta Q}
+
C_{\mathrm{prov}}
+
K_tC_{\Delta}
+
C_{\mathrm{index}}
\right),
\label{supp:eq:incremental_concept_cost}
\end{equation}

where \(C_{\mathrm{index}}\) represents maintenance of provenance
structures following the update.

Full concept recomputation instead requires

\begin{equation}
T_{\mathrm{full}}^{Q}
=
O(NC_Q).
\label{supp:eq:full_concept_cost}
\end{equation}

When predictor adaptation is included, let \(C_{\mathrm{repair}}\)
denote the incremental model-update cost and
\(C_{\mathrm{train}}^{\mathrm{full}}\) the cost of complete retraining.
The corresponding end-to-end costs are

\begin{equation}
T_{\mathrm{inc}}
=
O
\left(
C_{\Delta Q}
+
C_{\mathrm{prov}}
+
K_tC_{\Delta}
+
C_{\mathrm{repair}}
+
C_{\mathrm{index}}
\right),
\label{supp:eq:end_to_end_incremental}
\end{equation}

and

\begin{equation}
T_{\mathrm{full}}
=
O
\left(
NC_Q
+
C_{\mathrm{train}}^{\mathrm{full}}
\right).
\label{supp:eq:end_to_end_full}
\end{equation}


\subsection{Complete Proof of the Incremental Advantage}
\label{supp:subsec:incremental_advantage_proof}

We now provide the detailed proof of
Proposition~\ref{prop:incremental_advantage}. The proposition concerns
the concept-recomputation component of the framework. The
predictor-specific optimization term is intentionally excluded from the
formal asymptotic statement because different model families have
different update complexities.

\begin{IEEEproof}[Proof of Proposition~\ref{prop:incremental_advantage}]

The complete historical recomputation strategy evaluates the revised
concept over all \(N\) historical records. Its concept-evaluation cost
is therefore proportional to

\begin{equation}
NC_Q.
\label{supp:eq:prop2_full}
\end{equation}

The proposed procedure instead incurs the rule-compilation cost, the
provenance-processing cost, candidate-only concept evaluation, and
index-maintenance cost. Thus,

\begin{equation}
T_{\mathrm{inc}}^{Q}
=
O
\left(
C_{\Delta Q}
+
C_{\mathrm{prov}}
+
K_tC_{\Delta}
+
C_{\mathrm{index}}
\right).
\label{supp:eq:prop2_inc}
\end{equation}

To compare the asymptotic growth of the two procedures, divide the
incremental expression by the full concept-recomputation scale
\(NC_Q\). We obtain

\begin{equation}
\frac{T_{\mathrm{inc}}^{Q}}
{NC_Q}
=
O
\left(
\frac{C_{\Delta Q}}{NC_Q}
+
\frac{C_{\mathrm{prov}}}{NC_Q}
+
\frac{K_tC_{\Delta}}{NC_Q}
+
\frac{C_{\mathrm{index}}}{NC_Q}
\right).
\label{supp:eq:prop2_normalized}
\end{equation}

We now examine the four normalized terms individually.

By the assumptions of the proposition,

\begin{equation}
C_{\Delta Q}
=
o(NC_Q).
\label{supp:eq:prop2_delta_small}
\end{equation}

By the definition of little-\(o\), this means

\begin{equation}
\frac{C_{\Delta Q}}
{NC_Q}
\longrightarrow
0.
\label{supp:eq:prop2_delta_limit}
\end{equation}

The same reasoning applies to provenance processing:

\begin{equation}
C_{\mathrm{prov}}
=
o(NC_Q),
\label{supp:eq:prop2_prov_small}
\end{equation}

which implies

\begin{equation}
\frac{C_{\mathrm{prov}}}
{NC_Q}
\longrightarrow
0.
\label{supp:eq:prop2_prov_limit}
\end{equation}

Similarly, index maintenance satisfies

\begin{equation}
C_{\mathrm{index}}
=
o(NC_Q),
\label{supp:eq:prop2_index_small}
\end{equation}

and therefore

\begin{equation}
\frac{C_{\mathrm{index}}}
{NC_Q}
\longrightarrow
0.
\label{supp:eq:prop2_index_limit}
\end{equation}

It remains to analyze the only term that directly depends on the number
of candidate records. Rewrite that term as

\begin{equation}
\frac{K_tC_{\Delta}}
{NC_Q}
=
\frac{K_t}{N}
\frac{C_{\Delta}}{C_Q}.
\label{supp:eq:prop2_factorization}
\end{equation}

The proposition assumes

\begin{equation}
C_{\Delta}
\leq
C_Q.
\label{supp:eq:prop2_fragment_bound}
\end{equation}

Because both quantities represent nonnegative execution costs,
\eqref{supp:eq:prop2_fragment_bound} gives

\begin{equation}
0
\leq
\frac{C_{\Delta}}{C_Q}
\leq
1.
\label{supp:eq:prop2_cost_ratio_bound}
\end{equation}

Multiplying this inequality by the nonnegative candidate ratio yields

\begin{equation}
0
\leq
\frac{K_t}{N}
\frac{C_{\Delta}}{C_Q}
\leq
\frac{K_t}{N}.
\label{supp:eq:prop2_squeeze_bound}
\end{equation}

The final assumption of the proposition is

\begin{equation}
\frac{K_t}{N}
\longrightarrow
0.
\label{supp:eq:prop2_candidate_limit}
\end{equation}

The left side of
\eqref{supp:eq:prop2_squeeze_bound}
is nonnegative and its upper bound converges to zero. Therefore, by the
squeeze theorem,

\begin{equation}
\frac{K_t}{N}
\frac{C_{\Delta}}{C_Q}
\longrightarrow
0.
\label{supp:eq:prop2_candidate_term_limit}
\end{equation}

Using
\eqref{supp:eq:prop2_factorization}, this is equivalent to

\begin{equation}
\frac{K_tC_{\Delta}}
{NC_Q}
\longrightarrow
0.
\label{supp:eq:prop2_candidate_final}
\end{equation}

We have therefore shown that every normalized term in
\eqref{supp:eq:prop2_normalized}
converges to zero. Consequently,

\begin{equation}
\frac{T_{\mathrm{inc}}^{Q}}
{NC_Q}
\longrightarrow
0.
\label{supp:eq:prop2_total_limit}
\end{equation}

By definition of little-\(o\),

\begin{equation}
T_{\mathrm{inc}}^{Q}
=
o(NC_Q).
\label{supp:eq:prop2_final_result}
\end{equation}

Thus the concept-recomputation component of the proposed procedure is
asymptotically smaller than complete historical concept recomputation
under the stated assumptions.
\end{IEEEproof}


\subsection{Interpretation of the Complexity Result}
\label{supp:subsec:complexity_interpretation}

The previous proposition identifies the candidate fraction as a central
determinant of computational efficiency. Define

\begin{equation}
\rho_t
=
\frac{
|\mathcal{D}_t^{\mathrm{cand}}|
}{
N
}.
\label{supp:eq:candidate_ratio_final}
\end{equation}

Ignoring fixed structural and indexing overhead, the dominant
record-level ratio between selective and full concept execution is

\begin{equation}
R_t^{Q}
=
\rho_t
\frac{C_{\Delta}}{C_Q}.
\label{supp:eq:record_cost_ratio}
\end{equation}

This decomposition separates two sources of computational savings. The
first is \emph{data localization}: a small \(\rho_t\) means that
provenance and stability analysis successfully eliminate most of the
historical collection from reevaluation. The second is
\emph{computation localization}: a small
\(C_{\Delta}/C_Q\) means that each candidate can be processed by
executing only the changed rule fragment rather than the complete
concept program. The largest benefit occurs when both forms of
localization are present.

The expression also explains the natural operating limit of the method.
If a concept revision is effectively global, then

\begin{equation}
\rho_t
\approx
1.
\label{supp:eq:global_candidate}
\end{equation}

If the changed fragment is additionally as expensive as the complete
concept program,

\begin{equation}
C_{\Delta}
\approx
C_Q,
\label{supp:eq:global_cost}
\end{equation}

then

\begin{equation}
R_t^{Q}
\approx
1.
\label{supp:eq:no_gain}
\end{equation}

In this regime, selective maintenance naturally approaches full
recomputation. This is not a correctness failure. It means that the
concept revision contains little exploitable locality.


\subsection{Extension to Relational and Graph Predicates}
\label{supp:subsec:relational_graph_extension}

The Boolean stability theorem does not require primitive predicates to
be simple scalar comparisons. A relational or graph computation can be
treated as a leaf of the Boolean DAG once the framework determines
whether its Boolean result is invariant or potentially changed for the
record under consideration.

Suppose the old and revised concept definitions contain relational
predicates \(P_R^t\) and \(P_R^{t+1}\). If provenance and local
incremental query reasoning establish

\begin{equation}
P_R^t(z_i)
=
P_R^{t+1}(z_i),
\label{supp:eq:relation_invariant}
\end{equation}

then the relational revision does not belong to the effective change
set for that record. If equality cannot be established, the relational
predicate remains potentially effective and its influence is propagated
through the Boolean graph.

For example, consider

\begin{equation}
Q_t
=
P_A
\land
P_R^t
\label{supp:eq:relational_old}
\end{equation}

and

\begin{equation}
Q_{t+1}
=
P_A
\land
P_R^{t+1}.
\label{supp:eq:relational_new}
\end{equation}

If \(P_A\) is known to remain false for \(z_i\), then

\begin{equation}
P_A(z_i)=0
\label{supp:eq:relational_block}
\end{equation}

is already sufficient to certify that both concept outputs are false.
The potentially expensive revised relational query therefore does not
need to be executed for this record. If \(P_A(z_i)=1\), the relational
branch can become decisive and must be reevaluated unless its own
provenance establishes local invariance.

Graph-path predicates are handled identically at the Boolean level.
Their provenance additionally provides structural localization by
identifying the nodes, edges, relation types, or path signatures through
which the old graph predicate was evaluated.


\subsection{Correctness Conditions for Recurring Concept Reuse}
\label{supp:subsec:recurrence_correctness}

Versioned concept memory allows the framework to reuse information when
a previously observed definition recurs. However, it is important to
separate concept-level state from data-dependent state. Suppose the
sequence of definitions contains

\begin{equation}
Q_1
\rightarrow
Q_2
\rightarrow
Q_3
\rightarrow
Q_1.
\label{supp:eq:recurrence_sequence}
\end{equation}

If the final \(Q_1\) is canonically identical to the earlier
definition, its rule graph can be reused directly. The previously
constructed graph therefore does not need to be rediscovered.

Provenance state is different because it depends on the underlying data
snapshot. If new records have arrived or source tuples and graph edges
have changed since the earlier occurrence of \(Q_1\), old provenance
cannot automatically be treated as complete for those new dependencies.
The valid state is therefore the combination of reusable historical
provenance and newly constructed provenance for changed or newly
observed data.

The same qualification applies to the stored model state. A previous
parameter state \(\theta_1\) already represents a predictor adapted to
\(Q_1\) and can therefore provide an efficient recovery point.
Nevertheless, if the data distribution has also changed, the stored
predictor may still require incremental repair. Version memory therefore
reduces repeated computation without replacing the ordinary validity
checks required for data-dependent state.


\subsection{Scope of the Formal Guarantees}
\label{supp:subsec:guarantee_scope}

The theoretical analysis distinguishes certification from heuristic
selection. Under Conditions C1--C7, a record is removed from the
candidate set only when its invariance has been formally established.
In this certified operating mode, the stability theorem guarantees a
sound stable set and Proposition~\ref{prop:exact_delta_recovery}
guarantees exact affected-set recovery.

A practical implementation can instead use compressed provenance,
approximate dependency reconstruction, bounded relational neighborhoods,
or other engineering approximations. These mechanisms can reduce
storage or retrieval cost, but an approximation that omits a relevant
dependency need not satisfy the sufficient conditions of the theorem.
In such configurations, affected-set recall becomes an empirical
property of the implementation rather than a mathematical consequence
of the certificate.

This distinction is important because the theorem establishes the
correctness of the \emph{certification rule}: whenever its assumptions
hold and the system declares a record stable, that record is invariant.
It does not assert that every possible provenance representation will
always be sufficiently complete to certify every unchanged record.

The overall theoretical result can therefore be summarized as the
following implication chain:

\begin{equation}
\begin{aligned}
&
\text{change-complete rule delta}
+
\text{sound provenance}
\\
&
\qquad
+
\text{conservative propagation}
\\[1mm]
&
\qquad\Longrightarrow
\text{sound certified stable set}
\\[1mm]
&
\qquad\Longrightarrow
\text{all affected records}
\\
&
\qquad\qquad
\text{remain candidates}
\\[1mm]
&
\qquad\Longrightarrow
\text{exact changed-label recovery}
\\
&
\qquad\qquad
\text{on executable candidates}.
\end{aligned}
\label{supp:eq:guarantee_chain}
\end{equation}

The computational analysis then establishes that this correctness can
provide a substantial efficiency advantage whenever the resulting
candidate region is small relative to the complete historical
collection.



\section{Extended Experiments and Analysis}
\label{supp:sec:extended_experiments}

This section extends the empirical evaluation beyond the experiments
reported in the main manuscript. The analysis focuses on questions that
become important when the framework is deployed over longer periods or
under more demanding operating conditions. In particular, we examine
whether provenance-guided repair remains effective across different
predictor families, whether versioned concept memory continues to
provide useful recovery after multiple recurring revisions, how much
additional storage is introduced by provenance and concept-version
state, how sensitive the repair procedure is to its preservation and
provenance parameters, whether the observed predictive differences are
systematic across matched experimental conditions, and how performance
degrades as the assumptions favoring selective maintenance are
progressively weakened.


\subsection{Robustness Across Predictor Families}
\label{supp:subsec:predictor_family}

The proposed framework is designed to separate semantic maintenance
from the architecture used for prediction. Rule-delta compilation and
provenance-guided affected-data discovery operate before model repair,
and therefore the usefulness of the maintenance mechanism should not
depend on a single predictive family. We examine this property using
the three predictors introduced in the experimental setup: XGBoost,
the neural tabular model, and the Hoeffding Adaptive Tree. In this
experiment, the rule revisions, historical provenance, candidate-set
construction, and data splits are held fixed, while the predictor and
its corresponding update mechanism are changed.

Figure~\ref{supp:fig:predictor_robustness} summarizes the results using
three complementary views. Figure~\ref{supp:fig:predictor_retention}
reports the percentage of complete-retraining Macro-F1 retained by
incremental repair for every dataset--predictor combination.
Figure~\ref{supp:fig:predictor_latency} shows the range of update-time
reductions observed across the datasets for each predictor family.
Finally, Fig.~\ref{supp:fig:predictor_profile} summarizes normalized
predictive retention, update speed, repair gap, and cross-dataset
consistency in a common robustness profile.

\begin{figure*}[htbp]
\centering

\subfloat[Predictive retention.]{
    \includegraphics[width=0.315\textwidth]
    {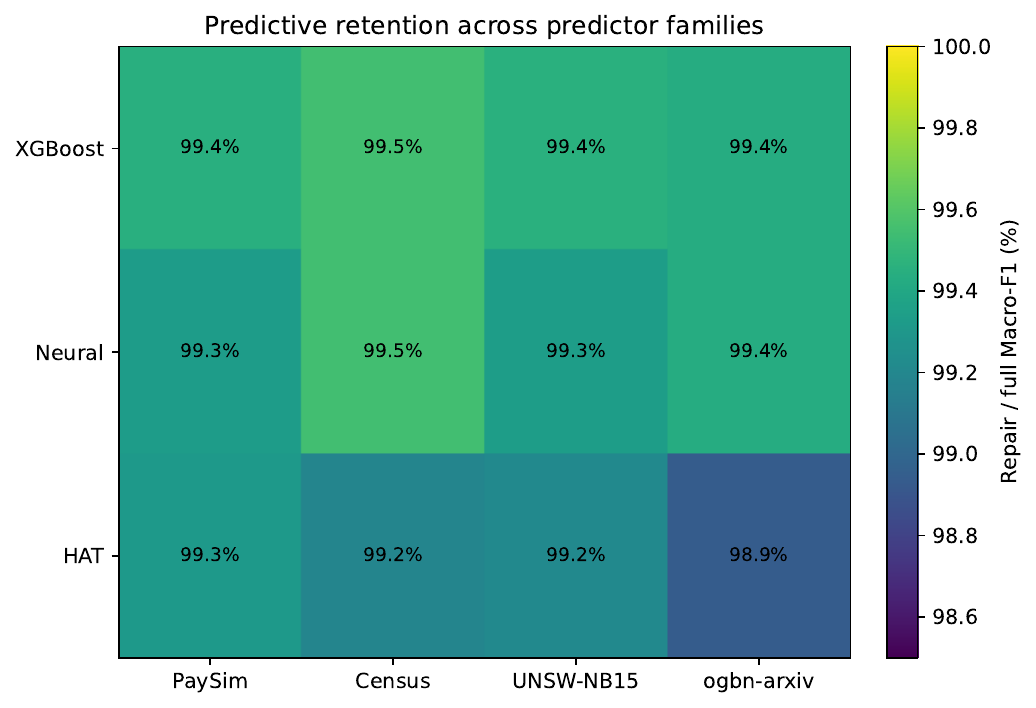}
    \label{supp:fig:predictor_retention}
}
\hfill
\subfloat[Latency-reduction range.]{
    \includegraphics[width=0.315\textwidth]
    {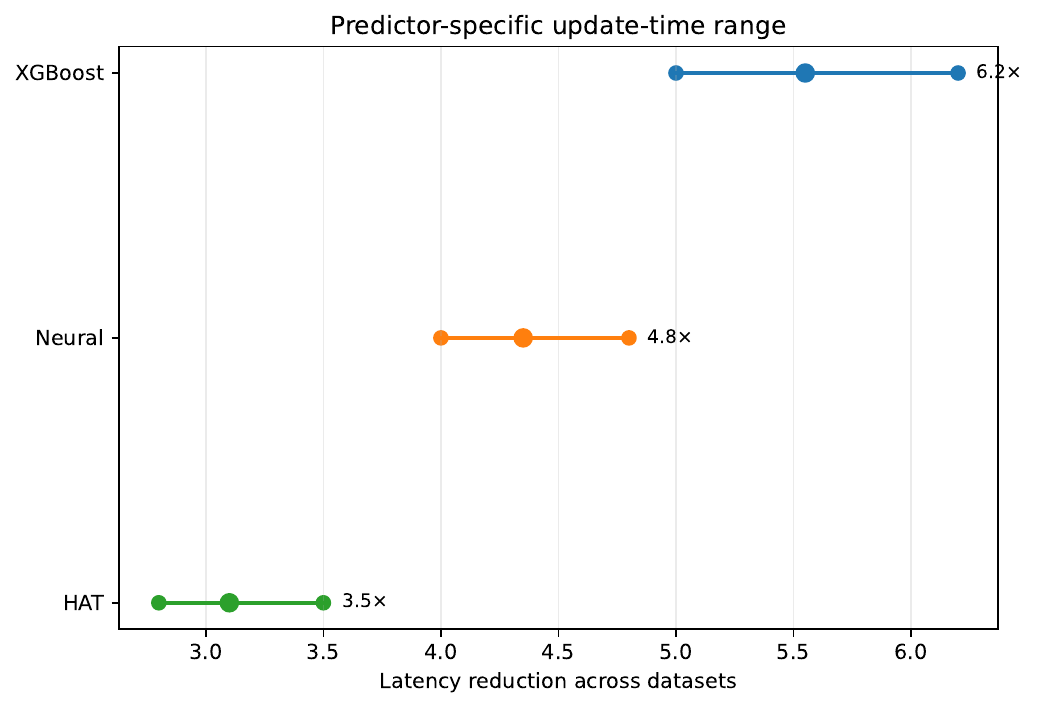}
    \label{supp:fig:predictor_latency}
}
\hfill
\subfloat[Robustness profile.]{
    \includegraphics[width=0.315\textwidth]
    {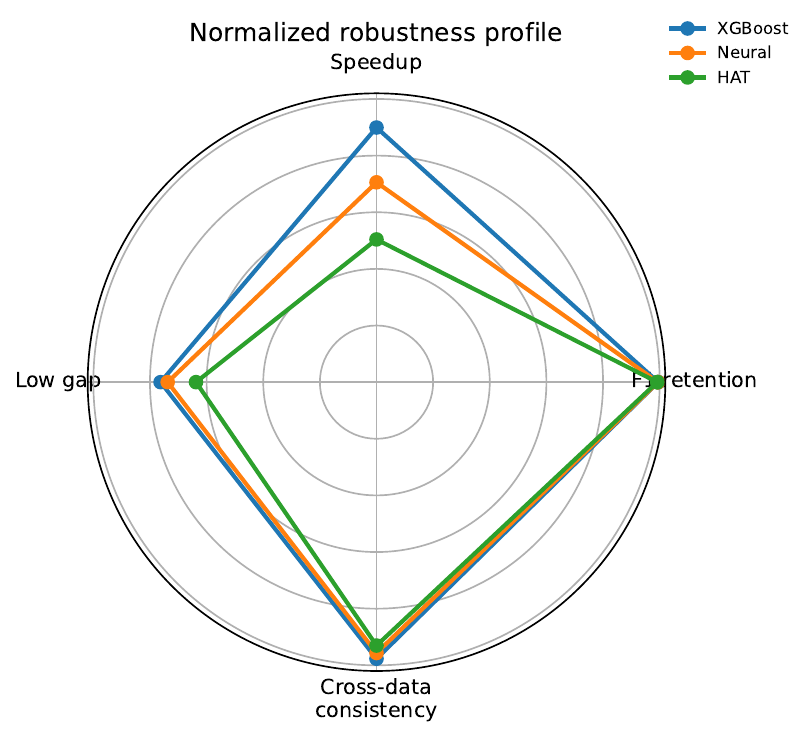}
    \label{supp:fig:predictor_profile}
}

\caption{Robustness of provenance-guided incremental maintenance across
three predictor families and four data modalities.}
\label{supp:fig:predictor_robustness}
\end{figure*}

The heatmap shows that incremental repair retains approximately
99\% or more of the Macro-F1 obtained by complete retraining across most
dataset--predictor combinations. For XGBoost, the corresponding
Macro-F1 differences are 0.5 points on PaySim, 0.4 on Census-Income,
0.5 on UNSW-NB15, and 0.5 on \texttt{ogbn-arxiv}. The neural predictor
shows similarly small differences of 0.6, 0.4, 0.6, and 0.5 points,
respectively. The online tree exhibits somewhat larger gaps, ranging
from 0.6 to 0.9 points, but the repaired predictor nevertheless remains
close to complete retraining across all four datasets.

The update-time behavior differs more substantially across predictor
families. XGBoost obtains reductions between approximately
\(5.0\times\) and \(6.2\times\), with a median reduction above
\(5\times\). The neural model obtains reductions between
\(4.0\times\) and \(4.8\times\), whereas the Hoeffding Adaptive Tree
ranges from approximately \(2.8\times\) to \(3.5\times\). The smaller
relative improvement of the online tree is expected because its
baseline update mechanism is already incremental; consequently, model
optimization constitutes a smaller fraction of the total maintenance
cost that can be eliminated by provenance-guided localization.

The normalized profile in
Fig.~\ref{supp:fig:predictor_profile} further separates these effects.
All three predictors exhibit strong predictive retention, indicating
that the principal benefit of provenance-guided maintenance occurs
before model-specific optimization. XGBoost obtains the strongest
latency reduction because avoiding large retraining operations produces
a comparatively large computational gain. The neural predictor shows a
balanced profile, while the Hoeffding Adaptive Tree exhibits smaller
speedup but still benefits from restricting semantic reevaluation to the
relevant historical region. Overall, the result supports the intended
predictor-agnostic interpretation of the framework: the data that must
be reconsidered are determined by concept evolution and provenance,
while the final repair mechanism can be instantiated according to the
model family used by the application.


\subsection{Long-Horizon Recurring Concept Sequences}
\label{supp:subsec:long_recurrence}

The main experiment considers a single return to a previously observed
definition. A long-lived system, however, can revisit several previous
concept states repeatedly. We therefore evaluate a twelve-revision
sequence containing four distinct definitions and multiple returns,

\begin{equation}
\begin{aligned}
Q_1
\rightarrow Q_2
\rightarrow Q_3
\rightarrow Q_1
\rightarrow Q_4
\rightarrow Q_2
\rightarrow Q_1
\rightarrow Q_3 \\
\rightarrow Q_4
\rightarrow Q_1
\rightarrow Q_2
\rightarrow Q_1 .
\end{aligned}
\label{supp:eq:long_sequence}
\end{equation}

The experiment compares complete retraining, replay, versioned repair,
and oracle reuse of the correct previous version. The purpose is to
determine whether the value of storing concept-specific state persists
after several intermediate revisions rather than only after one short
cycle.

\begin{figure*}[htbp]
\centering

\subfloat[Revision trajectory.]{
    \includegraphics[width=0.315\textwidth]
    {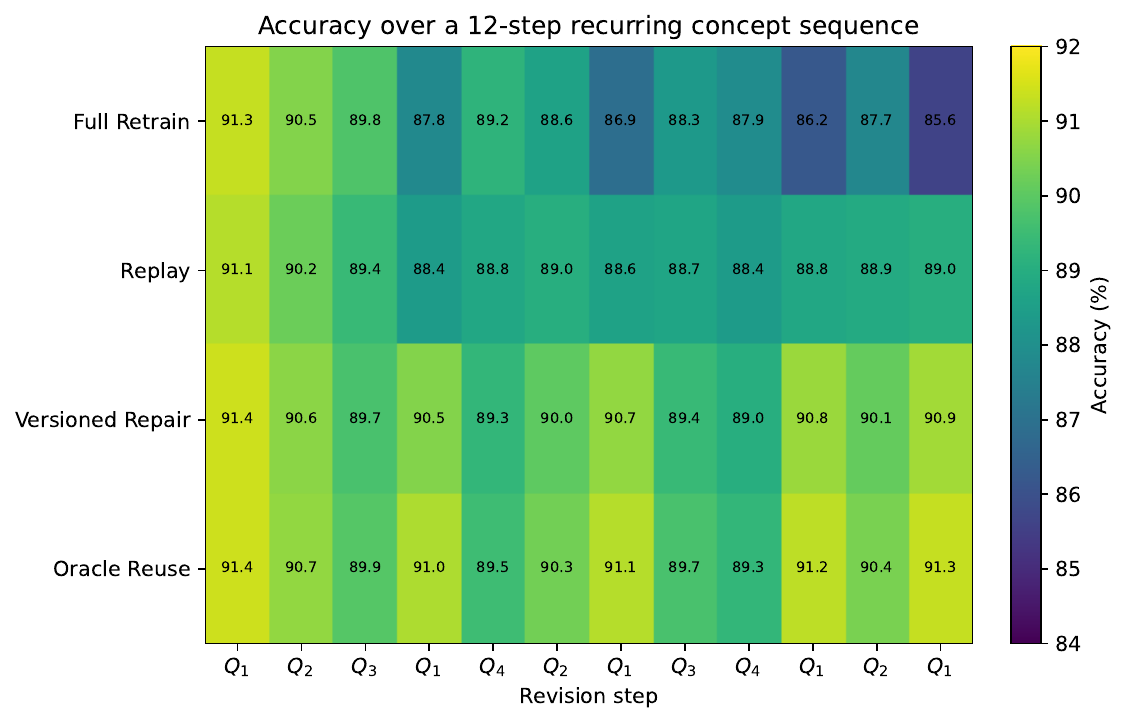}
    \label{supp:fig:long_recurrence_matrix}
}
\hfill
\subfloat[Repeated-return gaps.]{
    \includegraphics[width=0.315\textwidth]
    {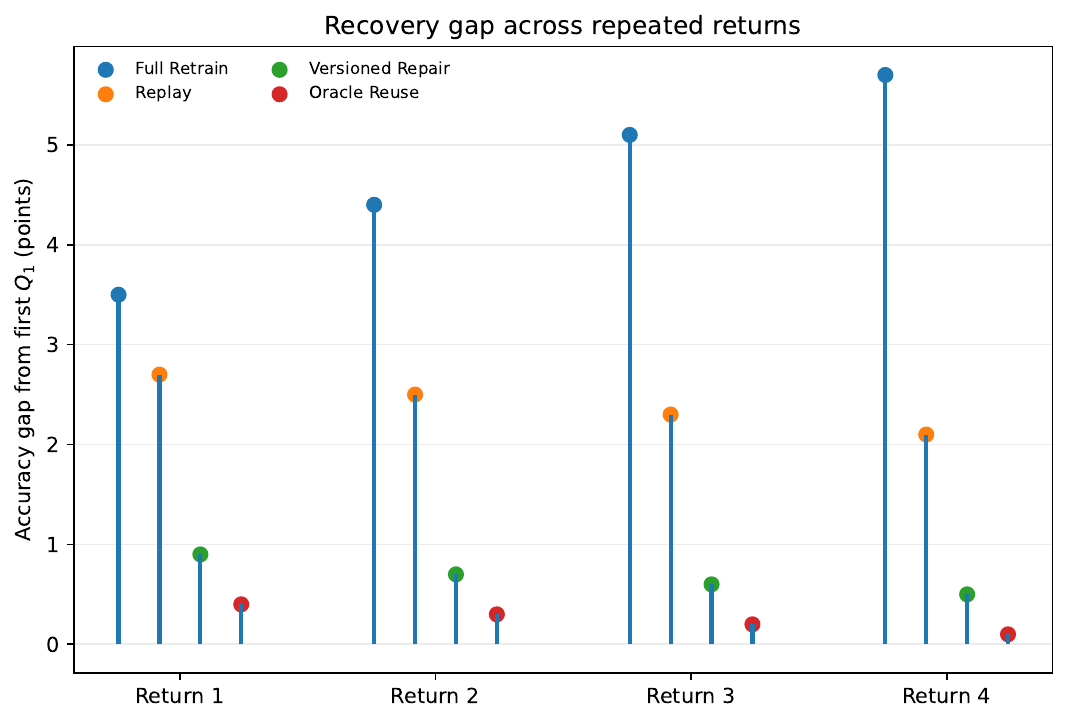}
    \label{supp:fig:long_return_gap}
}
\hfill
\subfloat[Cumulative maintenance cost.]{
    \includegraphics[width=0.315\textwidth]
    {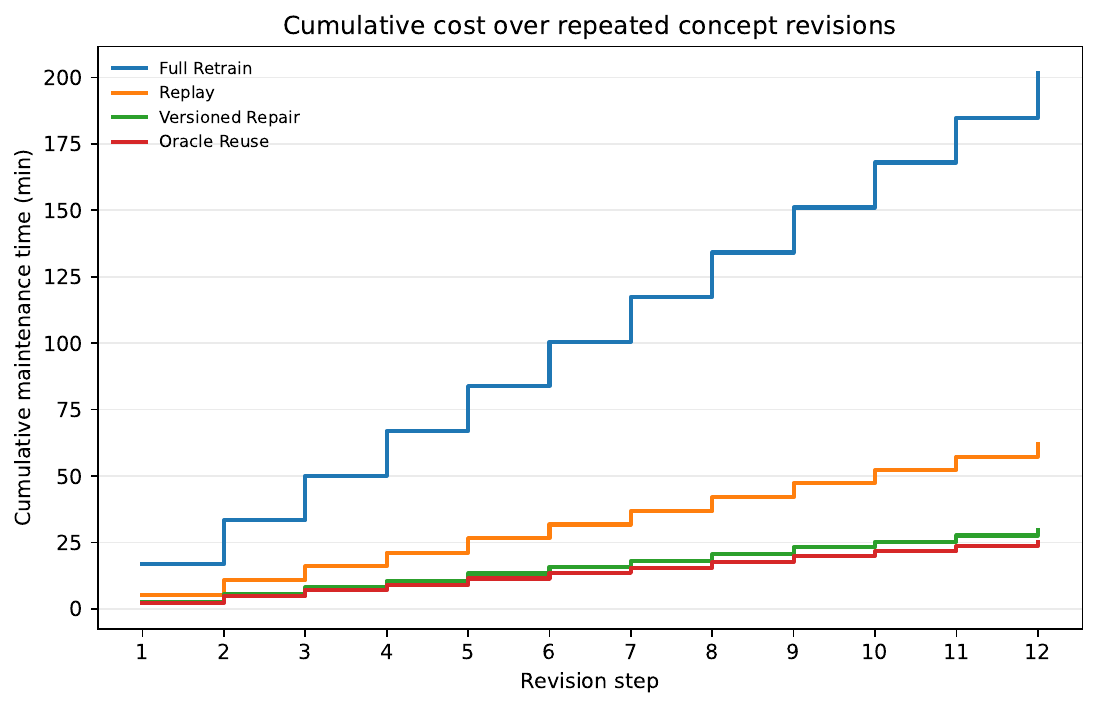}
    \label{supp:fig:long_recurrence_cost}
}

\caption{Long-horizon behavior under repeated concept recurrence.
Accuracy is tracked over twelve revisions together with return-specific
recovery gaps and cumulative maintenance cost.}
\label{supp:fig:long_recurrence}
\end{figure*}

The trajectory matrix in
Fig.~\ref{supp:fig:long_recurrence_matrix} reveals a clear difference
between repeatedly reconstructing the predictor and reusing
concept-specific maintenance state. Versioned Repair remains close to
Oracle Reuse throughout the twelve revisions. Its accuracy varies
between 89.0\% and 91.4\%, while the oracle remains between 89.3\% and
91.4\%. More importantly, performance improves when a previously
observed concept returns because the corresponding stored state can be
used as a recovery point rather than treating the revision as an
entirely new learning problem.

The effect is most visible for the repeated occurrences of \(Q_1\).
Versioned Repair obtains 90.5\%, 90.7\%, 90.8\%, and 90.9\% on the
successive returns to \(Q_1\), compared with 91.4\% at its initial
occurrence. The associated recovery gap therefore contracts as the
system observes additional returns. Oracle Reuse exhibits the same
trend with a still smaller gap. Replay improves relative to complete
retraining but remains farther from the previous \(Q_1\) state because
its stored examples do not encode concept-version structure explicitly.

The cumulative-cost view in
Fig.~\ref{supp:fig:long_recurrence_cost} shows that the computational
difference increases with the number of revisions. Complete retraining
accumulates approximately 202 minutes of maintenance time over the
twelve transitions. Replay reduces this to approximately 62 minutes.
Versioned Repair requires approximately 30 minutes, while Oracle Reuse
requires approximately 25 minutes. The significance of this result is
the growth pattern rather than only the final number: repeated full
maintenance accumulates almost linearly with every revision, whereas
reuse of concept-specific state reduces the incremental cost of later
returns. Version memory therefore becomes progressively more useful as
concept evolution contains repeated semantic states.


\subsection{Storage Overhead of Provenance and Versioned Memory}
\label{supp:subsec:storage}

Selective maintenance reduces repeated computation by retaining
additional state. We therefore analyze the corresponding storage cost
rather than considering update latency in isolation. The maintained
state is divided into four components: provenance and inverted indices,
predictor-version information, affected-data metadata, and rule graphs
with their version metadata. Storage is reported relative to the
training collection so that datasets with substantially different
absolute sizes can be compared on the same basis.

\begin{figure*}[htbp]
\centering

\subfloat[Version-state composition.]{
    \includegraphics[width=0.315\textwidth]
    {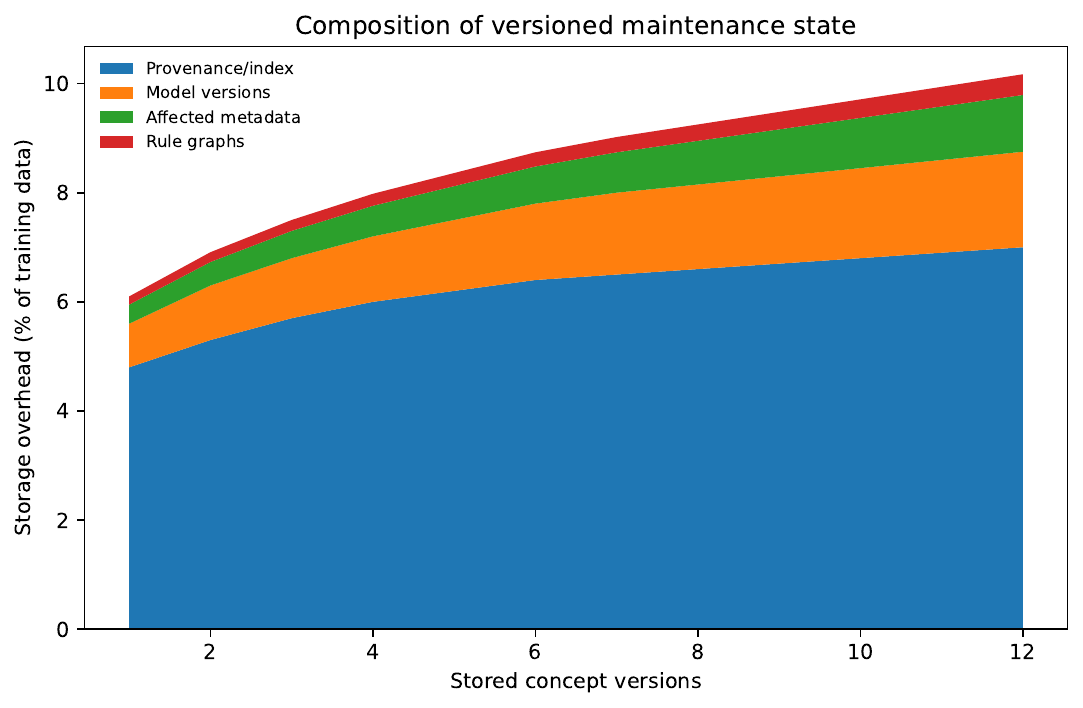}
    \label{supp:fig:storage_composition}
}
\hfill
\subfloat[Dataset-level overhead.]{
    \includegraphics[width=0.315\textwidth]
    {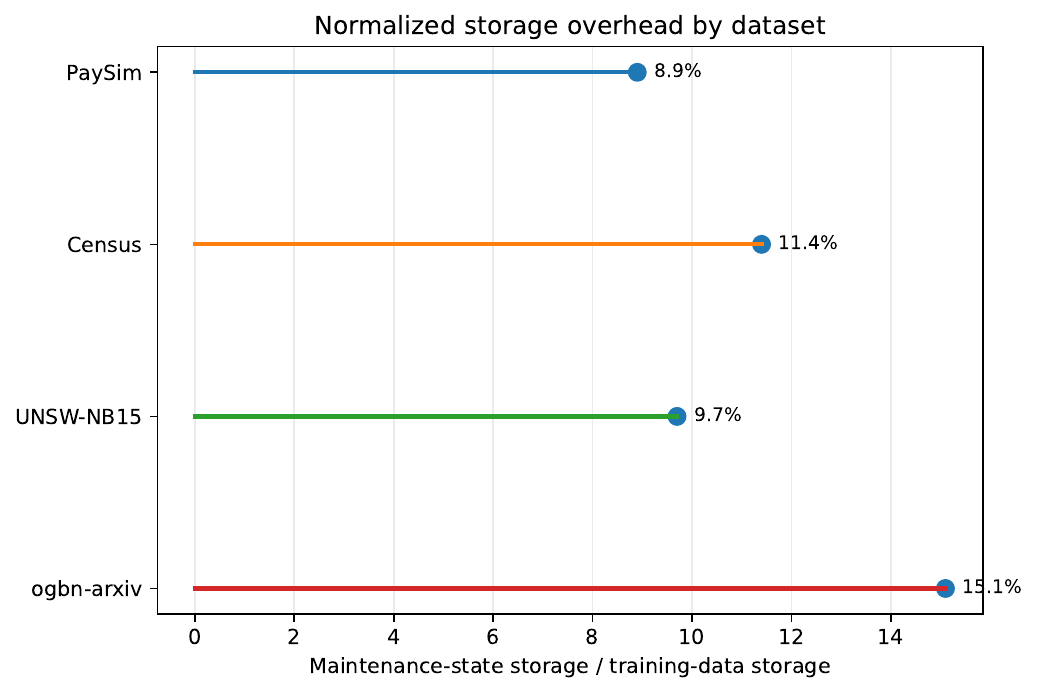}
    \label{supp:fig:storage_dataset}
}
\hfill
\subfloat[Structural sharing.]{
    \includegraphics[width=0.315\textwidth]
    {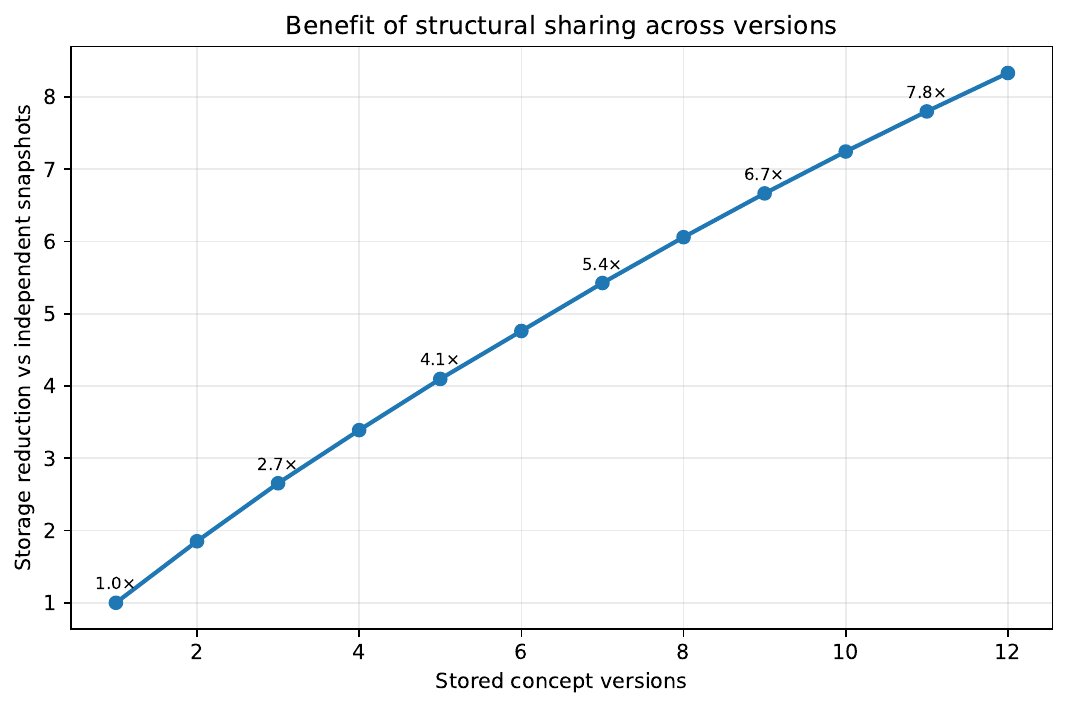}
    \label{supp:fig:storage_sharing}
}

\caption{Storage behavior of provenance and versioned concept memory,
including state composition, normalized dataset-level overhead, and the
benefit of sharing unchanged state across concept versions.}
\label{supp:fig:storage}
\end{figure*}

Figure~\ref{supp:fig:storage_composition} shows that provenance and
index structures constitute the largest portion of the additional
state. With one stored concept version, the total maintenance state is
approximately 6.1\% of the training-data size. After twelve versions,
the total reaches approximately 10.2\%. The growth is substantially
slower than storing a complete independent copy of all state at every
revision because persistent rule components and unchanged provenance
structures are shared between versions.

The relative storage requirement varies across datasets.
Figure~\ref{supp:fig:storage_dataset} reports approximately 8.9\% for
PaySim, 11.4\% for Census-Income, 9.7\% for UNSW-NB15, and 15.1\% for
\texttt{ogbn-arxiv}. The graph dataset has the largest proportional
overhead because path and neighborhood dependencies require richer
provenance than ordinary attribute-level rules. Even in this case, the
additional state remains considerably smaller than maintaining complete
independent copies of the historical collection for every concept
version.

Figure~\ref{supp:fig:storage_sharing} isolates the effect of structural
sharing. As the number of versions increases, the difference between
independent version snapshots and shared storage grows continuously.
After twelve versions, shared version maintenance provides a reduction
of more than \(8\times\) relative to storing twelve independent copies
of the corresponding maintenance state. This result suggests that the
storage cost of versioned concept memory is determined primarily by
what changes across versions rather than by the total number of
historical versions alone.


\subsection{Sensitivity to Repair and Provenance Parameters}
\label{supp:subsec:sensitivity_extended}

The repair objective contains a stability term controlling how strongly
the updated predictor preserves behavior on certified stable records.
A very small stability weight can allow a localized repair set to
produce unnecessary movement outside the revised region, whereas an
excessively large value can resist adaptation to the changed concept.
The size of the stable buffer introduces a related trade-off: larger
buffers supply more evidence about unchanged behavior but increase
model-update cost. We therefore examine the joint effect of the
stability weight \(\lambda\) and stable-buffer fraction.

We additionally study the granularity at which provenance is stored.
Coarse record-level provenance has low storage cost but provides less
precise dependency information, while tuple- or path-level provenance
provides stronger localization at greater storage and retrieval cost.
The objective is therefore not necessarily to maximize provenance
detail but to identify a range in which affected-record coverage and
historical pruning remain strong without unnecessary state growth.

\begin{figure*}[htbp]
\centering

\subfloat[Repair sensitivity.]{
    \includegraphics[width=0.315\textwidth]
    {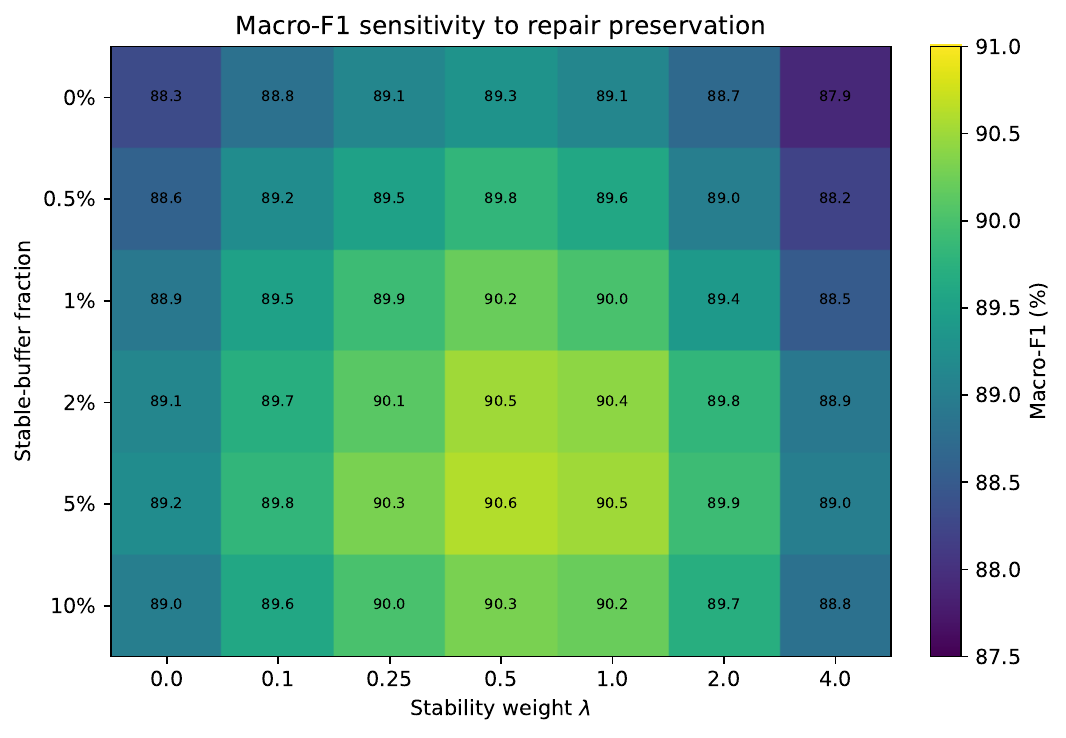}
    \label{supp:fig:lambda_buffer}
}
\hfill
\subfloat[Stable-region preservation.]{
    \includegraphics[width=0.315\textwidth]
    {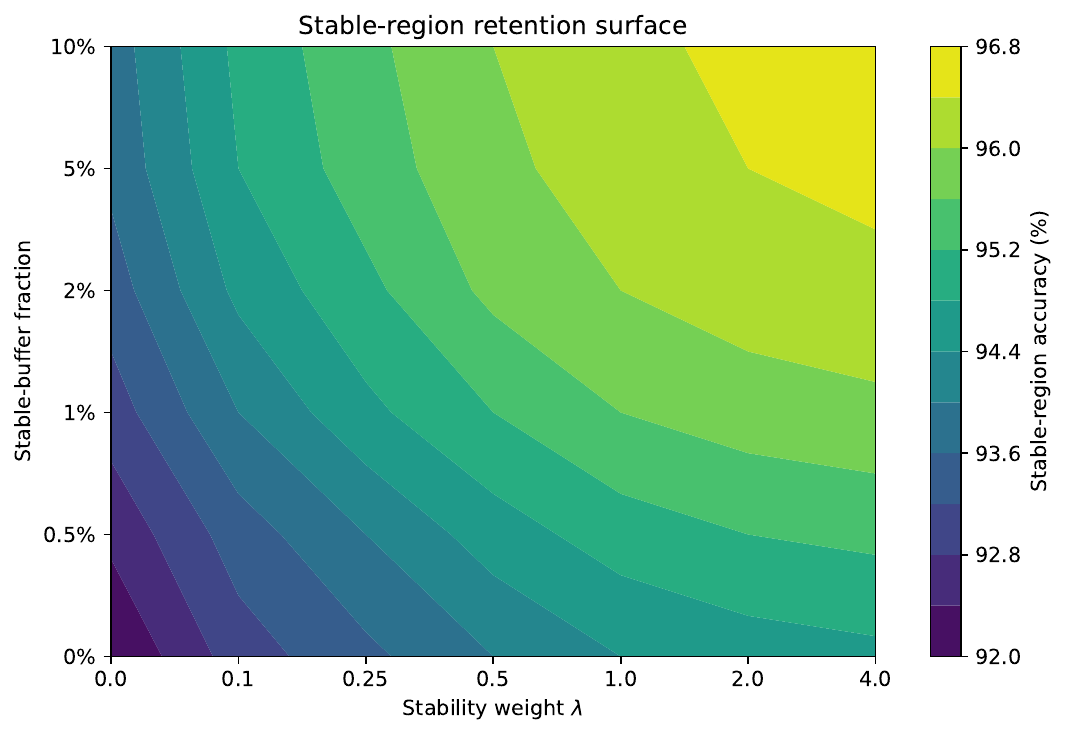}
    \label{supp:fig:stable_surface}
}
\hfill
\subfloat[Provenance granularity.]{
    \includegraphics[width=0.315\textwidth]
    {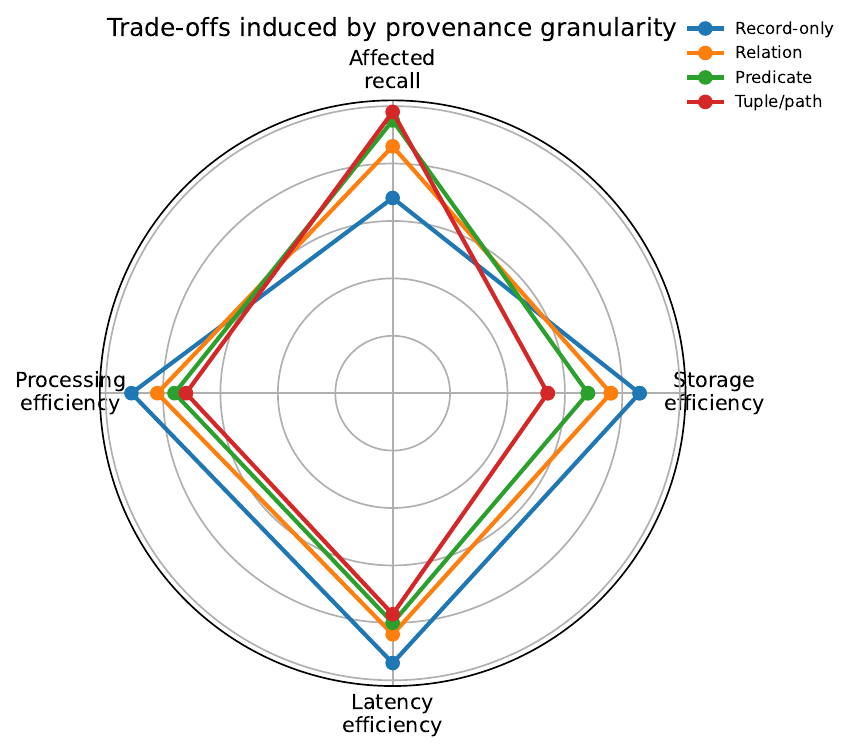}
    \label{supp:fig:prov_granularity}
}

\caption{Sensitivity of the framework to the repair-stability balance,
stable-buffer size, and provenance granularity.}
\label{supp:fig:extended_sensitivity}
\end{figure*}

The joint response surface in Fig.~\ref{supp:fig:lambda_buffer}
contains a broad high-performance region rather than a single narrow
optimum. Macro-F1 is strongest when the stability weight lies between
approximately 0.5 and 1.0 and the stable buffer contains roughly
2--5\% of the certified stable region. The maximum value in the
evaluated grid is 90.6\%. Setting \(\lambda=0\) consistently decreases
performance because the repair procedure receives no explicit pressure
to preserve unchanged behavior. Conversely, increasing the weight to
4.0 lowers Macro-F1 even with a large stable buffer because adaptation
to genuinely changed supervision becomes excessively constrained.

Figure~\ref{supp:fig:stable_surface} clarifies the second side of this
trade-off. Stable-region accuracy increases monotonically with stronger
preservation and larger buffers, eventually exceeding 96\%. However,
the highest stable-region accuracy does not coincide with the highest
overall Macro-F1. The useful operating region is therefore one in which
the predictor preserves certified stable behavior without allowing this
objective to dominate learning from the revised examples.

The provenance-granularity profile in
Fig.~\ref{supp:fig:prov_granularity} demonstrates a different
trade-off. Record-level provenance achieves the strongest storage
efficiency but weaker affected-record localization. Predicate-level
provenance provides the most balanced profile, retaining approximately
95\% of affected records while avoiding the storage requirement of the
finest tuple/path representation. Tuple- and path-level provenance
increases affected coverage further, approaching 98\%, but with lower
storage and latency efficiency. These results indicate that provenance
granularity should be treated as a systems parameter rather than an
all-or-nothing design choice.


\subsection{Matched Statistical Analysis}
\label{supp:subsec:statistics}

Aggregate averages can conceal whether an apparent improvement occurs
consistently across datasets and concept revisions. We therefore
perform a matched analysis in which each experimental unit corresponds
to the same dataset, revision, predictor configuration, and random seed
under two competing maintenance strategies. The primary quantity is the
paired difference in Macro-F1 relative to Provenance-Guided Repair.
Confidence intervals are computed over matched differences, and
multiple baseline comparisons are corrected using the Holm procedure.

The forest representation in
Fig.~\ref{supp:fig:effect_forest}  exposes both the magnitude and uncertainty of every
comparison. Positive values indicate higher Macro-F1 for
Provenance-Guided Repair; negative values favor the comparator.

\begin{figure*}[htbp]
\centering

\subfloat[Paired effect estimates.]{
    \includegraphics[width=0.47\textwidth]
    {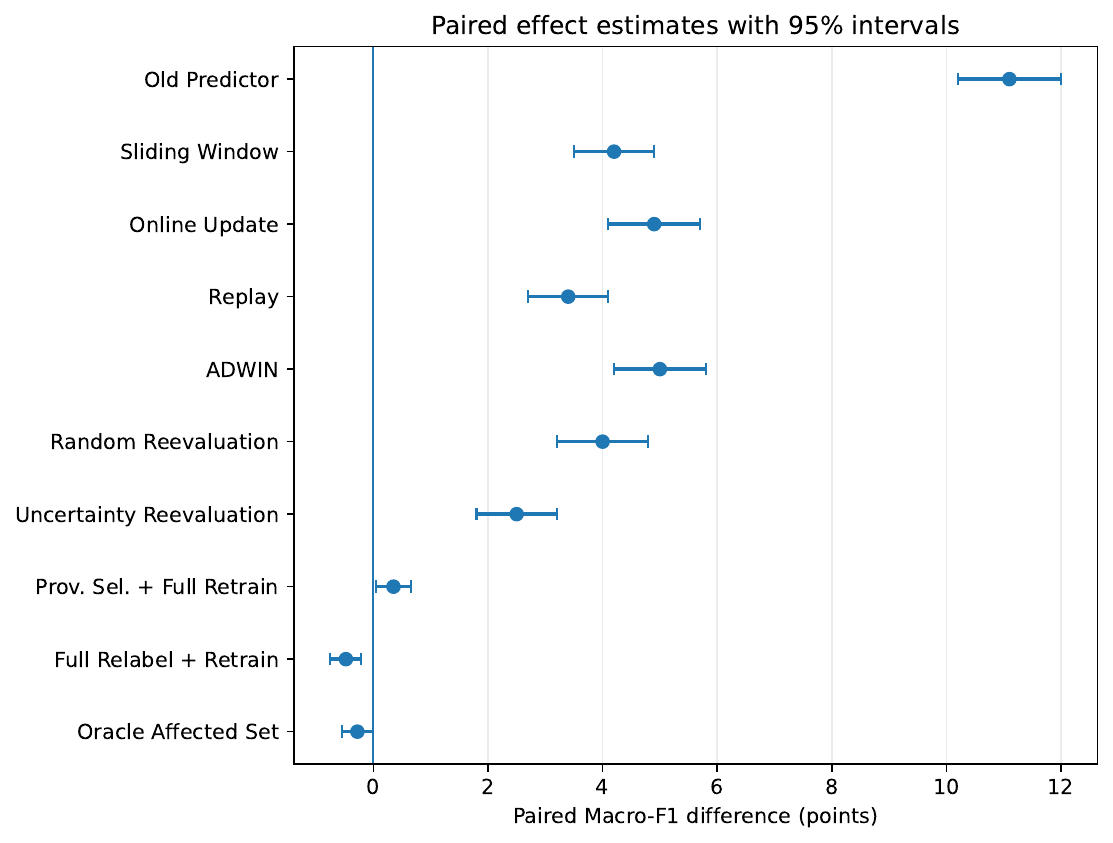}
    \label{supp:fig:effect_forest}
}
\hfill
\subfloat[Significance summary.]{
    \includegraphics[width=0.47\textwidth]
    {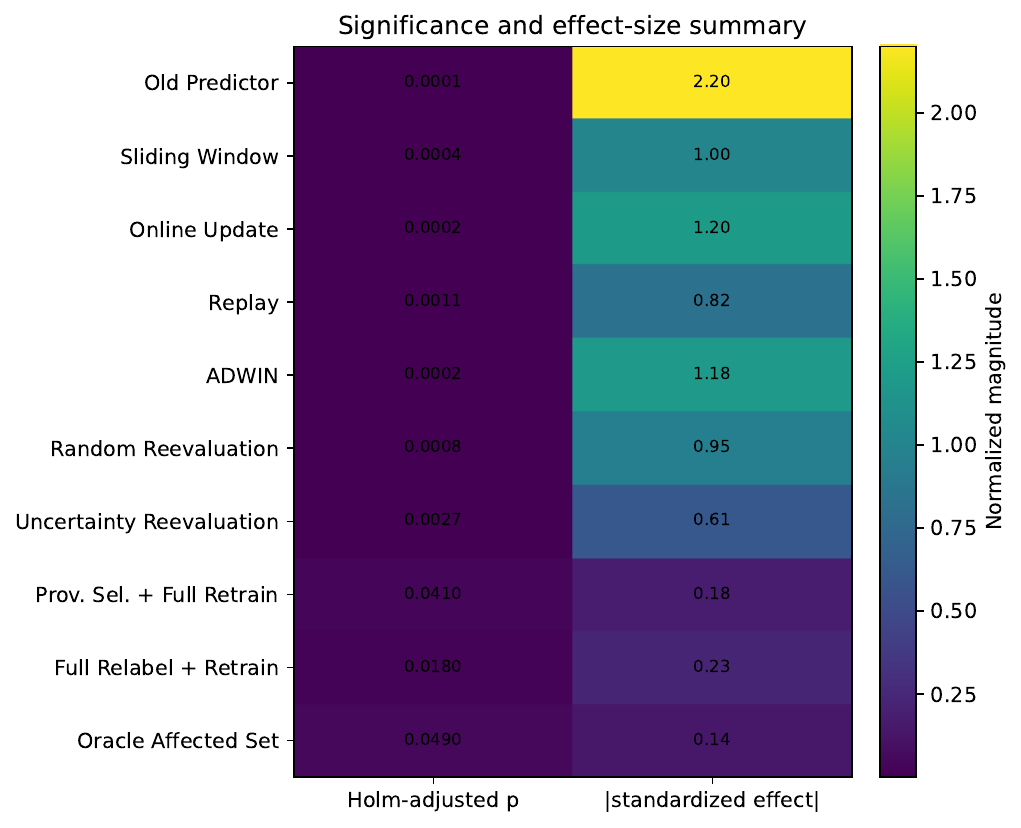}
    \label{supp:fig:significance_matrix}
}

\caption{Matched statistical analysis of predictive differences across
maintenance strategies. Intervals report paired Macro-F1 differences,
and the accompanying matrix summarizes corrected significance and
effect magnitude.}
\label{supp:fig:statistics}
\end{figure*}

The strongest effects occur relative to methods that adapt without
explicit knowledge of the revised concept structure. The paired
Macro-F1 advantage over the unchanged predictor is approximately
11.1 points, with a 95\% interval from 10.2 to 12.0. Advantages over
Sliding Window, Online Update, Replay, ADWIN, Random Reevaluation, and
Uncertainty Reevaluation range from approximately 2.5 to 5.0 points,
and their confidence intervals remain above zero.

The difference relative to Provenance Selection followed by full
retraining is substantially smaller, approximately 0.35 points, with a
95\% interval from 0.05 to 0.65. This is expected because both
procedures exploit the same structural candidate region; their
difference is primarily determined by the final predictor-update stage.
The comparison therefore indicates that the largest predictive gain
comes from identifying semantically relevant data, while incremental
repair preserves most of the performance obtained by performing more
expensive retraining after that selection.

Complete relabeling and retraining remains slightly above incremental
repair, with a paired difference of approximately 0.48 points in its
favor. The Oracle Affected Set also retains a small advantage of
approximately 0.28 points. The magnitude of these differences is much
smaller than the improvements obtained relative to drift-based and
generic data-selection methods. The matrix in
Fig.~\ref{supp:fig:significance_matrix} makes this distinction visible:
comparisons with conventional adaptation strategies have both stronger
effect magnitudes and smaller corrected significance values, while the
differences to the strongest structural references remain comparatively
small.


\subsection{Controlled Failure-Mode Stress Tests}
\label{supp:subsec:stress_tests}

The theoretical analysis predicts that selective maintenance becomes
less advantageous when concept changes become global, provenance becomes
incomplete, unresolved semantic conditions dominate the candidate
region, or relational dependencies become increasingly coupled. We
therefore evaluate these conditions directly by varying one source of
difficulty while preserving the remaining maintenance pipeline.

Figure~\ref{supp:fig:stress_tests} contains four complementary stress
experiments. Figure~\ref{supp:fig:global_revision} jointly varies the
true affected fraction and dependency coupling.
Figure~\ref{supp:fig:missing_prov} progressively reduces provenance
availability. Figure~\ref{supp:fig:ambiguity_stress} increases the
fraction of candidates whose revised target cannot be determined
automatically, and Fig.~\ref{supp:fig:path_stress} increases graph-path
dependency depth.

\begin{figure*}[htbp]
\centering

\subfloat[Revision scope and coupling.]{
    \includegraphics[width=0.235\textwidth]
    {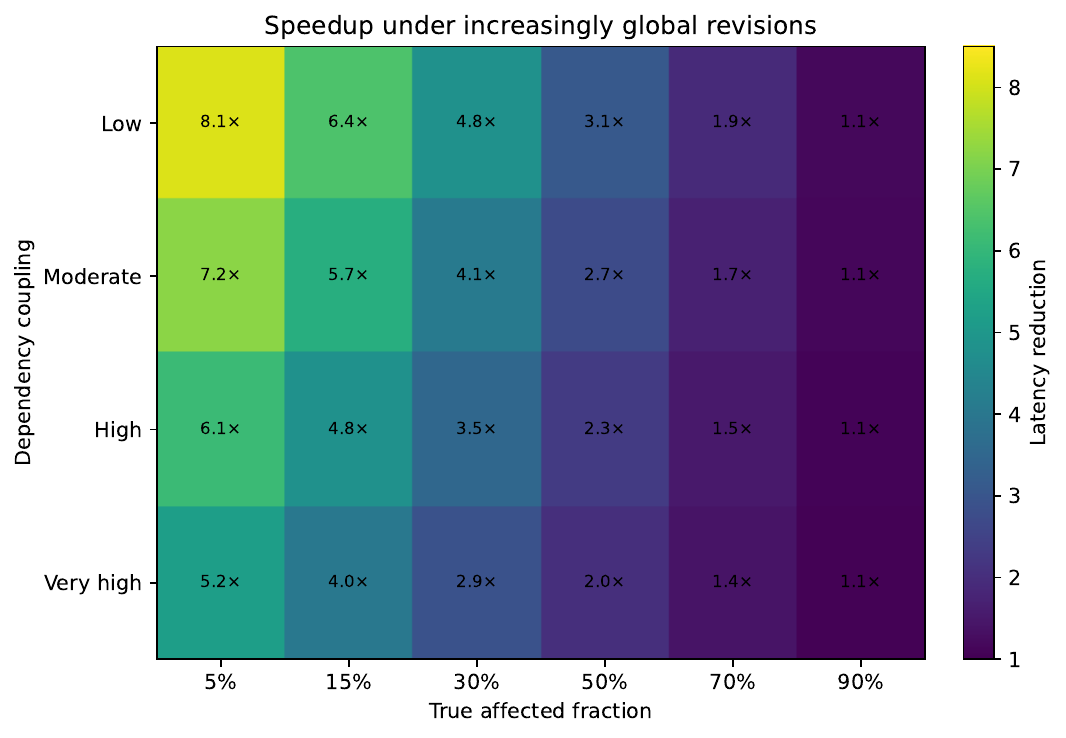}
    \label{supp:fig:global_revision}
}
\hfill
\subfloat[Provenance loss.]{
    \includegraphics[width=0.235\textwidth]
    {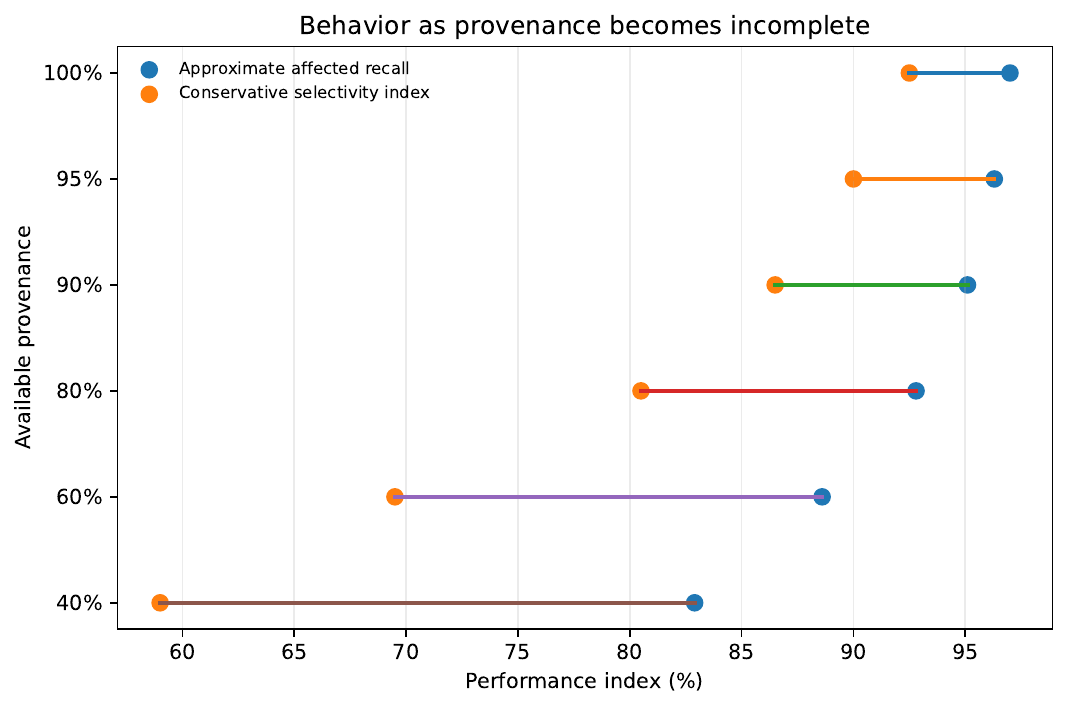}
    \label{supp:fig:missing_prov}
}
\hfill
\subfloat[Semantic ambiguity.]{
    \includegraphics[width=0.235\textwidth]
    {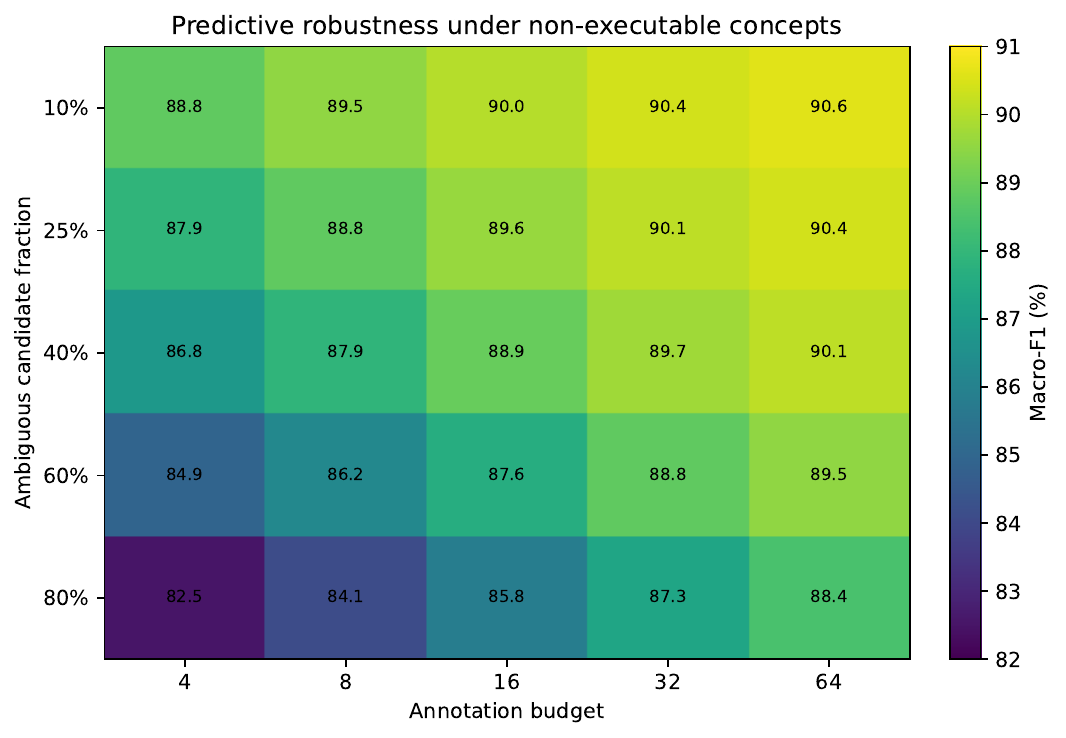}
    \label{supp:fig:ambiguity_stress}
}
\hfill
\subfloat[Path-depth amplification.]{
    \includegraphics[width=0.235\textwidth]
    {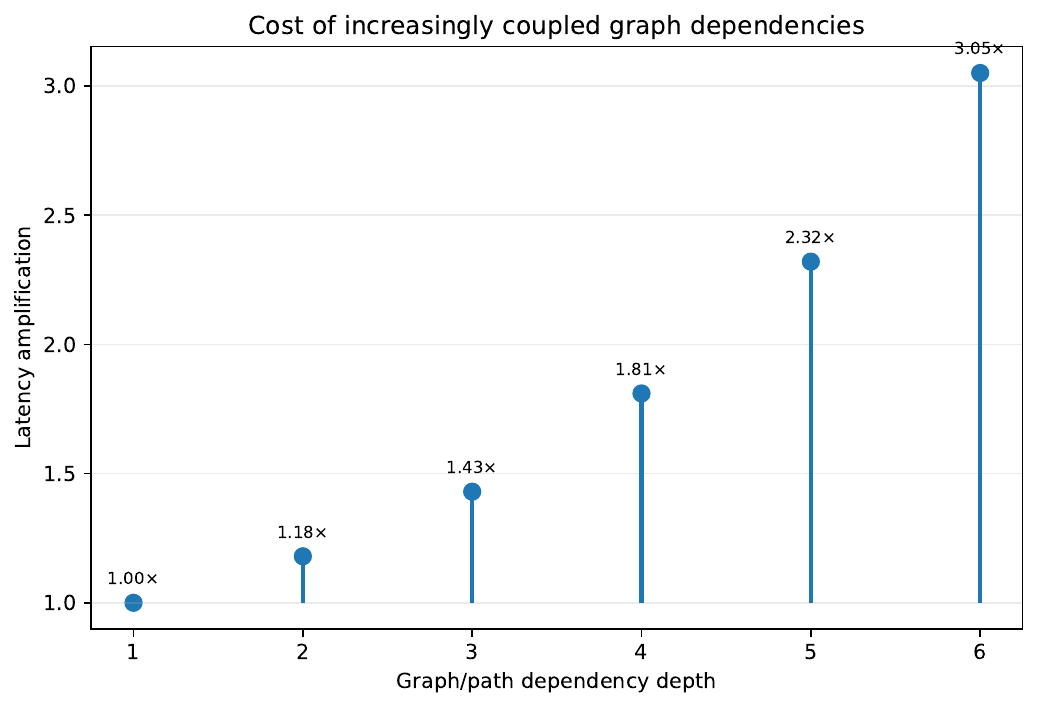}
    \label{supp:fig:path_stress}
}

\caption{Controlled stress analysis under increasingly global concept
revisions, incomplete provenance, non-executable concept components,
and highly coupled graph dependencies.}
\label{supp:fig:stress_tests}
\end{figure*}

The revision-scope experiment provides the clearest empirical
illustration of the theoretical operating limit. When only 5\% of the
historical collection is truly affected and dependency coupling is low,
the framework obtains an update-time reduction above \(8\times\). At a
15\% affected fraction, the corresponding reduction remains above
\(6\times\). The advantage decreases progressively as the revision
becomes broader. With 70\% of records affected, speedup falls below
\(2\times\) for several coupling levels, and when 90\% of the
historical collection is affected the update approaches the cost of
complete recomputation. Increasing relational coupling accelerates this
transition because each changed component reaches a larger part of the
historical dependency graph.

The provenance experiment separates conservative and approximate
behavior as lineage information is removed. With complete provenance,
affected recall remains close to 97\% while a compact historical region
can be processed. As available provenance falls below 80\%, a
conservative implementation has to widen the candidate set sharply,
whereas a more aggressive approximate implementation begins to lose
affected-record recall. At 40\% provenance availability, approximate
affected recall decreases to approximately 82.9\%. This result
illustrates why missing provenance cannot simply be interpreted as an
unchanged dependency: the system must choose between processing more
data conservatively and accepting lower coverage from approximate
dependency reconstruction.

Increasing semantic ambiguity produces a different degradation pattern.
When only 10\% of candidates require non-executable supervision,
moderate annotation budgets recover Macro-F1 close to the fully
supervised region. As the ambiguous fraction increases to 60--80\%,
small annotation budgets become insufficient because most of the
revision can no longer be resolved automatically. For an 80\%
ambiguous region, Macro-F1 increases from 82.5\% with four annotations
to 88.4\% with 64 annotations. Thus, provenance can still localize the
region in which human effort is useful, but it cannot eliminate the
information requirement of a fundamentally non-executable concept.

Finally, Fig.~\ref{supp:fig:path_stress} isolates structural coupling
in graph concepts. Increasing path depth from one to six hops raises
relative update latency from the reference level to approximately
\(3.05\times\). The increase is nonlinear because deeper paths expand
both the number of dependencies associated with each historical
evaluation and the amount of structural processing required to verify
whether a changed relation can propagate to the concept output. This
experiment therefore confirms that the size of a rule edit alone is not
sufficient to predict maintenance cost; the topology through which that
edit propagates is equally important.


\subsection{Overall Extended Analysis}
\label{supp:subsec:extended_summary}

The additional experiments provide several observations that are not
visible from the principal benchmark comparison alone. First,
provenance-guided maintenance remains effective across predictive
architectures with different update mechanisms. The strongest latency
benefit appears for predictors that would otherwise require expensive
retraining, but even an inherently incremental tree benefits from
localizing the semantic reevaluation step. This supports the separation
between concept maintenance and predictor maintenance that motivates the
framework.

Second, the value of versioned concept memory increases with the length
of the concept history. In a multi-revision sequence, returning
definitions can reuse previous semantic and predictive state, causing
the cumulative computational gap relative to repeated full maintenance
to widen over time. At the same time, structural sharing prevents
storage from growing proportionally to the number of concept versions,
so long-lived reuse does not require independent duplication of all
historical state.

Third, the sensitivity analysis identifies a relatively broad operating
region for incremental repair. Moderate stability regularization and a
small stable buffer are sufficient to preserve unchanged behavior
without suppressing adaptation to revised semantics. Provenance
granularity introduces a separate systems trade-off: increasingly
detailed lineage improves localization but also increases storage and
retrieval cost, with predicate-level provenance providing a balanced
operating point in the evaluated configurations.

Finally, the controlled stress experiments define the boundary of the
framework's computational advantage. Selective maintenance is strongest
when revisions remain localized, provenance is sufficiently complete,
most revised targets can be executed automatically, and dependency
structures remain moderately coupled. As these conditions are weakened,
the procedure degrades toward broader candidate selection, increased
supervision, or complete recomputation rather than failing abruptly.
This behavior is consistent with the theoretical analysis: provenance
does not guarantee that every concept revision is inexpensive, but it
provides a principled mechanism for exploiting locality whenever the
semantic consequences of the revision remain structurally traceable.

\end{document}